\documentclass[twoside]{article}

\usepackage[accepted]{aistats2026}

\usepackage[sort,round,compress]{natbib}

\usepackage{hyperref}
\usepackage{microtype}
\usepackage{graphicx}
\usepackage{subfigure}
\usepackage{booktabs} %
\usepackage{relsize}
\usepackage[sort,round,compress]{natbib}
\usepackage{pifont}   %

\usepackage{multirow}
\usepackage{threeparttable}
\usepackage{adjustbox}
\usepackage[table]{xcolor}
\definecolor{lightgrey}{rgb}{0.9, 0.9, 0.9}

\usepackage{amsmath}
\usepackage{amssymb}
\usepackage{mathtools}
\mathtoolsset{showonlyrefs,showmanualtags}
\usepackage{amsthm}
\usepackage{makecell}
\usepackage{mdframed}

\usepackage[capitalize,noabbrev]{cleveref}

\crefname{tcb@cnt@theorem}{Theorem}{Theorems}
\Crefname{tcb@cnt@theorem}{Theorem}{Theorems}
\crefname{assumption}{assumption}{assumptions}
\Crefname{assumption}{Assumption}{Assumptions}

\usepackage{url}
\usepackage{doi}
\usepackage[many]{tcolorbox}
\tcbuselibrary{theorems,skins,breakable}
\usepackage{cancel}

\usepackage[textsize=tiny]{todonotes}

\usepackage{amsmath,amsfonts,bm,bbm,mathrsfs}
\usepackage{amssymb,mathtools}

\def\eqref#1{(\ref{#1})}

\newcommand{\tr}{\mathrm{tr}}

\def\vzero{{\bm{0}}}

\def\ve{{\bm{e}}}

\def\vp{{\bm{p}}}

\def\vu{{\bm{u}}}
\def\vv{{\bm{v}}}

\def\vx{{\bm{x}}}
\def\vy{{\bm{y}}}
\def\vz{{\bm{z}}}

\def\mA{{\bm{A}}}
\def\mB{{\bm{B}}}
\def\mC{{\bm{C}}}
\def\mD{{\bm{D}}}

\def\mH{{\bm{H}}}
\def\mI{{\bm{I}}}
\def\mJ{{\bm{J}}}

\def\mM{{\bm{M}}}

\def\mP{{\bm{P}}}
\def\mQ{{\bm{Q}}}

\def\mS{{\bm{S}}}
\def\mT{{\bm{T}}}
\def\mU{{\bm{U}}}
\def\mV{{\bm{V}}}

\def\mX{{\bm{X}}}

\def\mZ{{\bm{Z}}}

\DeclareMathAlphabet{\mathsfit}{\encodingdefault}{\sfdefault}{m}{sl}
\SetMathAlphabet{\mathsfit}{bold}{\encodingdefault}{\sfdefault}{bx}{n}

\def\gA{{\mathcal{A}}}
\def\gB{{\mathcal{B}}}

\def\gE{{\mathcal{E}}}
\def\gF{{\mathcal{F}}}

\def\gH{{\mathcal{H}}}
\def\gI{{\mathcal{I}}}

\def\gL{{\mathcal{L}}}
\def\gM{{\mathcal{M}}}
\def\gN{{\mathcal{N}}}
\def\gO{{\mathcal{O}}}
\def\gP{{\mathcal{P}}}

\def\gR{{\mathcal{R}}}
\def\gS{{\mathcal{S}}}

\def\gX{{\mathcal{X}}}

\def\gZ{{\mathcal{Z}}}

\def\sN{{\mathbb{N}}}

\def\sP{{\mathbb{P}}}

\def\sR{{\mathbb{R}}}

\newcommand{\E}{\mathbb{E}}

\newcommand{\KL}{D_{\mathrm{KL}}}
\newcommand{\Var}{\mathrm{Var}}

\DeclareMathOperator*{\argmax}{arg\,max}
\DeclareMathOperator*{\argmin}{arg\,min}

\theoremstyle{plain}
\newtheorem{assumption}{Assumption}
\newtheorem{theorem}{Theorem}[section]
\newtheorem{definition}{Definition}[section]
\newtheorem{proposition}{Proposition}[section]
\newtheorem{lemma}{Lemma}[section]
\newtheorem{claim}{Claim}[section]

\newtheorem{remark}{Remark}

\def\floor#1{\lfloor #1 \rfloor}

\def\1{\mathbbm{1}}

\newcommand{\diag}{\mathrm{diag}}
\newcommand{\col}{\mathrm{col}}
\newcommand{\rank}{\mathrm{rank}}
\newcommand{\Ber}{\mathrm{Ber}}

\tcbset{
  baseboxstyle/.style={
    enhanced,
    breakable,
    boxrule=0pt,
    frame hidden,
    sharp corners,
    before skip=10pt,
    after skip=10pt,
    left=2pt, right=2pt,
    top=2pt, bottom=2pt
  }
}

\tcolorboxenvironment{definition}{baseboxstyle, colback=black!5}
\tcolorboxenvironment{theorem}{baseboxstyle, colback=blue!5}
\tcolorboxenvironment{lemma}{baseboxstyle, colback=teal!5}
\tcolorboxenvironment{proposition}{baseboxstyle, colback=blue!5}
\tcolorboxenvironment{corollary}{baseboxstyle, colback=blue!5}
\tcolorboxenvironment{remark}{baseboxstyle, colback=black!3}
\tcolorboxenvironment{assumption}{baseboxstyle, colback=yellow!5}
\tcolorboxenvironment{claim}{baseboxstyle, colback=orange!5}
\usepackage[algo2e, linesnumbered, ruled, vlined]{algorithm2e}
\SetAlgoNlRelativeSize{0}

\usepackage{dsfont}
\usepackage[normalem]{ulem}
\usepackage{soul}
\usepackage{doi}

\newcommand{\Reg}{\mathrm{Reg}}
\newcommand{\Unif}{\mathrm{Unif}}
\newcommand{\dmu}{\dot{\mu}}
\newcommand{\ddmu}{\ddot{\mu}}
\newcommand{\GL}{\mathrm{GL}}
\newcommand{\Sym}{\mathrm{Sym}}

\newcommand{\Proj}{\mathrm{Proj}}
\newcommand{\Harm}{\mathrm{Harm}}

\def\bignorm#1{\left\lVert #1 \right\rVert}
\def\bignormop#1{\left\lVert #1 \right\rVert_{\mathrm{op}}}
\def\bignormop#1{\left\lVert #1 \right\rVert_{\mathrm{op}}}
\def\bignormnuc#1{\left\lVert #1 \right\rVert_{\mathrm{nuc}}}

\newcommand{\indicator}{\mathds{1}}

\newcommand{\Skew}{\mathrm{Skew}}
\newcommand{\supp}{\mathrm{supp}}
\renewcommand{\vec}{\mathrm{vec}}
\newcommand{\poly}{\mathrm{poly}}
\newcommand{\nuc}{\mathrm{nuc}}
\newcommand{\row}{\mathrm{row}}

\definecolor{darkblue}{rgb}{0.0,0.0,0.65}
\definecolor{darkred}{rgb}{0.65,0.0,0.0}
\definecolor{darkgreen}{rgb}{0.0,0.5,0.0}
\definecolor{tab:blue}{RGB}{31,119,180}  %
\definecolor{tab:red}{RGB}{214,39,40}  %
\definecolor{tab:green}{RGB}{44,160,44}  %
\definecolor{tab:orange}{RGB}{255,127,14}  %
\hypersetup{
	colorlinks = true,
	citecolor  = darkblue,
	linkcolor  = darkred,
	filecolor  = darkblue,
	urlcolor   = darkgreen,
}

\allowdisplaybreaks

\begin{document}

\runningtitle{A Nearly Instance-Wise Minimax-Optimal Estimator for Generalized Low-Rank Trace Regression}

\twocolumn[

\aistatstitle{\texttt{GL-LowPopArt}: A Nearly Instance-Wise Minimax-Optimal Estimator for Generalized Low-Rank Trace Regression}

\aistatsauthor{ Junghyun Lee \And Kyoungseok Jang \And Kwang-Sung Jun \And Milan Vojnovi\'{c} \And Se-Young Yun }

\aistatsaddress{ KAIST AI \And  CAU AI \And POSTECH CSE/AI \And LSE Statistics \And KAIST AI } ]

\begin{abstract}
	We present \texttt{GL-LowPopArt}, a novel Catoni-style estimator for generalized low-rank trace regression. Building on \texttt{LowPopArt}~\citep{jang2024lowrank}, it employs a two-stage approach: nuclear norm regularization followed by matrix Catoni estimation. We establish state-of-the-art estimation error bounds, surpassing existing guarantees~\citep{fan2019generalized,kang2022generalized}, and reveal a novel experimental design objective, $\GL(\pi)$.
    The key technical challenge is controlling bias from the nonlinear inverse link function, which we address with our two-stage approach.
    We prove a \textit{local minimax lower bound}, showing that our \texttt{GL-LowPopArt} enjoys instance-wise optimality up to the condition number of the ground-truth Hessian.
    Our method immediately achieves an improved Frobenius error guarantee for generalized linear matrix completion.
    We also introduce a new problem setting called \textbf{bilinear dueling bandits}, a contextualized version of dueling bandits with a general preference model.
    Using an explore-then-commit approach with \texttt{GL-LowPopArt}, we show an improved Borda regret bound over na\"{i}ve vectorization~\citep{wu2024dueling}.
\end{abstract}

\section{INTRODUCTION}
\label{sec:introduction}
Low-rank structures are ubiquitous across diverse domains, where the estimation of high-dimensional, low-rank matrices frequently pops up~\citep{chen2018lowrank}.
Beyond simply possessing a low-rank structure, real-world observations are often subject to nonlinearities.
One ubiquitous example is modeling discrete event occurrences by the Poisson point processes~\citep{mutny2021poisson,kingman1992poisson}, such as crime rate~\citep{shirota2017crime} and environmental modeling~\citep{heikkinen1999forest}.
In news recommendation and online ad placement, outputs are often quantized, representing categories such as ``click'' or ``no click''~\citep{bennett2007netflix,richardson2007click,matchbox,li2010news,li2012glm,mcmahan2013click}.
Other applications involve predicting interactions between multiple features, e.g., hotel-flight bundles~\citep{lu2021generalized}, online dating/shopping~\citep{jun2019bilinear}, protein-drug pair searching~\citep{luo2017drug}, graph link prediction~\citep{berthet2020graph}, stock return prediction~\citep{fan2019generalized}, and, more recently, general preference learning~\citep{zhang2024bilinear}, among others.
In these settings, it is natural to model the problem as matrix-valued covariates passed through a nonlinear regression model.
In particular, when the observations are sampled from the generalized linear model~\citep{glm}, these diverse problems fall under the umbrella of \textit{generalized low-rank trace regression}~\citep{fan2019generalized}, which we describe.

\paragraph{Problem Setting.}
$\bm\Theta_\star \in \sR^{d_1 \times d_2}$ is an unknown matrix of rank at most $r \ll d_1 \wedge d_2$, and $\gX \subseteq \sR^{d_1 \times d_2}$ is an arm-set (e.g., sensing matrices).
The learner's goal is to output $\widehat{\bm\Theta}$ that well-estimates $\bm\Theta_\star$ from some observations $\{(\mX_t, y_t)\}_{t \in [N]}$, collected as follows.

For a given budget $N \in \sN$, a \textit{sampling policy (design)} is a sequence $\bm\pi_N = (\pi_t)_{t \in [N]}$, where each $\pi_t \in \gP(\gX)$ is potentially history-dependent; see \citet[Chapter 4.6]{banditalgorithms} for the precise measure-theoretic formulation.
When the learner uses $\bm\pi_N$, at each time $t \in [N]$, she samples a $\mX_t \sim \pi_t$ and observes $y_t$ sampled from \textit{generalized linear model (GLM)} whose (conditional) density is given as follows:
\begin{equation}
\label{eqn:GLM}
    p(y_t | \mX_t; \bm\Theta_\star) \propto \exp\left( \frac{y_t \langle \mX_t, \bm\Theta_\star \rangle - m(\langle \mX_t, \bm\Theta_\star \rangle)}{g(\tau)} \right).
\end{equation}
Here, $m : \sR \rightarrow \sR$ is the log-partition function, $\tau$ is the dispersion parameter, $g : \sR \rightarrow \sR_{>0}$ is a fixed function, and the density is with respect to some known base measure (e.g., Lebesgue, counting).
We refer to $\mu := \dot{m}$ as the \textit{inverse link function}.
We assume that all components of the GLM, other than $\bm\Theta_\star$, are known to the learner; see Appendix~\ref{app:future}\textbf{(5)} for discussions on what happens under model misspecification.

Given $\bm\pi_N = (\pi_t)_{t \in [N]}$, we distinguish two setups: in the \textit{adaptive scenario}, each $\pi_t \in \gP(\gX)$ may depend on past observations, as is standard in interactive learning problems such as bandits~\citep{banditalgorithms} and active learning~\citep{settles2012active}, whereas in the \textit{nonadaptive (passive) scenario}, all policies are identical, i.e., $\pi_t = \pi$ for a known $\pi \in \gP(\gX)$, fixed before interaction begins.
In our setting, we omit $t$ dependence for brevity, since in the adaptive case our algorithm only switches policy once ($\pi_0$ for Stage~I and $\pi$ for Stage~II). We also remark in advance that our lower bound holds for any adaptive policy $\bm\pi_N$.

\paragraph{Related Works.}
Owing to its ubiquity, much work have been done in providing statistically and computationally efficient estimators for this problem, both generally~\citep{fan2019generalized,kang2022generalized} and in specific scenarios such as \textit{generalized linear matrix completion}~\citep{cai2013completion,cai2016completion,davenport2014completion,lafond2015lexponential,lafond2014finite,klopp2014general,klopp2015multinomial} and learning low-rank preference matrix~\citep{rajkumar2016comparison}.
Corresponding minimax lower bounds have also been proven that are tight with respect to rank $r$, dimension $d_1, d_2$, and sample size $N$; see Appendix~\ref{app:related-work} for a detailed survey.

\paragraph{Main Contributions.}
While prior work has made significant progress, a crucial aspect has been overlooked: the instance-specific nature of curvature.
To our knowledge, all existing analyses rely on worst-case curvature bounds,  obscuring the problem's true difficulty.
For example, known performance guarantees for generalized linear matrix completion depend inversely with $\min_{|z| \leq \gamma} \dmu(z)$, where $\gamma > 0$ is such that $\max_{i,j} \left| (\bm\Theta_\star)_{ij} \right| \leq \gamma$ and $\dmu$ is the derivative of the inverse link function~\citep{davenport2014completion,klopp2015multinomial}.
For instance, when $\mu(z) = (1 + e^{-z})^{-1}$, this leads to a dependence of $e^{\gamma}$~\citep{faury2020logistic}.
This is \emph{not} instance-wise guarantee in the sense that it arises from the worst-case $\dmu$ over the entry-wise domain $[-\gamma, \gamma]$, rather than adapting to the specific instance $\Theta_\star$.

Our contributions are as follows:
\begin{itemize}
    \item We propose \texttt{GL-LowPopArt}, an extension of \texttt{LowPopArt}~\citep{jang2024lowrank} to generalized low-rank trace regression, that consists of 1) Stage~I computes an initial estimate $\bm\Theta_0$, and 2) Stage~II refines it via matrix Catoni estimator~\citep{minsker2018matrix}.
    We prove its \textbf{\textit{instance-wise error rate}} for any design $\pi \in \gP(\gX)$ (\cref{thm:estimation-final}):
    \begin{equation}
        \bignormop{\widehat{\bm\Theta} - \bm\Theta_\star}^2 \lesssim_{\log} \frac{\GL(\pi; \bm\Theta_0)}{N} \lesssim \frac{d_1 \vee d_2}{N \lambda_{\min}(\mH(\pi; \bm\Theta_\star))},
    \end{equation}
    where $\GL(\pi)$ (\Cref{def:GL-Design}) is a new experimental design objective that captures the nonlinearity and the arm-set geometry, and $\lambda_{\min}(\mH(\pi; \bm\Theta_\star))$ is the minimum eigenvalue of the Hessian of the negative log-likelihood loss at $\bm\Theta_\star$.
    The proof requires careful bias control of one-sample estimators during matrix Catoni estimation~\citep{minsker2018matrix}.
    In the adaptive scenario, one can directly optimize the error bound as $\min_\pi \GL(\pi; \bm\Theta_0)$.
    (\cref{sec:gl-lowpopart})

    \item We prove the \textbf{\textit{first instance-wise minimax lower bound}} for generalized low-rank trace regression (\cref{thm:lower-bound}): for \textit{any} $\bm\pi_N$ and instance $\bm\Theta_\star$, there is a $\widetilde{\bm\Theta}_\star$ \textit{near} $\bm\Theta_\star$ such that
    \begin{equation}
        \bignorm{\widehat{\bm\Theta} - \widetilde{\bm\Theta}_\star}_F^2 \gtrsim \frac{r (d_1 \vee d_2)}{N \lambda_{\max}(\mH(\bm\pi_N; \bm\Theta_\star))},
    \end{equation}
    where $\lambda_{\max}(\cdot)$ is the maximum eigenvalue.
    When $\bm\pi_N = (\pi_t)_{t \in [N]}$ is of the form $\pi_t = \pi_0 \indicator[t \leq N_1] + \pi \indicator[t > N_1]$ with $N_1 = o(N)$, \texttt{GL-LowPopArt} is nearly instance-wise optimal, up to the condition number, $\frac{\lambda_{\max}(\mH(\pi; \bm\Theta_\star))}{\lambda_{\min}(\mH(\pi; \bm\Theta_\star))}$.
    (\cref{sec:lower-bound})
    
    \item As the first application of \texttt{GL-LowPopArt}, we revisit the problem of \textit{generalized linear matrix completion}~\citep{davenport2014completion,lafond2015lexponential,klopp2015multinomial} and prove an error rate with instance-specific curvature $\left( \min_{i,j} \dmu((\bm\Theta_\star)_{i,j}) \right)^{-1}$.
    This improves upon prior results that depend on the worst-case curvature.
    (\cref{sec:completion})

    \item As another application, we propose and tackle \textbf{\textit{bilinear dueling bandits}}, a new variant of generalized linear dueling bandits with the contextual bilinear preference model of \citet{zhang2024bilinear}.
    We propose a \texttt{GL-LowPopArt}-based explore-then-commit algorithm and prove its Borda regret upper bound (\Cref{thm:borda-bound}):
    \begin{equation}
        \Reg^B(T) \lesssim_{\log} \left( \GL_{\min}(\gX) \right)^{1/3} \left( \kappa_\star^B T \right)^{2/3},
    \end{equation}
    where $\kappa_\star^B$ is a new curvature-dependent quantity specific to each bandit instance.
    (\cref{sec:bilinear})
\end{itemize}

\section{TECHNICAL PRELIMINARIES}

\paragraph{Notation.}
For a $\mA \in \sR^{m \times n}$ with singular values $\sigma_1 \geq \cdots \geq \sigma_{\min\{m, n\}}$, $\bignorm{\mA}_\nuc := \sum_{i=1}^{\min\{m, n\}} \sigma_i$ is its nuclear norm, and $\bignormop{\mA} := \sigma_1$ is its operator (spectral) norm.
For $\mB \in \sR^{m \times n}$, their Frobenius inner product is defined as $\langle \mA, \mB \rangle := \tr(\mA^\top \mB)$.
For a symmetric $\mA \in \sR^{m \times m}$, $\lambda_i(\mA)$ is its $i$-th largest eigenvalue, $\lambda_{\max} := \lambda_1$, and $\lambda_{\min} := \lambda_m$.
On the positive semidefinite cone, define the Loewner order $\preceq$ as $\mA \preceq \mB$ if and only if $\mB - \mA$ is positive semidefinite.
For a $S > 0$, let us denote $\gB^{d_1 \times d_2}_i(S) := \{ \mX \in \sR^{d_1 \times d_2} : \bignorm{\mX}_i \leq S \}$ for $i \in \{ {\mathrm{op}}, \nuc, F \}$.
$\vec : \sR^{d_1 \times d_2} \rightarrow \sR^{d_1 d_2}$ performs column-wise stacking of a matrix into a vector, and $\vec^{-1}$ is its inverse.
$f(n) \lesssim g(n)$ and $f(n) \asymp g(n)$ indicates $f(n) \leq c g(n)$ and $c g(n) \leq f(n) \leq c' g(n)$ for some constants $c, c' > 0$, respectively; $\log$ in the subscript (e.g., $\lesssim_{\log}$) indicates up to logarithmic factors.
Denote $a \wedge b := \min(a, b)$ and $a \vee b := \max(a, b)$.
For a \( n \in \sN \), let \( [n] := \{1, 2, \dots, n\} \). 
For a set \( X \), \( \gP(X) \) is the set of all probability distributions on \( X \).

\paragraph{Assumptions.}

We assume the following for the parameter space $\Omega$:
\begin{assumption}
\label{assumption:omega}
    $\Omega \subseteq \sR^{d_1 \times d_2}$ is a linear subspace satisfying $\bm\Theta \in \Omega \Longrightarrow \Proj_r(\bm\Theta) \in \Omega$, where $\Proj_r(\bm\Theta) := \argmin_{\mathrm{rank}(\bm\Theta') \leq r} \bignorm{\bm\Theta - \bm\Theta'}_F^2$.
\end{assumption}
Some notable examples include $\sR^{d_1 \times d_2}$ and the set of (skew-)symmetric matrices when $d_1 = d_2$.

We assume the following for the arm-set $\gX$:
\begin{assumption}
\label{assumption:arm-set}
     $\mathrm{span}(\gX) = \sR^{d_1 \times d_2}$.
\end{assumption}
This is essential as otherwise one cannot hope to recover $\bm\Theta_\star$ in the direction of $\mathrm{span}(\gX)^\perp \neq \emptyset$.

We assume the following for $m(\cdot)$:
\begin{assumption}
\label{assumption:mu}
    $m : \sR \rightarrow \sR$ is three-times differentiable and convex.
    Moreover, the \textit{inverse link function} $\mu := \dot{m}$ satisfies the following two conditions:
    \begin{enumerate}
        \item[(a)] $R_s$-self-concordant
        for a known $R_s \in [0, \infty)$, i.e., $|\ddmu(z)| \leq R_s \dmu(z), \ \forall z \in \sR$,
        \item[(b)] $L_\mu := \sup_{\mX \in \gX, \bm\Theta \in \Omega} \dmu(\langle \mX, \bm\Theta \rangle) < \infty$.
    \end{enumerate}
\end{assumption}
This encompasses Gaussian ($m(z) = \frac{1}{2} z^2$), Bernoulli ($m(z) = \log(1 + e^z)$), and Poisson ($m(z) = e^z$).

\begin{algorithm2e*}[!t]
\caption{\texttt{GL-LowPopArt}}
\label{alg:gl-lowpopart}
\textbf{Input:} Sample sizes $(N_1, N_2)$, designs $\pi_0 \in \gP(\gX)$ for Stage I, Regularization coefficient $\lambda_{N_1} > 0$\;
\tcc{Stage I: Nuclear Norm-regularized Initial Estimator}
\For{$t = 1, 2, \cdots, N_1$}{
    Pull $\mX_t \sim \pi_0$ and receive $y_t \sim p(\cdot | \mX_t; \bm\Theta_\star)$\;
}

Compute an initial estimator $\bm\Theta_0$ using $\{(\mX_t, y_t)\}_{t=1}^{N_1}$ collected via $\pi_0$\;

If in the adaptive scenario, compute $\pi$ for Stage~II; else, set $\pi = \pi_0$\;

\tcc{Stage II: Generalized Linear Matrix Catoni Estimation}
\For{$t = N_1 + 1, N_1 + 2, \cdots, N_1 + N_2$}{
    Pull $\mX_t \sim \pi$ and receive $y_t \sim p(\cdot | \mX_t; \bm\Theta_\star)$\;

    Compute the matrix one-sample estimators:
    $$
        \widetilde{\bm\Theta}_t \gets \vec^{-1}\left( \tilde{\bm\theta}_t \right), \quad 
        \tilde{\bm\theta}_t \gets \mH(\pi; \bm\Theta_0)^{-1} \left( y_t - \mu(\langle \mX_t, \bm\Theta_0 \rangle) \right) \vec(\mX_t)
    $$ \label{line:one-sample}
}
$\bm\Theta_1 \gets \Proj_\Omega\left( \bm\Theta_0 + \frac{1}{N_2} \left( \sum_{t=N_1+1}^{N_1+N_2} \tilde{\psi}_\nu(\widetilde{\bm\Theta}_t) \right)_{\mathrm{ht}} \right)$ with $\nu := \sqrt{\frac{2}{5 g(\tau) \GL(\pi; \bm\Theta_0) N_2} \log\frac{4(d_1+d_2)}{\delta}}$\;\label{line:update_theta_1} 

Let $\bm\Theta_1 = \mU \mD \mV^\top$ be its SVD and $\widetilde{\mD}$ be $\mD$ after zeroing out singular values at most $\sqrt{40} g(\tau) \GL(\pi; \bm\Theta_0) \nu$\;\label{line:svd_theta_1} 

\textbf{Return:} $\widehat{\bm\Theta} := \mU \widetilde{\mD} \mV^\top$\;
\end{algorithm2e*}

\section{\texttt{GL-LowPopArt}: A MATRIX CATONI-STYLE ESTIMATOR}
\label{sec:gl-lowpopart}

For $\pi \in \gP(\gX)$ and $\bm\Theta \in \sR^{d_1 \times d_2}$, we define the \textit{vectorized design/Hessian matrix}:
\begin{align}
    \mV(\pi) &:= \E_{\mX \sim \pi}[\vec(\mX) \vec(\mX)^\top], \label{eqn:V} \\
    \mH(\pi; \bm\Theta) &:= \E_{\mX \sim \pi}[\dmu(\langle \mX, \bm\Theta \rangle) \vec(\mX) \vec(\mX)^\top], \label{eqn:H}
\end{align}
where $\mH(\pi; \bm\Theta)$ is the Hessian of the population negative log-likelihood: $\bm\Theta \mapsto - g(\tau) \E_{\mX \sim \pi}[\log p(y | \mX; \bm\Theta)]$.
Observe that $\mV(\pi) = \mH(\pi; \bm\Theta)$ when $\mu(z) = z$.
We will also define an instance-specific quantity:
\begin{equation}
    \kappa_\star := \min_{\mX \in \gX} \dmu(\langle \mX, \bm\Theta_\star \rangle) > 0.
\end{equation}

We now recall some notions from matrix analysis to describe the matrix Catoni estimator~\citep{catoni2012,minsker2018matrix}.
For any $f : \sR \rightarrow \sR$ and symmetric $\mM \in \sR^{d \times d}$, we define $f(\mM)$ as $f(\mM) := \mU \diag(\{ f(\lambda_i) \}_{i \in [d]}) \mU^\top$, where $\mM = \mU \bm\Lambda \mU^\top$ with $\bm\Lambda = \diag(\{ \lambda_i \}_{i \in [d]})$ being the eigenvalue decomposition of $\mM$, i.e., $f$ acts on its spectrum.
The \textit{Hermitian dilation}~\citep{tropp2015survey}, $\gH : \sR^{d_1 \times d_2} \rightarrow \sR^{(d_1 + d_2) \times (d_1 + d_2)}$, is defined as
\begin{equation}
    \gH(\mA) := 
    \begin{bmatrix}
        \vzero_{d_1 \times d_1} & \mA \\
        \mA^\top & \vzero_{d_2 \times d_2}
    \end{bmatrix}.\label{eqn: hermitiaon dilation}
\end{equation}
The \textit{influence function}~\citep{catoni2012} is defined as
\begin{equation}
    \psi(x) :=
    \begin{cases}
        \log(1 + x + x^2 / 2), \quad &x \geq 0,\\
        -\log(1 - x + x^2 / 2), \quad &x < 0.
    \end{cases}\label{eqn: influence function}
\end{equation}
We then define $\tilde{\psi}_\nu(\mA) := \frac{1}{\nu} \psi(\nu \gH(\mA))_{\mathrm{ht}}$ for $\nu > 0$, where for $\mM \in \sR^{(d_1 + d_2) \times (d_1 + d_2)}$, we define its \textit{horizontal truncation} as $\mM_{\mathrm{ht}} := \mM_{1:d_1,d_1+1:d_1+d_2}$.

\subsection{Overview of \texttt{GL-LowPopArt}}
\label{sec:overview}

We present \texttt{GL-LowPopArt} (Generalized Linear LOW-rank POPulation covariance regression with hARd Thresholding; Algorithm~\ref{alg:gl-lowpopart}), a novel estimator for generalized low-rank trace regression.
\texttt{GL-LowPopArt} consists of two stages: the first stage provides a rough, initial estimate, and the second stage refines it via matrix Catoni estimator~\citep{minsker2018matrix}.
It takes two designs $\pi_0$ and $\pi$ as inputs for Stage~I and II, respectively.
When the learner is in the adaptive learning scenario, she can choose $\pi$ dependent on the data collected during Stage~I.
If not, she can set $\pi_0 = \pi$.

Stage~I uses the observations $\{(\mX_t, y_t)\}_{t=1}^{N_1}$ collected via $\pi_0$ to compute a ``sufficiently good'' initial estimate $\bm\Theta_0$.
In Stage~II, for each sample $(\mX_t, y_t)$ colllected via another design $\pi$ for $t = N_1 + 1, \cdots, N_1 + N_2$, \texttt{GL-LowPopArt} constructs one-sample estimator $\widetilde{\bm\Theta}_t$ such that $\E[\widetilde{\bm\Theta}_t] \approx \bm\Theta_\star - \bm\Theta_0$ (line~\ref{line:one-sample}).
Then, the $\Omega$-projected matrix Catoni estimator $\bm\Theta_1$ is computed (line \ref{line:update_theta_1}).
The final estimator $\widehat{\bm\Theta}$ is obtained by singular value thresholding $\bm\Theta_1$ (line \ref{line:svd_theta_1}).
Note that the algorithm is parameter-free in the sense that the algorithm does not require the knowledge of the rank $r$ of $\bm\Theta_\star$.

The final estimation error guarantee is solely due to matrix Catoni estimation~\citep{minsker2018matrix} in Stage~II, yet unlike the linear trace regression~\citep{jang2024lowrank}, we \textit{require} for the initial estimate $\bm\Theta_0$ to be asymptotically consistent.
This is the main technical novelty/challenge for the algorithm design and analysis.

Before presenting the error rate of \texttt{GL-LowPopArt}, let us first define a new design objective for Stage~II:
\begin{definition}[GL-Design]
\label{def:GL-Design}
    For any $\pi \in \gP(\gX)$ and $\bm\Theta_\star \in \sR^{d_1 \times d_2}$, the \textbf{GL-Design objective} $\GL(\pi; \bm\Theta_\star)$ is defined as follows:
    \begin{equation}
        \GL(\pi; \bm\Theta_\star) := \max\{ H^{(\row)}(\pi; \bm\Theta_\star), H^{(\col)}(\pi; \bm\Theta_\star) \},
    \end{equation}
    where
    \begin{align}
        H^{(\row)}(\pi; \bm\Theta_\star) &:= \lambda_{\max}\left( \sum_{{\color{blue}m}=1}^{d_1} \mD_{\color{blue}m}^{(\row)}(\pi; \bm\Theta_\star) \right), \\
        \mD_{\color{blue}m}^{(\row)}(\pi; \bm\Theta_\star) &:=
         [(\mH(\pi; \bm\Theta_\star)^{-1})_{jk}]_{j,k \in {\color{purple}\gI_{\color{blue}m}^{(\row)}}},\\
        H^{(\col)}(\pi; \bm\Theta_\star) &:= \lambda_{\max}\left( \sum_{{\color{blue}m}=1}^{d_2} \mD_{\color{blue}m}^{(\col)}(\pi; \bm\Theta_\star) \right),\\
        \mD_{\color{blue}m}^{(\col)}(\pi; \bm\Theta_\star) &:= (\mH(\pi; \bm\Theta_\star)^{-1})_{{\color{teal}\gI_{\color{blue}m}^{(\col)}}, {\color{teal}\gI_{\color{blue}m}^{(\col)}}},
    \end{align}
    and the index sets defined as ${\color{purple}\gI_{\color{blue}m}^{(\row)}} := \{d_1 (\ell-1) + {\color{blue}m}: \ell \in [d_2]\}$ and ${\color{teal}\gI_{{\color{blue}m}}^{(\col)}} := [d_1 ({\color{blue}m}-1)+1:d_1 {\color{blue}m}]$.\footnote{A nice illustration of $\mD^{(\row)}_{\color{blue}m}$ and $\mD^{(\col)}_{\color{blue}m}$ is provided in Figure 1 of \citet{jang2024lowrank}.}
\end{definition}

We now present the error rate of \texttt{GL-LowPopArt}, which holds for \textit{any} $\pi_0, \pi$, adaptive or nonadaptive:
\begin{theorem}[Error Rate of \texttt{GL-LowPopArt}]\label{thm:estimation-final}
    Let $\delta \in (0, 1)$, $\gX \subseteq \gB_{\mathrm{op}}^{d_1 \times d_2}(1)$,\footnote{This has been considered before in low-rank and bilinear bandits~\citep{jang2024lowrank,jun2019bilinear}.} and denote
    \begin{equation}
        \gE(N_2, \delta) := \sqrt{\frac{g(\tau) \GL(\pi; \bm\Theta_\star)}{N_2} \log\frac{d_1 + d_2}{\delta}}.
    \end{equation}

    The final estimator $\widehat{\bm\Theta}$ from Stage~II satisfies the following w.p. at least $1 - \delta$: $\rank(\widehat{\bm\Theta}) \leq r$ and
    \begin{equation}
        \bignormop{\widehat{\bm\Theta} - \bm\Theta_\star} \lesssim \gE(N_2, \delta),
    \end{equation}
    \textbf{provided} that $\bm\Theta_0$ from Stage~I satisfies $\bignormnuc{\bm\Theta_\star - \bm\Theta_0}^2 \lesssim \kappa_\star g(\tau)$, and $R_s \bignormnuc{\bm\Theta_\star - \bm\Theta_0}^2 \lesssim \lambda_{\min}(\mH(\pi; \bm\Theta_0)) \gE(N_2, \delta)$ w.p. at least $1 - \delta/2$.
\end{theorem}

\begin{table*}[t]
    \centering
    \normalsize
    \setlength{\tabcolsep}{5pt}
    \caption{\label{tab:comparison}
    Comparison of squared Frobenius error rates under our \Cref{assumption:omega,assumption:arm-set,assumption:mu}, and $\pi_0 = \pi$.
    We denote $\gE(N) = \frac{r (d_1 \vee d_2)^2}{N}$ for brevity. We also assume that other assumptions required by the prior works~\citep{fan2019generalized,kang2022generalized} hold, additional to our assumptions, as long as they are not conflicting.
    For our \texttt{GL-LowPopArt}, we employ optimal GL-design in Stage~II.
    }
    \begin{adjustbox}{max width=\textwidth}
    \begin{threeparttable}
    \begin{tabular}{|c|c|c|c|} \hline
        {\bf Estimator} & $\gX = \gB^{d_1 \times d_2}_F(1)$ & $\gX = \gB^{d_1 \times d_2}_{\mathrm{op}}(1)$ & $\gX = \gM$ \\
        \hline
        \makecell{Nuclear Penalized MLE\tnote{*} \\ {\small \citep[Theorem 2]{fan2019generalized}}} &
        $\kappa_\star^{-2} d_1 d_2 (d_1 \wedge d_2) \gE(N)$ &
        $\kappa_\star^{-2} (d_1 \wedge d_2) \gE(N)$ &
        $\kappa_\star^{-2} d_1 d_2 (d_1 \wedge d_2) \gE(N)$ \\
        \hline
        \makecell{Stein Estimator\tnote{$\dagger$} \\ {\small \citep[Theorem 4.1]{kang2022generalized}}} &
        $\kappa_\star^{-2} (d_1 \wedge d_2) \gE(N)$ &
        $\kappa_\star^{-2} (d_1 \vee d_2) \gE(N)$ &
        N/A \\
        \hline
        \rowcolor{gray!12}
        \begin{tabular}{@{}c@{}} \texttt{GL-LowPopArt} \\ {\small\bf (Ours, \cref{thm:estimation-final})}\end{tabular} &
        $\kappa_\star^{-1} (d_1 \wedge d_2) \gE(N)$ &
        $\kappa_\star^{-1} \gE(N)$ &
        $\Harm(\bm\Theta_\star)^{-1} (d_1 \wedge d_2) \gE(N)$ \\
        \hline
    \end{tabular}
    \begin{tablenotes}
        \item[$*$] We choose $\pi = \argmax_{\pi \in \gP(\gX)} \lambda_{\min}(\mV(\pi))$ and compute the error bounds via \citet[Appendix D]{jang2024lowrank}. Note that this E-optimal design is the best one could do with the given information.
        \item[$\dagger$] This requires $\pi$ to have a continuously differentiable density over $\gX$, and thus, not applicable to $\gM$.
    \end{tablenotes}
    \end{threeparttable}
    \end{adjustbox}
\end{table*}

\subsection{\texorpdfstring{$\GL(\pi; \bm\Theta_\star)$: A New Experimental Design}{GL: A New Experimental Design}}
\label{sec:adaptive}

\paragraph{GL-Design Objective.}
$\GL(\pi; \bm\Theta_\star)$ captures two problem-specific characteristics: nonlinearity due to $\mu$ via the Hessian $\mH(\pi; \bm\Theta_\star)$ and the geometry of $\gX$.
When $\mu(z) = z$, it reduces to the linear design objective~\citep[Theorem 3.4]{jang2024lowrank}.

The following proposition shows that the GL-Design objective is closely related to the Hessian $\mH(\pi; \bm\Theta_\star)$:
\begin{proposition}
\label{prop:improve}
    For any $\pi \in \gP(\gX)$,
    \begin{equation}
        \frac{d_1 d_2 (d_1 \vee d_2)}{\tr(\mH(\pi; \bm\Theta_\star))} \leq \GL(\pi; \bm\Theta_\star) \le \frac{d_1 \vee d_2}{\lambda_{\min}(\mH(\pi; \bm\Theta_\star))}.
    \end{equation}
\end{proposition}

\paragraph{Optimal GL-Design.}
In the adaptive setting, the learner may choose $\pi$ to minimize the GL-Design objective for Stage~II, which would then minimize the error rate as in \Cref{thm:estimation-final}.
This is in essence a optimal design problem~\citep{pukelsheim2005design}, in which the learner chooses $\pi_{\GL} \in \argmin_{\pi \in \gP(\gX)} \GL(\pi; \bm\Theta_\star)$.
Let $\GL_{\min}(\gX) := \GL(\pi_{\GL}; \bm\Theta_\star)$ be the optimal value.

Two practical concerns arise.
First, $\bm\Theta_\star$ is unknown, so the ideal problem is not directly solvable.
In practice, one can instead solve $\argmin_{\pi \in \gP(\gX)} \GL(\pi; \bm\Theta_0)$.
Given that $\bm\Theta_0$ is sufficiently close to $\bm\Theta_\star$, by self-concordance and the fact that $\gX \subseteq \gB_{\mathrm{op}}^{d_1 \times d_2}(1)$, we have that $\mH(\pi; \bm\Theta_0) \approx \mH(\pi; \bm\Theta_\star)$ (see our \Cref{lem:jun5}).
This then ensures that $\GL_{\min}(\gX)$ is retained up to a constant factor, i.e., the optimal GL-Design w.r.t. $\bm\Theta_0$ still captures the true problem characteristics.

The second concern is regarding the computational complexity of solving the optimal GL-design.
This can indeed be efficiently solved, as $\pi \mapsto \GL(\pi; \bm\Theta_0)$ is convex.
Implementation-wise, one can formulate it into an epigraph form via Schur complement~\citep{cvxbook} and use any available convex optimization solver, e.g., CVXPY~\citep{diamond2016cvxpy,agrawal2018rewriting}; see \Cref{app:complexity} for more detailed discussions.

We conclude with upper bounds of $\GL_{\min}$ for Frobenius/operator unit balls and matrix-completion basis:
\begin{proposition}[Upper bounds of $\GL_{\min}$]
\label{prop:unit-ball}
    Let $\gM := \left\{ \ve_i \tilde{\ve}_j^\top : (i,j)\in[d_1]\times[d_2] \right\}$ be the matrix-completion basis. Then, with $\sum_{i,j} := \sum_{i=1}^{d_1}\sum_{j=1}^{d_2}$,
    \begin{align}
        &\GL_{\min}\!\left(\gB_F^{d_1\times d_2}(1)\right)
        \lesssim \kappa_\star^{-1} d_1 d_2 (d_1\vee d_2),\\
        &\GL_{\min}\!\left(\gB_{\mathrm{op}}^{d_1\times d_2}(1)\right)
        \lesssim \kappa_\star^{-1} (d_1\vee d_2)^2,\\
        &\GL_{\min}(\gM)
        \leq \Harm(\bm\Theta_\star)^{-1} d_1 d_2 (d_1 \vee d_2),
    \end{align}
    where $\Harm(\bm\Theta_\star) := d_1 d_2 \bigg(\sum_{i,j}
    \dmu\!\left((\bm\Theta_\star)_{ij}\right)^{-1}\bigg)^{-1}$ is the harmonic mean of the entry-wise curvatures.
\end{proposition}
\begin{proof}
    The first two claims follow directly from \Cref{prop:improve} and \citet[Appendix D]{jang2024lowrank}
    The third claim follows from \Cref{prop:improve} and solving the optimal design explicitly; see \Cref{app:unit-ball}.
\end{proof} 

\subsection{Comparison with Prior Works}
The comparison with prior works is rather intricate, as different works consider different assumptions, especially regarding $\gX$ and the design $\pi$.
For clarity, we will make the comparisons under our \Cref{assumption:omega,assumption:arm-set,assumption:mu} and that the GLM is bounded; we defer detailed discussions regarding the discrepancies in the settings compared to prior works to \Cref{app:comparison}.

\Cref{tab:comparison} summarizes the comparison with prior state-of-the-arts: nuclear penalized estimator~\citep[Theorem 2]{fan2019generalized} and Stein's lemma-based estimator~\citep[Theorem 4.1]{kang2022generalized}.
For fair comparison, we consider the bound achievable by optimal design in all cases whenever applicable.
Our \texttt{GL-LowPopArt} achieves improvements in curvature-dependent quantites, and sometime even in dimensions.\footnote{This is similar to how \citet{jang2024lowrank} improved over \citet{koltchinskii2011nuclear} in linear trace regression.}
This suggests that our optimal GL-Design captures the problem-specific characteristics more effectively than prior works.

\begin{remark}[Unbounded $\gX$]
    It is common to consider the case when the (matrix-valued) covariates are drawn from $\sR^{d_1 \times d_2}$ with $\pi$ being some tail-bounded distribution~\citep{fan2019generalized,negahban2011estimation,lu2021generalized,kang2022generalized}.
    Thus, one may wonder if our results hold for unbounded $\gX$.
    In \cref{app:unbounded}, we show that when $\mu(z) = z$, our results extend to unbounded $\gX$ and obtain better error rates than the nuclear penalized estimator in multivaraite regression~\citep{negahban2011estimation}.
    However, when $\mu$ is nonlinear, our current analysis relies heavily on the fact that $\gX \subseteq \gB_{\mathrm{op}}^{d_1 \times d_2}(1)$.
    On the bright side, our proof suggests that under proper moment conditions on $\pi$, our results may extend to unbounded $\gX$ with potentially stronger requirement on $\bm\Theta_0$.
    We leave this direction for future work.
\end{remark}

\subsection{\texorpdfstring{Proof Sketch of Theorem~\ref{thm:estimation-final}}{Proof Sketch of Theorem 3.1}}
\label{sec:stage-ii}
The proof is inspired by \citet[Theorem 3.1]{jang2024lowrank}, but some crucial differences make the extension non-trivial.
The full proof is deferred to \Cref{app:estimation-final}.

For simplicity, let us denote $\mH := \mH(\pi; \bm\Theta_0)$ in this proof sketch, and ignore $\Proj_\Omega$.
The core idea is to apply matrix Catoni estimation of \citet{minsker2018matrix} to $\bm\Theta_\star - \bm\Theta_0$, by constructing vectorized one-sample estimators (line~\ref{line:one-sample}) as follows:
\begin{equation}
    \tilde{\bm\theta}_t = \mH^{-1} \left( y_t - \mu(\langle \mX_t, \bm\Theta_0 \rangle) \right) \vec(\mX_t).
\end{equation}
For matrix Catoni to apply, the one-sample estimators should be unbiased, i.e., they should satisfy $\E[\tilde{\bm\theta}_t] = \vec(\bm\Theta_\star - \bm\Theta_0)$.
However, note that
\begin{equation}
    \E[\tilde{\bm\theta}_t] = \mH^{-1} \E_\mX\left[ \left( \mu(\langle \mX, \bm\Theta_\star \rangle) - \mu(\langle \mX, \bm\Theta_0 \rangle) \right) \vec(\mX) \right].
\end{equation}
When $\mu(z) = z$ as in \citet{jang2024lowrank}, above reduces to $\vec(\bm\Theta_\star - \bm\Theta_0)$, making $\tilde{\bm\theta}_t$ its unbiased estimator.
However, when $\mu$ is nonlinear, it becomes \emph{biased}.

The key technical novelty is appropriately dealing with this bias, inspired by recent progress in logistic and generalized linear bandits~\citep{abeille2021logistic,jun2021confidence,lee2024glm}.
Specifically, by the first-order Taylor expansion of $\mu$  and self-concordance (Assumption~\ref{assumption:mu}(a)), we show the following (\textbf{Eqn.~(C.1)}):
\begin{equation}
    \bignorm{\E[\widetilde{\bm\Theta}_t] - (\bm\Theta_\star - \bm\Theta_0)}_F \lesssim R_s \bignormnuc{\bm\Theta_\star - \bm\Theta_0}^2.
\end{equation}
We remark that the sample splitting approach\footnote{run Stage~II with $N_2/2$ samples with $\vzero$ to obtain $\bm\Theta_0$, then Stage~II again using the remaining samples and $\bm\Theta_0$} of \texttt{Warm-LowPopArt}~\citep[Algorithm 2]{jang2024lowrank} fails due to this bias.
On the other hand, if the initial estimator $\bm\Theta_0$ is sufficiently good (in the precise sense as stated in the theorem statement), then this bias becomes negligible.
Lastly, the GL-Design objective $\GL(\pi; \bm\Theta_\star)$ arises from the matrix variance statistics for $\widetilde{\bm\Theta}_t$'s, and we conclude with the error rate for matrix Catoni estimator~\citep[Corollary 3.1]{minsker2018matrix}.
\qed

\subsection{Stage I via Nuclear Penalized MLE}
\label{sec:nuclear}
For $\bm\Theta_0$ from Stage~I, we consider the standard \emph{nuclear penalized MLE}, which has been extensively studied in generalized low-rank trace regression literature~\citep{fan2019generalized,lu2021generalized,kang2022generalized}.
Concretely, using $N_1$ samples from $\pi_0$,
\begin{align}
    \bm\Theta_0 &\gets \argmin_{\bm\Theta \in \Omega} \gL_{N_1}(\bm\Theta)
    + \lambda_{N_1} \bignormnuc{\bm\Theta}, \\
    \gL_{N_1}(\bm\Theta) &:= \frac{1}{N_1} \sum_{t=1}^{N_1} \left\{ m(\langle \mX_t, \bm\Theta \rangle) - y_t \langle \mX_t, \bm\Theta \rangle \right\}.
\end{align}

We first present the its error rate:
\begin{theorem}[Error Rate of Stage~I]
\label{thm:estimation-0}
    Let $\delta \in (0, 1)$, $\gX \subseteq \gB_{\mathrm{op}}^{d_1 \times d_2}(1)$, and set $\lambda_N \asymp \sqrt{\frac{C}{N_1} \log\frac{d_1 + d_2}{\delta}}$.\footnote{For bounded GLMs, $C \asymp g(\tau) \E_{\mX \sim \pi} \left[ \dmu(\langle \mX, \bm\Theta_\star \rangle) \right]$, and for GLMs with finite subexponential norm of $K$, $C \asymp K^2$; see \Cref{prop:lambda}.}
    Then, the following holds w.p. at least $1 - \delta/2$:
    \begin{equation}
        \bignorm{\bm\Theta_0 - \bm\Theta_\star}_F \lesssim \frac{1}{\lambda_{\min}(\mH(\pi; \bm\Theta_\star))} \sqrt{\frac{C r}{N_1} \log\frac{d_1 + d_2}{\delta}},
    \end{equation}
    \textbf{provided} that $N_1 \gtrsim \frac{C (d_1 + d_2) r}{\lambda_{\min}(\mH(\pi; \bm\Theta_\star))^2} \log\frac{d_1 + d_2}{\delta}$.
\end{theorem}
 Combining the above with \Cref{thm:estimation-0}, we have the following requirement on $N_1$ and $N_2$ as follows:
\begin{equation}
\label{eqn:N1-requirement}
    N_1 \gtrsim R_s \frac{C r^2}{\lambda_{\min}(\mH(\pi; \bm\Theta_\star))^3} \sqrt{\frac{N_2 \log\frac{d_1 + d_2}{\delta}}{g(\tau) \GL(\pi; \bm\Theta_\star)}},
\end{equation}
i.e., asymptotically, the requirement for Stage~I is negligible compared to that for Stage~II. Also, note that the requirement vanishes when $R_s = 0$, which mirrors the results in \citet{jang2024lowrank} for linear trace regression.

\begin{proof}[Proof Sketch of \Cref{thm:estimation-0}]
    We follow the general framework for analyzing M-estimators with decomposable regularizers~\citep[Chapter 9]{wainwright2019high}, introducing notable analytical improvements over \citet[Theorem 2]{fan2019generalized}. The complete proof is deferred to Appendix~\ref{app:estimation-0}.
    
    First, leveraging the boundedness of the covariates and matrix Bernstein inequalities~\citep{tropp2015survey,tropp2012matrix}, we derive a tight choice for the regularization coefficient $\lambda_{N_1}$ for GLMs with bounded or sub-exponential responses (\textbf{\Cref{prop:lambda}}). This approach eliminates a $d_1 \vee d_2$ factor present in the double covering argument of \citet[Lemma 1]{fan2019generalized} and accommodates a broader class of GLMs, such as Poisson, which their assumptions preclude.\footnote{\citet[Lemma 2]{fan2019generalized} requires $\|\bm\Theta_\star\|_F \gtrsim \sqrt{d_1 \vee d_2}$ and $|\ddmu(z)| \leq \frac{1}{|z|}$ for $|z| > 1$ (conditions C4 and C5).}

    Next, we establish the restricted strong convexity (RSC) of the regularized loss function (\textbf{\Cref{lem:RSC}}). By combining the self-concordance of the link function (\Cref{assumption:mu}(a))~\citep{ostrovskii2021self} with standard empirical process theory~\citep{ledoux-talagrand} and the matrix Bernstein inequality yields the requisite RSC condition. The final error bound then immediately follows from standard M-estimation guarantees~\citep[Corollary 9.20]{wainwright2019high}.
\end{proof}

\begin{remark}[E-Optimal Design]
\label{rmk:e-optimal}
    Before Stage~I, one can consider the E-optimal design~\citep{pukelsheim2005design}: $\pi_0 \in \argmax_{\pi \in \gP(\gA)} \lambda_{\min}(\mV(\pi))$.
    This would optimize the error rate of Stage~I, further alleviating the overall sample size requirements.
    We defer the details of to \Cref{app:stage-i}.
\end{remark}

\section{\texorpdfstring{LOCAL MINIMAX LOWER BOUND FOR $\lVert \cdot \rVert_F$}{LOCAL MINIMAX LOWER BOUND FOR FROBENIUS ERROR}}
\label{sec:lower-bound}
In this section, we prove a \textit{local (instance-wise)} minimax lower bound on the estimation error for generalized low-rank trace regression in the intersection of rank and nuclear norm balls.
For each instance $\bm\Theta_\star$ with $\rank(\bm\Theta_\star) \leq r$ and $\bignormnuc{\bm\Theta_\star} \leq S_*$ for some $S_* > 0$, define its local neighborhood of radius $\varepsilon > 0$ as
\begin{align}
    &\gN(\bm\Theta_\star; \varepsilon, r, S_*) := \left\{ \bm\Theta \in \Theta(r, S_*) : \bignorm{\bm\Theta - \bm\Theta_\star}_F \leq \varepsilon \right\}, \\
    &\Theta(r, S_*) := \left\{ \bm\Theta \in \sR^{d_1 \times d_2} : \rank(\bm\Theta) \leq r, \ \bignormnuc{\bm\Theta} \leq S_* \right\}.
\end{align}
$\Theta(r, S_*)$ has been considered before in the context of minimax lower bound by \citet{rohde2011estimation}, similar to that of sparse regression in the intersection of $\ell_0$ and $\ell_1$-ball constraints~\citep[Theorem 5.3]{rigollet2011sparse}.

We now present our generic lower bound for generalized linear trace regression encompassing active scenarios:
\begin{theorem}[{\small Local Minimax Lower Bound}]\label{thm:lower-bound}
    Let $\gA \subseteq \gB^{d_1 \times d_2}_F(1)$, and $S_* > 0, r \geq 1$ be such that $\frac{S_*^2}{r} \geq \gamma$ for some $\gamma > 0$.
    Also, suppose that $N \geq \frac{R_s^2 \log 2}{2^{10} e} \frac{r (d_1 \vee d_2) g(\tau)}{\lambda_{\max}(\mH(\pi; \bm\Theta_\star))}.$
    Then, there exist $C_1, C_2(\gamma) > 0$\footnote{$C_1$ is a universal constant, and $C_2 = \frac{C_2' \gamma}{(1 + \sqrt{\gamma})^2}$ for a universal constant $C_2' > 0$.} and $c \in (0, 1)$ such that the following holds:
    for \textit{each} $\bm\pi_N$ and $\bm\Theta_\star \in \Theta(r, S_*)$ with $\bignorm{\bm\Theta_\star}_F^2 \geq \frac{9 \gamma}{8}$, there exists a \textbf{small enough} $\varepsilon = \varepsilon(\bm\pi_N, \bm\Theta_\star) > 0$ such that $\inf_{\widehat{\bm\Theta}} \sup_{\widetilde{\bm\Theta}_\star \in \gN_\star} \sP_{\bm\pi_N, \widetilde{\bm\Theta}_\star}\left( \gE \right) \geq c$, where
    \begin{align}
         &\gE := \left\{ \bignorm{\widehat{\bm\Theta} - \widetilde{\bm\Theta}_\star}_F^2 \geq \frac{C_2 g(\tau) r (d_1 \vee d_2)}{N \lambda_{\max}(\mH(\bm\pi_N; \bm\Theta_\star)) S_*^2} \right\}, \\
         &\mH(\bm\pi_N; \bm\Theta_\star) := \frac{1}{N} \sum_{t=1}^N \mH(\pi_t; \bm\Theta_\star),
    \end{align}
    $\gN_\star := \gN(\bm\Theta_\star; \varepsilon, r, S_*)$, and $\sP_{\bm\pi_N, \widetilde{\bm\Theta}_\star}$ is the probability measure of the $N$ observations under $\bm\pi_N$ and $\widetilde{\bm\Theta}_\star$.
\end{theorem}
\begin{proof}[Proof Sketch]
    We mainly utilize the many hypotheses technique of \citet[Chapter 2]{tsybakov09} for high-probability minimax lower bound; see also \citet{yang1999minimax}.
    One key technical novelty is the construction of a \textit{local} packing $\Theta_{r,\varepsilon,\beta} \subset \Theta(r, S_*)$ around the given instance $\bm\Theta_\star$.
    Then, we carefully expand the $\KL$ between two GLMs from the packing by utilizing its Bregman divergence form~\citep{lee2024logistic} and self-concordance of $\mu$ (Assumption~\ref{assumption:mu}(a)), which leads to the instance-specific quantity $\lambda_{\max}(\mH(\pi; \bm\Theta_\star))^{-1}$.
    Also, note that we don't explicitly require any restricted isometry assumption~\citep[Eqn. (2.4)]{koltchinskii2011nuclear}.
    Refer to \cref{app:lower-bound} for the full proof.
    
    This notably deviates from \citet[Theorem 5]{rohde2011estimation}, where they considered a packing around $\bm\Theta_\star = \vzero$ for linear trace regression.
    When $\mu(z) = z$, this still results in a tight lower bound as the problem difficulty is uniform across all $\bm\Theta_\star \in \Theta(r, S_*)$.
\end{proof}

\paragraph{Instance-Specific Nature.}
Our lower bound scales inversely with the ``optimistic'' instance-specific curvature, $\lambda_{\max}(\mH(\pi; \bm\Theta_\star))^{-1}$, thereby capturing how problem difficulty varies across different instances $\bm\Theta_\star$.
Here, ``optimistic'' in the sense that it reflects the most favorable local curvature at $\bm\Theta_\star$, in contrast to worst-case bounds based on $\lambda_{\min}$.
To the best of our knowledge, this is the \emph{first} result in generalized linear trace regression that exhibits such an instance-wise dependency.
This mirrors the local minimax lower bounds established for logistic bandits~\citep[Theorem 2]{abeille2021logistic} and online LQR~\citep[Theorem 1]{simchowitz2020optimal}, which also highlight instance-specific complexities.
In contrast, classical worst-case minimax lower bounds for generalized linear trace regression and matrix completion~\citep{koltchinskii2011nuclear,rohde2011estimation,davenport2014completion,lafond2015lexponential,taki2021minimax} fail to capture such local variations.

\paragraph{Near Instance-wise Optimality.}
Our lower bound (\cref{thm:lower-bound}) holds for \emph{each} sampling policy $\bm\pi_N$, including adaptive ones. In other words, optimality here is defined relative to each instance $(\bm\Theta_\star, \bm\pi_N)$, where the policy itself forms part of the instance. To clarify in what sense \texttt{GL-LowPopArt} is optimal, we focus on the class of \emph{single-switch policies}, $\pi_t = \pi_1 \indicator[t \leq N_1] + \pi_2 \indicator[t > N_1]$ with $N_1 = o(N)$.
For brevity, write $\lambda_i := \lambda_i(\mH(\pi_2; \bm\Theta_\star))$ for $i \in \{ \min, \max \}$.

In the large-$N$ regime, the ratio between the upper and lower bounds (\Cref{thm:estimation-final,thm:lower-bound}) on squared Frobenius error maximally scales with the Hessian's condition number $\frac{\lambda_{\max}}{\lambda_{\min}}$.
Thus, within the single-switch policy class, \texttt{GL-LowPopArt} is nearly instance-wise minimax optimal.
This stands in contrast to guarantees for the nuclear penalized MLE (\Cref{thm:estimation-0}), where the gap between upper and lower bounds is significantly larger, involving unfavorable curvature terms and even dimension factors.

\paragraph{Requirement on $N$.}
A keen reader may observe that our local minimax lower bound holds under the condition $N \gtrsim \frac{R_s^2 r (d_1 \vee d_2)}{\lambda_{\max}(\mH(\pi; \bm\Theta_\star))}$.
We emphasize that this is not restrictive and actually provides an intuitive justification for Stage~I as a warm-up phase; in fact, we believe that some condition of this form on $N$ is necessary---although we do not currently have a formal proof.  
The requirement on $N$ arises when bounding the KL divergence between the true model $\bm\Theta_\star$ and an alternative model from the constructed local packing.  
Intuitively, this stems from the necessity for the two models to be sufficiently close for self-concordance properties to take effect; this was also the case for prior local minimax lower bounds~\citep[Theorem~2]{abeille2021logistic}~\citep[Theorem~1]{simchowitz2020optimal}, where the requirement on horizon length $T$ arises similarly.  
Finally, we remark that in the linear setting ($\mu(z) = z \Rightarrow R_s = 0$), our requirement on $N$ vanishes.

\section{APPLICATIONS OF \texttt{GL-LowPopArt}}

\subsection{Generalized Linear Matrix Completion under Uniform Sampling at Random}
\label{sec:completion}
Here, $\gX = \gM$ (see \Cref{prop:unit-ball}), $\pi^U = \Unif(\gX)$, and $\max_{i,j} |(\bm\Theta_\star)_{i,j}| \leq \gamma$ for a $\gamma > 0$.
For clarity, we focus on the \textit{1-bit matrix completion}~\citep{davenport2014completion} with $\mu(z) = (1 + e^{-z})^{-1}$.
Let $\gE_F := \lVert \widehat{\bm\Theta} - \bm\Theta_\star \rVert_F^2$.

We first compare the error bound of \texttt{GL-LowPopArt} when $\pi_0 = \pi = \pi^U$ with \citet[Theorem 1]{davenport2014completion} and \citet[Corollary 2]{klopp2015multinomial}, denoted as Dav14 and Klo15, respectively:
\begin{align}
    \gE_F &\lesssim \frac{1}{\min_{i,j} \dmu((\bm\Theta_\star)_{ij})} \frac{r d_1 d_2 (d_1 \vee d_2)}{N}, \tag{\textbf{ours}} \\
    \gE_F &\lesssim \frac{1}{\min_{|z| \leq \gamma} \dmu(z)} \sqrt{\frac{r (d_1 d_2)^2 (d_1 \vee d_2)}{N}}, \tag{Dav14} \\
    \gE_F &\lesssim \left( \frac{1}{\min_{|z| \leq \gamma} \dmu(z)} \right)^2 \frac{r d_1 d_2 (d_1 \vee d_2)}{N}. \tag{Klo15}
\end{align}
For our error bound, we utilize \Cref{prop:improve} to express $\GL(\pi^U; \bm\Theta_\star)$ in terms of dimensions for a clear comparison, but we remark that the above bound is looser than \Cref{thm:estimation-final}.
Still, our bound captures the instance-specific difficulty via $\frac{1}{\min_{i,j} \dmu((\bm\Theta_\star)_{ij})}$.
On the other hand, the other bounds depend on the worst-case curvature $\frac{1}{\min_{|z| \leq \gamma} \dmu(z)}$, which may be quite loose in situations where $\min_{i,j} \dmu((\bm\Theta_\star)_{ij}) \gg \min_{|z| \leq \gamma} \dmu(z)$, e.g., due to misspecification of $\gamma$.
We also note that, by \Cref{prop:unit-ball}, when optimal GL-design is employed (i.e., adaptive scenario), our bound improves to
\begin{equation}
    \gE_F \lesssim \Harm(\bm\Theta_\star)^{-1} \frac{r d_1 d_2 (d_1 \vee d_2)}{N},
\end{equation}
which is tighter than all the aforementioned bounds.

Algorithm-wise, \citet{davenport2014completion,klopp2015multinomial}, along with other approaches~\citep{srebro2010tracenorm,cai2013completion,cai2016completion,lafond2015lexponential}, requires the knowledge of $\gamma > 0$, to compute the nuclear-norm regularized estimator \textit{with} the constraint of $\bignorm{\bm\Theta_\star}_{\infty} \leq \gamma$ or $\bignorm{\bm\Theta_\star}_{\max} \leq \gamma$.
Interestingly, \texttt{GL-LowPopArt} does \textit{not} require any knowledge about $\bm\Theta_\star$, yet it fully adapts to the given instance.
In \Cref{app:experiments}, we demonstrate the numerical efficacy of \texttt{GL-LowPopArt} in the 1-bit matrix completion setting compared to nuclear penalized MLE.

We now ask whether \texttt{GL-LowPopArt} is instance-wise optimal in the generalized linear matrix completion setting.
In many literature~\citep{koltchinskii2011nuclear,davenport2014completion,klopp2015multinomial}, the optimality is often discussed in the context of entrywise norm-constrained neighborhoods, e.g., $\max_{i,j} | (\bm\Theta_\star)_{i,j} | \leq \gamma$ for some $\gamma > 0$~\citep[Section 4]{koltchinskii2011nuclear}, which does not match with our assumption on $\bm\Theta_\star$ for our lower bound (\Cref{thm:lower-bound}), i.e., we cannot claim (near) instance-wise optimality here.
We leave this to a future direction.

\begin{remark}[Comparison to \citet{rohde2011estimation}]
    For linear matrix completion with \textbf{fixed design} (i.e., $\pi$ is a mixture of Diracs), \citet[Corollary 1, Theorem 5]{rohde2011estimation} establish a minimax optimal error rate of $\lVert \mathrm{vec}(\widehat{\bm\Theta} - \bm\Theta_\star) \rVert_{\mV(\pi)}^2 \asymp \frac{r (d_1 \vee d_2)}{N}$, with the lower bound holding under restricted isometry.
    Our bound is weaker in this case, as they achieve a tight characterization in the Mahalanobis norm.
    Whether similar rates can be derived for generalized linear matrix completion involving $\bignorm{\cdot}_{\mH(\pi; \bm\Theta_\star)}$ is an interesting future direction.
\end{remark}

\subsection{Bilinear Dueling Bandits}
\label{sec:bilinear}

\subsubsection{Problem Description}
In \textbf{bilinear dueling bandits}, $\gA \subseteq \gB^d(1)$ is considered to be a fixed vector-valued arm-set that is compact, i.e., $\gX$ may be infinite as in continuous dueling bandits~\citep{kumagai2017dueling}.
At each timestep $t$, the learner chooses a pair of arms $(\bm\phi_{w,t}, \bm\phi_{l,t}) \in \gA \times \gA$, and receives a feedback sampled as
\begin{equation}
\label{eqn:bilinear-preference}
    o_t = \indicator[\bm\phi_{w,t} \succ \bm\phi_{l,t}] \sim \Ber(\mu\left( \bm\phi_{w,t}^\top \bm\Theta_\star \bm\phi_{l,t} \right)),
\end{equation}
for an unknown, \emph{skew-symmetric} $\bm\Theta_\star$ of rank $2r$, and a known link function $\mu : \sR \rightarrow [0, 1]$ that additionally satisfies the following~\citep{wu2024dueling}:
\begin{assumption}
    In addition to Assumption~\ref{assumption:mu},
    $\mu : \sR \rightarrow [0, 1]$ satisfies  $\mu(z) + \mu(-z) = 1, \ z \in \sR$.
\end{assumption}
Some examples of $\mu$ that satisfies the above include $\mu(z) = \frac{1 + z}{2}$ and $\mu(z) = (1 + e^{-z})^{-1}$ (Bernoulli).

The learner's goal is to minimize the Borda regret~\citep{saha2021borda,wu2024dueling}:
\begin{equation}
    \Reg^B(T) := \sum_{t=1}^T \left\{ B(\bm\phi_\star) - \frac{B(\bm\phi_{w,t}) + B(\bm\phi_{l,t})}{2} \right\},
\end{equation}
where the \textit{(shifted) Borda score} of $\bm\phi \in \gA$ is defined as
\begin{equation}
    B(\bm\phi) := \E_{\bm\phi' \sim \Unif(\gA)}[\mu(\bm\phi^\top \bm\Theta \bm\phi')],
\end{equation}
and $\bm\phi_\star = \argmax_{\bm\phi \in \gA} B(\bm\phi)$ is the \textit{Borda winner}.
Note that when $\gX$ is finite, it reduces to the usual definition of Borda regret/winner in the finite-armed dueling bandits~\citep{jamieson2015borda,saha2021borda}.
Unlike the Condorcet winner, the Borda winner always exists for any preference model~\citep{bengs2021survey}.

\begin{remark}[Significance of the Setting]
    We emphasize that this is a \textbf{novel} dueling bandits setting, motivated by recent progress in contextualizing general preference learning with item-wise features in RLHF~\citep{zhang2024bilinear}.
    We defer further discussions on the proposed setting to Appendix~\ref{app:bilinear}.
\end{remark}

\begin{algorithm2e}[!t]
\caption{\texttt{BETC-GLM-LR}}
\label{alg:betc-glm-lr}
\For{$t = 1, 2, \cdots, N_1 + N_2$}{
    Run \texttt{GL-LowPopArt}($N_1, N_2$) and obtain $\widehat{\bm\Theta}$\;
}
Obtain the estimated Borda winner:
    $$
    \hat{\bm\phi} \gets \argmax_{\bm\phi \in \gA} \left\{ \widehat{B}(\bm\phi) \triangleq \E_{\bm\phi' \sim \Unif(\gA)}\left[ \mu\left( \bm\phi^\top \widehat{\bm\Theta} \bm\phi' \right) \right] \right\}
    $$

\For{$t = N_1 + N_2 +1, \cdots, T$}{
    Pull $(\widehat{\bm\phi}, \widehat{\bm\phi})$\;
}
\end{algorithm2e}

\subsubsection{\texttt{BETC-GLM-LR} and Regret Upper Bound}
\label{sec:betc-glm-lr}
We consider an explore-then-commit approach, where the exploration is done via our \texttt{GL-LowPopArt} (Algorithm~\ref{alg:betc-glm-lr}).
It attains the following Borda regret bound:
\begin{theorem}[(Informal)]\label{thm:borda-bound}
    With appropriate choices of $N_1$ and $N_2$ in \texttt{GL-LowPopArt} and large enough $T$, \texttt{BETC-GLM-LR} attains the following Borda regret bound with probability at least $1 - \delta$:
    \begin{equation}
        \Reg^B(T) \lesssim \left( \GL_{\min}(\gX) \log\frac{d}{\delta} \right)^{1/3} \left( \kappa_\star^B T \right)^{2/3},
    \end{equation}
    where $\gX := \{ \bm\phi_i \bm\phi_j^\top : 1 \leq i < j \leq K \} \subseteq \gB^{d \times d}_{\mathrm{op}}(1)$ and $\kappa_\star^B := \E_{\bm\phi' \sim \mathrm{Unif}(\gA)}[\dmu(\bm\phi_\star^\top \bm\Theta \bm\phi')]$.
\end{theorem}
\begin{proof}[Proof Sketch]
    We deviate significantly from \citet{wu2024dueling} by using the self-concordance of $\mu$ as in \citet[Theorem 1]{abeille2021logistic} for the regret bound to scale \textit{with} $\kappa_\star^B$.
    Refer to \cref{app:bilinear-alg} for the full proof.
\end{proof}
We believe $T^{2/3}$ dependency of the Borda regret is unavoidable, as there are already $\Omega(T^{2/3})$ Borda regret lower bounds (omitting other dependencies) for entrywise~\citep[Theorem 16]{saha2021borda} and generalized linear dueling bandits~\citep[Theorem 4.1]{wu2024dueling}.

Two quantities make our regret bound truly instance-specific.
One is $\GL_{\min}(\gX)$, which, as discussed previously, captures the geometry of $\gX$ and the associated nonlinearity via the Hessian.
Another is $\kappa_\star^B$, which captures the average curvature of the preference model when compared against the Borda winner.
Analogous to recent findings in generalized linear bandits \citep{abeille2021logistic,lee2024glm}, our Borda regret can \textit{benefit} from the ``flatness'' of the link function.
Intuitively, we enjoy lower regret in highly flat landscapes, provided that the worst-case flatness across all item pairs is comparable to the flatness around the Borda winner. However, if there exists a hard pair elsewhere in the arm set, the estimation difficulty dominates and we lose this variance-reduction benefit.
We defer a more in-depth discussions, including comparisons with \citet{wu2024dueling}, to Appendix~\ref{app:bilinear-regret}.

\section{CONCLUSION}
\label{sec:conclusion}
This work addresses the critical gap in prior work by explicitly considering instance-specific curvature in generalized low-rank trace regression.
We introduce \texttt{GL-LowPopArt}, a novel estimator that achieves state-of-the-art performance, adapting to both the nonlinearity of the model and the underlying arm-set geometry.
We establish the first instance-wise minimax lower bound, demonstrating the near-optimality of \texttt{GL-LowPopArt}.
We showcase its benefits through applications to generalized linear matrix completion and bilinear dueling bandits, a novel setting of independent interest for general preference learning~\citep{zhang2024bilinear}.
In Appendix~\ref{app:future}, we collect further future directions.

\newpage

\subsubsection*{Acknowledgements}
We thank all the anonymous reviewers for their helpful comments that helped us improve the manuscript.
K. Jang was supported by the Institute of Information \& Communications Technology Planning \& Evaluation (IITP) grant funded by the Korea government (MSIT) [RS-2021-II211341, Artificial Intelligence Graduate School Program (Chung-Ang University)].
K.-S. Jun was supported in part by the National Science Foundation under grant CCF-2327013 and Meta Platforms, Inc.

\bibliographystyle{plainnat}
\bibliography{references}

\clearpage
\section*{Checklist}

\begin{enumerate}

  \item For all models and algorithms presented, check if you include:
  \begin{enumerate}
    \item A clear description of the mathematical setting, assumptions, algorithm, and/or model. [Yes] Section 2,3,5
    \item An analysis of the properties and complexity (time, space, sample size) of any algorithm. [Yes] Section 3,5
    \item (Optional) Anonymized source code, with specification of all dependencies, including external libraries. [Yes] Appendix M.
  \end{enumerate}

  \item For any theoretical claim, check if you include:
  \begin{enumerate}
    \item Statements of the full set of assumptions of all theoretical results. [Yes] Section 2,3,4,5
    \item Complete proofs of all theoretical results. [Yes] The whole Appendix
    \item Clear explanations of any assumptions. [Yes] Section 2
  \end{enumerate}

  \item For all figures and tables that present empirical results, check if you include:
  \begin{enumerate}
    \item The code, data, and instructions needed to reproduce the main experimental results (either in the supplemental material or as a URL). [Yes] Appendix M.
    \item All the training details (e.g., data splits, hyperparameters, how they were chosen). [Yes] Appendix M.
    \item A clear definition of the specific measure or statistics and error bars (e.g., with respect to the random seed after running experiments multiple times). [Yes] Appendix M.
    \item A description of the computing infrastructure used. (e.g., type of GPUs, internal cluster, or cloud provider). [No] Simple CPU experiments. 
  \end{enumerate}

  \item If you are using existing assets (e.g., code, data, models) or curating/releasing new assets, check if you include:
  \begin{enumerate}
    \item Citations of the creator If your work uses existing assets. [Yes] Appendix M.
    \item The license information of the assets, if applicable. [Not Applicable]
    \item New assets either in the supplemental material or as a URL, if applicable. [Yes] Appendix J
    \item Information about consent from data providers/curators. [Not Applicable]
    \item Discussion of sensible content if applicable, e.g., personally identifiable information or offensive content. [Not Applicable]
  \end{enumerate}

  \item If you used crowdsourcing or conducted research with human subjects, check if you include:
  \begin{enumerate}
    \item The full text of instructions given to participants and screenshots. [Not Applicable]
    \item Descriptions of potential participant risks, with links to Institutional Review Board (IRB) approvals if applicable. [Not Applicable]
    \item The estimated hourly wage paid to participants and the total amount spent on participant compensation. [Not Applicable]
  \end{enumerate}

\end{enumerate}

\clearpage
\appendix
\thispagestyle{empty}

\crefname{appendix}{Appendix}{Appendices}
\Crefname{appendix}{Appendix}{Appendices}
\crefalias{section}{appendix}
\crefalias{subsection}{appendix}
\crefalias{subsubsection}{appendix}

\onecolumn

\tableofcontents
\newpage

\section{RELATED WORK}
\label{app:related-work}

\paragraph{Generalized Linear Matrix Completion.}

This has been extensively studied in the early 2010s under various noise assumptions: Gaussian~\citep{rohde2011estimation,koltchinskii2011nuclear}, Bernoulli~\citep{alquier2019oracle}, multinomial~\citep{lafond2014finite,klopp2015multinomial}, general exponential family~\citep{lafond2015lexponential}, and even with the only assumption of bounded variance~\citep{klopp2014general}.
We refer interested readers to \citet{davenport2016overview} for an overview of works on matrix completion.
Note that our model implicitly implies that for each $(i, j) \in [d_1] \times [d_2]$ may be observed multiple times, which is often the case in recommender systems and bandits where the same item can be recommended multiple times for exploration, or it may be that ``users are more active than others
and popular items are rated more frequently.''~\citep{klopp2015multinomial}.
On a slightly different note, many works have explored the same setting under the assumption that each entry of $\bm\Theta_\star$ can be sampled at most once~\citep{candes2010noise,cai2013completion,davenport2014completion,gunasekar2014exponential,cao2016poisson,mokhtar2019collective,mcrae2020poisson}.
When $\bm\Theta_\star$ is additionally is skew-symmetric ($\bm\Theta_\star^\top = -\bm\Theta_\star$), this is also related to learning the low-rank preference model~\citep{gleich2011nuclear,lu2015individualized,rajkumar2016comparison,wu2024dueling,zhang2024bilinear}.

\paragraph{Low-Rank Matrix Bandits.} 

Researchers in low-rank bandits have long focused on fundamental and specific models. For example, \citet{katariya2017rank1, katariya2017rank2, trinh2020bernoulli, jedra2024lowrank, sentenac2021matching} studied a bilinear bandit setting (which means $\gX = \{ \vx \vx^\top: \vx \in \gA \subset \sR^{d_1} \}$) with canonical basis (\(\gA=\{ \ve_i: i \in [d_1] \} \) and \(\gZ=\{ \ve_j: j \in [d_2] \}\)). \citet{katariya2017rank1, katariya2017rank2, trinh2020bernoulli, sentenac2021matching} added an assumption that $\text{rank}(\bm\Theta_\star)=1$ over a bilinear bandit setting.
\citet{stojanovic2023lowrank} presents an entry-wise matrix estimation for low-rank reinforcement learning, including low-rank bandits. Another popular assumption on arm sets in low-rank bandits is a unit ball (or a unit sphere) assumption~\cite{kotlowski2019banditpca,lattimore2021phase,huang2021optimal}.
For bilinear bandits, \citet{kotlowski2019banditpca} assumed that $\gX = \{ \vx \vx^\top : \vx \in \gS^{d-1}(1)\}$ and $\bm\Theta_\star$ should be also symmetric. \cite{lattimore2021phase} even added an assumption that $\bm\Theta_\star$ is a symmetric rank-1 matrix. For low-rank bandits, \citet{huang2021optimal} assumed $\gX=\gB_F^{d \times d}(1)$. These tailored algorithms often outperform general approaches significantly, yet extending these algorithms to other settings has generally proven challenging due to the highly specialized nature of their settings.  

The first study on low-rank bandits with general arm sets is \citet{jun2019bilinear}. This work introduced the first general bilinear low-rank linear bandit algorithm that could be applied flexibly to any \(d\)-dimensional arm set \(\gX\) and \(\gZ\). Subsequently, \citet{lu2021generalized} extended this approach beyond bilinear settings, proposing a generalized low-rank linear bandit algorithm applicable to all matrix arm sets. Later, \citet{kang2022generalized} introduced a novel method leveraging Stein’s method, and \citet{li2022unified} developed a general framework for high-dimensional linear bandits, including low-rank bandits. However, none of these studies explicitly addressed experimental design; rather, they handled the issue of experimental designs by assuming that their arm sets are sufficiently well-distributed in all directions. As a result, they failed to fully capture how the regret bound varies with the geometry of the arm set. For example, \cite{jun2019bilinear} and \cite{lu2021generalized} conjectured that the lower bound for the bilinear low-rank bandit problem should be \(\Omega(\sqrt{rd^3 T})\), based on results from trace regression. However, \citet{jang2021bilinear} later demonstrated that by considering the structure of the arm set in the bilinear setting, this bound could be further improved, highlighting the importance of optimal design tailored to the arm set.

Recent work by \citet{jang2024lowrank} systematically addresses arm set geometry and experimental design in the low-rank linear bandits. This work applied thresholding at the subspace level called \texttt{LowPopArt} and proposed a novel experimental design for this new regression method. 
They then analyzed the experimental design assumptions underlying previous studies and successfully proved that their \texttt{LowPopArt} with their experimental design outperforms the previous works, even order-wise improvements in some cases.
Our paper further extends the \texttt{LowPopArt} to the generalized linear scenario and provides performance guarantees in both upper and lower bounds that are nearly optimal even in terms of instance-specific, curvature-dependent quantities.

\paragraph{Generalized Linear Bandits (GLBs).}
\textbf{GLB} is a natural nonlinear extension of linear bandits, first proposed by \citet{filippi2010glm}, and later studied by much works~\citep{lee2024glm,sawarni2024glm,jun2017conversion,li2017glm}.
\textbf{GLB}s encompass a wide range of bandits, including linear, logistic, Poisson, logit, and more.
Out of these, especially \textbf{logistic bandits (LogB)}~\citep{faury2020logistic,faury2022logistic,mason2022logistic,abeille2021logistic,lee2024logistic} has garnered much attention, as it can naturally model binary feedback (`click' or `no click'; \citet{li2012glm}).
Also, owing to its similarity to the Bradley-Terry model-based RLHF, the confidence sets of logistic bandits have been used for quantifying the uncertainty of the linear reward model~\citep{das2024rlhf,xiong2024iterative,zhong2024rlhf}.
In \textbf{GLB}s, the key quantity describing the problem difficulty is\footnote{In the mentioned literature, the quantity is denoted as $\kappa_\star$. To keep our notation consistent with the dueling bandits' literature, we chose to denote this as $\kappa_\star^{-1}$.} $\kappa_\star := \dmu(\langle \vx_\star, \bm\theta_\star \rangle)$, where $\bm\theta_\star$ is the unknown vector and $\vx_\star$ is the optimal arm vector.
\cite{abeille2021logistic} showed a regret lower bound of $\Omega(d \sqrt{T \kappa_\star})$ for \textbf{LogB}s, which was matched by various UCB-type algorithms~\citep{abeille2021logistic,faury2022logistic,lee2024logistic}.
Despite the lack of a generic lower bound for general \textbf{GLB}s, recent breakthroughs~\citep{sawarni2024glm,lee2024glm,liu2024free} showed that for self-concordant \textbf{GLB}s, regret upper bound of $\widetilde{\gO}(d \sqrt{T \kappa_\star})$ can be attained.

\begin{remark}[Different Definitions of Self-Concordance]
{In the optimization literature, the original definition of the self-concordance takes the form of $|\dddot{\mu}(z)|\leq 2\ddot{\mu}(z)^{3/2}$ $\forall z\in\mathbb{R}$, originally motivated for convergence analysis of Newton's method by \citet{nesterov1988polynomial}. \citet{bach2010self} was the first to adapt the concept to extend the M-estimator results of squared loss to logistic loss, later further extended by \citet{ostrovskii2021self}. Later, the bandit community further adapted it for logistic and generalized linear bandits~\citep{faury2020logistic,abeille2021logistic,russac2021glm}, which is the form we consider here (Assumption~\ref{assumption:mu}(a))}
\end{remark}

\paragraph{Burer-Monteiro Factorization}
The Burer–Monteiro factorization (BMF, \citet{burer-monteiro1,buere-monteiro2}) approach has been extensively studied for noiseless low-rank matrix recovery from deterministic linear measurements~\citep{candes2009completion,candes2011oracle}, primarily from an optimization perspective~\citep{bi2022bmf,ge2017bmf,park2017bmf,zhang2024bmf,boumal2016bmf,yalcin2022bmf,bhojanapalli2016bmf,stoger2021lowrank,kim2023lowrank}.
In contrast, our work focuses on noisy matrix completion under a generalized linear model (GLM) framework, aiming to achieve accurate estimation with high probability as the sample size increases.
This fundamental difference in problem settings implies that the optimization complexity measures used to analyze BMF methods, such as the optimization complexity metric (OCM) introduced by \citet{yalcin2022bmf} and \citet{zhang2024bmf}, are not directly comparable to our statistical analysis.
Specifically, their OCM quantifies the non-convexity of the BMF landscape, which is related to the success of local search methods (e.g., gradient descent), while our ``statistical complexity metric'', arguably $\lambda_{\max}(\mH(\pi; \bm\Theta_\star))$ that pops up in our lower bound (\cref{thm:lower-bound}), is information-theoretic and dictates the minimum sample size required for any estimator to obtain a desired accuracy with high probability.

While BMF methods offer computational efficiency and have been shown to perform well empirically, especially in large-scale problems, they all rely on some non-convex optimization, whose landscape is not always guaranteed to be benign, especially in the presence of noise~\citep{ma2023bmf}.
Our \texttt{GL-LowPopArt} only involves convex optimization subroutines and thus is computationally tractable, but inefficient: for instance, \texttt{GL-LowPopArt} requires computing the SVD and inverting $d^2 \times d^2$ matrices.
Therefore, while BMF and our work both address low-rank matrix recovery, their respective advantages depend on the specific problem context.

\newpage
\section{TABLE OF NOTATION}

\begin{table}[h]
    \centering
    \caption{Summary of notation used in this paper.}
    \label{tabs:notation}
    \begin{tabular}{|c|p{12cm}|}
        \hline
        \textbf{Notation} & \textbf{Description} \\
        \hline
        $\bignorm{\cdot}_\nuc$ & Nuclear norm\\
        $\bignormop{\cdot}$ & Operator (spectral) norm\\
        $\langle \mA, \mB \rangle$ for $\mA, \mB \in \sR^{m \times n}$ & $\tr(\mA^\top \mB)$\\
        $\lambda_i(\mA)$ & The $i$-th largest eigenvalue of a symmetric matrix $\mA$\\
        $\lambda_{\max} $& The largest eigenvalue, same as $ \lambda_1$\\
        $\lambda_{\min}$ & The smallest eigenvalue, same as $\lambda_m$\\
        $\gB^{d_1 \times d_2}_i(S)$ for $i \in \{ {\mathrm{op}}, \nuc, F \}$ & $ \{ \mX \in \sR^{d_1 \times d_2} : \bignorm{\mX}_i \leq S \}$\\
        $\vec : \sR^{d_1 \times d_2} \rightarrow \sR^{d_1 d_2}$ & Column-wise stacking operation of a matrix into a vector\\
        $\vec^{-1}: \sR^{d_1 d_2} \rightarrow \sR^{d_1 \times d_2}$ & Reshape operation of a vector to a matrix\\ 
        $[n]$ for $n \in \mathbb{N}$ & $\{1, 2, \dots, n\}$\\ 
\( \gP(X) \) & The set of all probability distributions on \( X \)\\
$\Omega$ & Parameter space \\
    $\bm\Theta_\star \in \sR^{d_1 \times d_2}$& An unknown reward matrix of rank at most $r \ll d_1 \wedge d_2$\\
    $\gX \subseteq \sR^{d_1 \times d_2}$& Arm-set (e.g., sensing matrices).\\
     $p(y | \mX; \bm\Theta_\star)$ & Probability density function of the generalized linear model of the reward $y$ when $X$ is chosen by the learner, $\propto \exp\left( \frac{y \langle \mX, \bm\Theta_\star \rangle - m(\langle \mX, \bm\Theta_\star \rangle)}{g(\tau)} \right)$\\
     $m: \mathbb{R}\to\mathbb{R}$ & log-partition function of GLM\\
     $\tau$ & Dispersion parameter\\
     $\mu$ & $\dot{m}$, {Inverse link function}.     \\
     $\pi \in \mathcal{P}(\mathcal{A})$ & Sampling policy (design) \\
    $\mV(\pi)$ & Design matrix, $\E_{\mX \sim \pi}[\vec(\mX) \vec(\mX)^\top]$\\
         $\mH(\pi; \bm\Theta)$ & Hessian matrix $\E_{\mX \sim \pi} \left[ \dmu(\langle \mX, \bm\Theta \rangle) \vec(\mX) \vec(\mX)^\top \right]$\\
$L_\mu, R_s, \kappa_\star$ & Parameters on $\mu$, check Assumption \ref{assumption:mu}\\
    $\mathcal{H}$ & Hermitian Dilation (Check Eq. \eqref{eqn: hermitiaon dilation})\\
    $\psi$ & Influence function (Check Eq. \eqref{eqn: influence function}\\
    $\tilde{\psi}_\nu(\mA)$& $\frac{1}{\nu} \psi(\nu \gH(\mA))_{\mathrm{ht}}$, where for $\mM \in \sR^{(d_1 + d_2) \times (d_1 + d_2)}$, $\mM_{\mathrm{ht}} := \mM_{1:d_1,d_1+1:d_1+d_2}$\\
     $\GL(\pi; \bm\Theta_\star)$ & Our new experimental design objective (See \Cref{def:GL-Design})\\
        \hline
    \end{tabular}
\end{table}

\newpage
\section{\texorpdfstring{PROOF OF THEOREM~\ref{thm:estimation-final} -- ERROR BOUND OF STAGE~II}{PROOF OF THEOREM 3.1 -- ERROR BOUND OF STAGE~II}}
\label{app:estimation-final}
Recall the Hessian at the initial estimator $\bm\Theta_0$:
\begin{equation}
    \mH(\pi; \bm\Theta_0) := \E_{\mX \sim \pi} \left[ \dmu(\langle \mX, \bm\Theta_0 \rangle) \vec(\mX) \vec(\mX)^\top \right],
\end{equation}
and the one-sample estimators (line 9 of Algorithm~\ref{alg:gl-lowpopart}): for each $t \in [N_1]$, 
\begin{equation}
    \widetilde{\bm\Theta}_t = \vec_{d \times d}^{-1}\left( \widetilde{\bm\theta}_t \right), \quad
    \widetilde{\bm\theta}_t := \mH(\pi; \bm\Theta_0)^{-1} \left( y_t - \mu(\langle \mX_t, \bm\Theta_0\rangle) \right) \vec(\mX_t),
\end{equation}

We will also denote ${\color{violet}E} := \bignormnuc{\bm\Theta_\star - \bm\Theta_0}^2$, and will track any requirements on ${\color{violet}E}$ in {\color{violet}violet} throughout the proof.

First observe that
\begin{align}
    \E[\widetilde{\bm\theta}_t | \mX_t = \mX] &= \mH(\pi; \bm\Theta_0)^{-1} \left[ \mu(\langle \mX, \bm\Theta_\star \rangle) - \mu(\langle \mX, \bm\Theta_0 \rangle) \right] \vec(\mX) \\
    &\overset{(*)}{=}  \mH(\pi; \bm\Theta_0)^{-1} \left[ \dmu(\langle \mX, \bm\Theta_0 \rangle) \langle \bm\Theta_\star - \bm\Theta_0, \mX \rangle + \langle \bm\Theta_\star - \bm\Theta_0, \mX \rangle^2 G(\bm\Theta_0, \bm\Theta_\star; \mX) \right] \vec(\mX) \tag{first-order Taylor expansion with integral remainder} \\
    &= \mH(\pi; \bm\Theta_0)^{-1} \left[ \dmu(\langle \mX, \bm\Theta_0 \rangle)  \vec(\mX) \vec(\mX)^\top \vec(\bm\Theta_\star - \bm\Theta_0) + \langle \bm\Theta_\star - \bm\Theta_0, \mX \rangle^2 G(\bm\Theta_0, \bm\Theta_\star; \mX)  \vec(\mX) \right],
\end{align}
where at $(*)$, we define
\begin{equation}
    G(\bm\Theta_0, \bm\Theta_\star; \mX) := \int_0^1 (1 - z) \ddot{\mu}(\langle z \bm\Theta_\star + (1 - z) \bm\Theta_0, \mX \rangle) dz.
\end{equation}

By taking the expectation over $\mX \sim \pi$, we have that
\begin{equation}
    \E\left[ \widetilde{\bm\theta}_t \right] = \vec(\bm\Theta_\star - \bm\Theta_0) + \E_{\mX \sim \pi}\left[ \langle \bm\Theta_\star - \bm\Theta_0, \mX \rangle^2 G(\bm\Theta_0, \bm\Theta_\star; \mX)  \mH(\pi; \bm\Theta_0)^{-1}  \vec(\mX) \right].
\end{equation}

Note that $\widetilde{\bm\theta}_t$'s are \textit{biased} estimators of $\vec(\bm\Theta_\star - \bm\Theta_0)$.
This is the key difference from \citet{jang2024lowrank}.

With $\bm\Theta_1$ as the matrix Catoni estimator for $\bm\Theta_\star - \bm\Theta_0$ (line 14 of Algorithm~\ref{alg:gl-lowpopart}), we utilize the following error decomposition:
\begin{align}
    \bignormop{\bm\Theta_1 - \bm\Theta_\star} &\leq \underbrace{\bignormop{(\bm\Theta_1 - \bm\Theta_0) - \E[\widetilde{\bm\Theta}_t]}}_{(a)} + \bignormop{\E[\widetilde{\bm\Theta}_t] - (\bm\Theta_\star - \bm\Theta_0)} \\
    &\leq (a) + \bignorm{\E[\widetilde{\bm\Theta}_t] - (\bm\Theta_\star - \bm\Theta_0)}_F \\
    &= (a) + \underbrace{\bignorm{\E_{\mX \sim \pi}\left[ \langle \bm\Theta_\star - \bm\Theta_0, \mX \rangle^2 G(\bm\Theta_0, \bm\Theta_\star; \mX) \mH(\pi; \bm\Theta_0)^{-1}  \vec(\mX) \right]}_2}_{(b)}.
\end{align}

We first control the unbiased error term $(a)$ using the following error rate for the matrix Catoni estimator due to \citet{minsker2018matrix}:
\begin{lemma}[Corollary 3.1 of \citet{minsker2018matrix}]
\label{lem:minsker}
    Let $\{ \mA_i \}_{i=1}^n \subset \sR^{d_1 \times d_2}$ be independent with $\E[\mA_i] = \mA$, and define their matrix variance statistics as
    \begin{equation}
        \sigma_n^2 := \max\left\{ \bignormop{\sum_{i=1}^n \E[\mA_i \mA_i^\top]}, \bignormop{\sum_{i=1}^n \E[\mA_i^\top \mA_i]} \right\}.
    \end{equation}
    Then we have that for any $\delta \in (0, 1)$,
    \begin{equation}
        \sP\left( \bignormop{\widehat{\mT} - \mA} \leq \sqrt{\frac{2\sigma_n^2}{n^2} \log\frac{2(d_1 + d_2)}{\delta}} \right) \geq 1 - \delta,
    \end{equation}
    where
    \begin{equation}
        \widehat{\mT} := \frac{1}{n} \left( \sum_{i=1}^n \tilde{\psi}_\nu(\mA_i) \right)_{\mathrm{ht}}, \ \nu := \sqrt{\frac{2}{\sigma_n^2} \log\frac{2(d_1+d_2)}{\delta}}.
    \end{equation}
\end{lemma}
\begin{remark}
    The significance of the Catoni-type robust estimator~\citep{catoni2012,minsker2018matrix} is that the guarantee does not assume the boundedness of the matrices, yet it still gives a Bernstein-type concentration.
    This has been successfully utilized in obtaining tight, instance-specific guarantees for various reinforcement learning problems, such as sparse linear bandits~\citep{jang2022popart}, low-rank bandits~\citep{jang2024lowrank}, linear MDP~\citep{wagenmaker2022linearrl}, and more.
\end{remark}

To apply the above lemma, we bound the matrix variance statistics of the one-sample estimators $\widetilde{\bm\Theta}_t$'s, whose proof is deferred to \Cref{app:minsker-variance}: recalling that $\kappa_\star := \min_{\mX \in \gX} \dmu(\langle \mX, \bm\Theta_\star \rangle) > 0$,\footnote{Note that this holds because $\gX$ is bounded.}
\begin{lemma}
\label{lem:minsker-variance}
Suppose that {\color{violet}$E^2 \leq \frac{\kappa_\star g(\tau)}{4 L_\mu^2}$}. Then,
    \begin{equation}
        \max\left\{ \bignormop{\sum_{t=1}^{N_2} \E[\widetilde{\bm\Theta}_t \widetilde{\bm\Theta}_t^\top]}, \bignormop{\sum_{t=1}^{N_2} \E[\widetilde{\bm\Theta}_t^\top \widetilde{\bm\Theta}_t]} \right\}
        \leq 5 g(\tau) \GL(\pi; \bm\Theta_0) N_2,
    \end{equation}
    where $\GL(\pi; \bm\Theta_0) := \max\{ H^{(\row)}(\pi; \bm\Theta_0), H^{(\col)}(\pi; \bm\Theta_0) \}$ with
    \begin{equation}
        H^{(\row)}(\pi; \bm\Theta_0) := \lambda_{\max}\left( \sum_{{\color{blue}m}=1}^{d_1} \mD_{\color{blue}m}^{(\row)}(\pi; \bm\Theta_0) \right), \quad
        \mD_{\color{blue}m}^{(\row)}(\pi; \bm\Theta_0) :=
         [(\mH(\pi; \bm\Theta_0)^{-1})_{jk}]_{j,k \in {\color{purple}\gI_{\color{blue}m}^{(\row)}}},
    \end{equation}
    and
    \begin{equation}
        H^{(\col)}(\pi; \bm\Theta_0) := \lambda_{\max}\left( \sum_{{\color{blue}m}=1}^{d_2} \mD_{\color{blue}m}^{(\col)}(\pi; \bm\Theta_0) \right), \quad
        \mD_{\color{blue}m}^{(\col)}(\pi; \bm\Theta_0) := (\mH(\pi; \bm\Theta_0)^{-1})_{{\color{teal}\gI_{\color{blue}m}^{(\col)}}, {\color{teal}\gI_{\color{blue}m}^{(\col)}}},
    \end{equation}
    where the index sets defined as ${\color{purple}\gI_{\color{blue}m}^{(\row)}} := \{d_1 (\ell-1) + {\color{blue}m}: \ell \in [d_2]\}$ and ${\color{teal}\gI_{{\color{blue}m}}^{(\col)}} := [d_1 ({\color{blue}m}-1)+1:d_1 {\color{blue}m}]$.
\end{lemma}

Then, we have that by \Cref{lem:minsker,lem:minsker-variance},
\begin{equation}
    \sP\left( (a) = \bignormop{(\bm\Theta_1 - \bm\Theta_0) - \E[\widetilde{\bm\Theta}_t]} \leq \sqrt{\frac{10 g(\tau) \GL(\pi)}{N_2} \log\frac{4(d_1 + d_2)}{\delta}} \right) \geq 1 - \frac{\delta}{2}.
\end{equation}

We now control the bias term $(b)$.
For notational simplicity, we denote $\vx \triangleq \mathrm{\vec}(\mX)$ and $\bm\delta \triangleq \mathrm{vec}(\bm\Theta_\star - \bm\Theta_0)$.
\begin{align}
    \bignorm{\E[\widetilde{\bm\Theta}_t] - (\bm\Theta_\star - \bm\Theta_0)}_F &= \bignorm{\E_\vx\left[ \langle \bm\delta, \vx \rangle^2 G(\bm\Theta_0, \bm\Theta_\star; \mX) \mH(\pi; \bm\Theta_0)^{-1} \vx \right]}_2 \\
    &= \bignorm{\mH(\pi; \bm\Theta_0)^{-1} \E_\vx\left[ G(\bm\Theta_0, \bm\Theta_\star; \mX) \langle \bm\delta, \vx \rangle^2 \vx \right]}_2 \\
    &\leq \bignormop{\mH(\pi; \bm\Theta_0)^{-1}} \bignorm{ \E_\vx\left[ G(\bm\Theta_0, \bm\Theta_\star; \mX) \langle \bm\delta, \vx \rangle^2 \vx \right]}_2 \\
    &\leq \lambda_{\min}(\mH(\pi; \bm\Theta_0))^{-1} \E_\vx\left[ |G(\bm\Theta_0, \bm\Theta_\star; \mX)| \langle \bm\delta, \vx \rangle^2 \bignorm{\vx}_2 \right] \tag{Jensen's inequality} \\
    &\leq \frac{R_s L_\mu}{2 \lambda_{\min}(\mH(\pi; \bm\Theta_0))} \E_\vx\left[ \langle \bm\delta, \vx \rangle^2 \right] \tag{$|G(\bm\Theta_0, \bm\Theta_\star; \mX)| \leq \frac{1}{2} R_s L_\mu$} \\
    &\leq \frac{R_s L_\mu}{2 \lambda_{\min}(\mH(\pi; \bm\Theta_0))} {\color{violet}E}^2, \tag{C.1}
\end{align}
where the last inequality follows from H\"{o}lder's inequality:
\begin{equation}
    \E_\vx\left[ \langle \bm\delta, \vx \rangle^2 \right] = \E_{\mX}\left[ \langle \bm\Theta_\star - \bm\Theta_0, \mX \rangle^2 \right]
    \leq \E_{\mX}\left[ \bignormnuc{\bm\Theta_\star - \bm\Theta_0}^2 \bignormop{\mX}^2 \right]
    \leq {\color{violet}E}^2.
\end{equation}

Thus, with {\color{violet}$R_s E^2 \lesssim \lambda_{\min}(\mH(\pi; \bm\Theta_0)) \sqrt{\frac{g(\tau) \GL(\pi; \bm\Theta_0)}{N_2} \log\frac{d_1 + d_2}{\delta}}$}, the following holds with probability at least $1 - \delta$:
\begin{equation}
    \bignormop{\bm\Theta_1 - \bm\Theta_\star} \leq \sigma_{\mathrm{thres}} \triangleq 4 \sqrt{\frac{g(\tau) \GL(\pi; \bm\Theta_0)}{N_2} \log\frac{4(d_1 + d_2)}{\delta}}.
\end{equation}

As the last step of the proof, we recall Weyl's inequality for singular values:
\begin{lemma}[Problem 7.3.P16 of \citet{hornjohnson}]
    For any $\mA, \Delta \in \sR^{d_1 \times d_2}$, we have
    \begin{equation}
        |\sigma_k(\mA + \Delta) - \sigma_k(\mA)| \leq \sigma_1(\Delta), \quad \forall k \in [d_1 \wedge d_2].
    \end{equation}
\end{lemma}

As $\sigma_k(\bm\Theta_\star) = 0$ for $k \geq r + 1$, we have that $\sigma_k(\bm\Theta_1) \leq \sigma_{\mathrm{thres}}$ for the same $k$'s as well.
This proves that the thresholding part of Algorithm~\ref{alg:gl-lowpopart} (line 10) indeed yields $\rank(\widehat{\bm\Theta}) \leq r$.

The final error bound follows from the triangle inequality and \Cref{lem:jun5}, which implies that $\GL(\pi; \bm\Theta_0) \leq 2 \GL(\pi; \bm\Theta_\star)$.
\qed

\subsection{Proof of Lemma~\ref{lem:minsker-variance} -- Controlling the Matrix Variance Statistics}
\label{app:minsker-variance}

Throughout, we utilize the following self-concordance control formalizing the intuition that $\mH(\pi; \bm\Theta_0) \approx \mH(\pi; \bm\Theta_\star)$:
\begin{lemma}[Lemma 5 of \citet{jun2021confidence}, adapted to our notations]
\label{lem:jun5}
    Suppose {\color{violet}$R_s E \leq \frac{1}{2}$}.
    Then, we have that
    \begin{equation}
        \frac{1}{2} \mH(\pi; \bm\Theta_\star) \preceq \mH(\pi; \bm\Theta_0)
        \preceq 2 \mH(\pi; \bm\Theta_\star).
    \end{equation}
\end{lemma}

We will bound $\bignormop{\E[\widetilde{\bm\Theta}_t \widetilde{\bm\Theta}_t^\top]}$ only, as the other one follows analogously.

For $y \sim p(\cdot | \mX; \bm\Theta_\star)$,
\begin{align}
    \E[(y - \mu(\langle \mX, \bm\Theta \rangle))^2 | \mX] &\leq 2\E[(y - \mu(\langle \mX, \bm\Theta_\star))^2] + 2(\mu(\langle \mX, \bm\Theta_\star \rangle) - \mu(\langle \mX, \bm\Theta_0 \rangle))^2 \\
    &\leq 2 g(\tau) \dmu(\langle \mX, \bm\Theta_\star \rangle) + 2 L_\mu^2 \langle \mX, \bm\Theta_\star - \bm\Theta_0 \rangle^2 \\
    &\leq 2 \left( g(\tau) + \frac{L_\mu^2 {\color{violet}E}^2}{\kappa_\star} \right) \dmu(\langle \mX, \bm\Theta_\star \rangle). \tag{C.2}
\end{align}

For notational simplicity, we further introduce $\mA_\mX := \vec^{-1}\left(  \mH(\pi; \bm\Theta_0)^{-1}  \vec(\mX) \right) \in \sR^{d_1 \times d_2}$.
Then, we have
\begin{align}
    \bignormop{\E[\widetilde{\bm\Theta}_t \widetilde{\bm\Theta}_t^\top]} &= \bignormop{\E\left[ \left(y_t - \mu(\langle \mX_t, \bm\Theta_0\rangle) \right)^2 \mA_{\mX_t} \mA_{\mX_t}^\top \right]} \\
    &= \bignormop{\E_{\mX \sim \pi}\left[ \E_{y \sim p(\cdot | \mX; \bm\Theta_\star)}[(y - \mu(\langle \mX, \bm\Theta_0 \rangle))^2 | \mX] \mA_\mX \mA_\mX^\top \right]} \tag{Tower property} \\
    &\leq 2 \left( g(\tau) + \frac{L_\mu^2 {\color{violet}E}^2}{\kappa_\star} \right) \bignormop{\E_{\mX \sim \pi}\left[ \dmu(\langle \mX, \bm\Theta_\star\rangle) \mA_\mX \mA_\mX^\top \right]}.
\end{align}

We largely follow the proof of \citet[Lemma B.2]{jang2024lowrank} as follows: denoting $\gS^{d_1 - 1}$ as the unit sphere in $\sR^{d_1}$,
\begin{align}
    &\bignormop{\E_{\mX \sim \pi}\left[ \dmu(\langle \mX, \bm\Theta_\star\rangle) \mA_\mX \mA_\mX^\top \right]} \\
    &= \max_{\vu \in \gS^{d_1 - 1}} \vu^\top \E_{\mX \sim \pi}\left[ \dmu(\langle \mX, \bm\Theta_\star\rangle) \mA_\mX \mA_\mX^\top \right] \vu \\
    &= \max_{\vu \in \gS^{d_1 - 1}} \vu^\top \E_{\mX \sim \pi}\left[ \dmu(\langle \mX, \bm\Theta_\star\rangle) \mA_\mX \left( \sum_{{\color{blue}m}=1}^{d_2} \ve_{\color{blue}m} \ve_{\color{blue}m}^\top \right) \mA_\mX^\top \right] \vu \tag{let $\{\ve_{\color{blue}m}\}_{{\color{blue}m} \in [d_2]}$ be the standard basis vectors of $\sR^{d_2}$} \\
    &= \max_{\vu \in \gS^{d_1 - 1}} \E_{\mX \sim \pi}\left[ \dmu(\langle \mX, \bm\Theta_\star\rangle) \sum_{{\color{blue}m}=1}^{d_2} \left( \vu^\top \mA_\mX \ve_{\color{blue}m} \right)^2 \right] \\
    &= \max_{\vu \in \gS^{d_1 - 1}} \E_{\mX \sim \pi}\left[ \dmu(\langle \mX, \bm\Theta_\star\rangle) \sum_{{\color{blue}m}=1}^{d_2} \langle \ve_{\color{blue}m} \otimes \vu, \vec(\mA_\mX) \rangle^2 \right] \tag{$\vx^\top \mA \vy = \langle \vy \otimes \vx, \vec(\mA) \rangle$; Eqn. (40) of \citet{minka1997matrix}} \\
    &= \max_{\vu \in \gS^{d_1 - 1}} \E_{\mX \sim \pi}\left[ \dmu(\langle \mX, \bm\Theta_\star\rangle) \sum_{{\color{blue}m}=1}^{d_2} \langle \ve_{\color{blue}m} \otimes \vu, \mH(\pi; \bm\Theta_0)^{-1} \vec(\mX) \rangle^2 \right] \tag{Definition of $\mA_\mX$} \\
    &= \max_{\vu \in \gS^{d_1 - 1}} \sum_{{\color{blue}m}=1}^{d_2} (\ve_{\color{blue}m} \otimes \vu)^\top  \mH(\pi; \bm\Theta_0)^{-1} \mH(\pi; \bm\Theta_\star) \mH(\pi; \bm\Theta_0)^{-1}  (\ve_{\color{blue}m} \otimes \vu) \\
    &\leq 2 \max_{\vu \in \gS^{d_1 - 1}} \sum_{{\color{blue}m}=1}^{d_2} (\ve_{\color{blue}m} \otimes \vu)^\top  \mH(\pi; \bm\Theta_0)^{-1} (\ve_{\color{blue}m} \otimes \vu) \tag{\Cref{lem:jun5}, $\mA \mB \mA \preceq \mA \mC \mA$ for positive semi-defnite $\mA, \mB, \mC$} \\
    &= 2 \max_{\vu \in \gS^{d_1 - 1}} \sum_{{\color{blue}m}=1}^{d_2} \vu^\top \left(  (\mH(\pi; \bm\Theta_0)^{-1})_{{\color{teal}\gI_{\color{blue}m}^{(\col)}}, {\color{teal}\gI_{\color{blue}m}^{(\col)}}}  \right) \vu \tag{see \Cref{lem:minsker-variance} for the definition of ${\color{teal}\gI_{\color{blue}m}^{(\col)}}$} \\
    &= 2 \underbrace{\lambda_{\max}\left( \sum_{{\color{blue}m} = 1}^{d_2} (\mH(\pi; \bm\Theta_0)^{-1})_{{\color{teal}\gI_{\color{blue}m}^{(\col)}}, {\color{teal}\gI_{\color{blue}m}^{(\col)}}} \right)}_{=: H^{(\col)}(\pi; \bm\Theta_0)}.
\end{align}

All in all, we have that
\begin{equation}
    \bignormop{\E[\widetilde{\bm\Theta}_t \widetilde{\bm\Theta}_t^\top]} \leq 4 \left( g(\tau) + \frac{L_\mu^2 {\color{violet}E}^2}{\kappa_\star} \right) H^{(\col)}(\pi; \bm\Theta_0).
\end{equation}

Similarly, one can obtain 
\begin{equation}
    \bignormop{\E[\widetilde{\bm\Theta}_t^\top \widetilde{\bm\Theta}_t]} \leq 4 \left( g(\tau) + \frac{L_\mu^2 {\color{violet}E}^2}{\kappa_\star} \right) H^{(\row)}(\pi; \bm\Theta_0).
\end{equation}
We then conclude by plugging in the bound for ${\color{violet}E}^2$.
\qed

\subsection{A Glimpse of Unbounded Arm-Set: The Linear Case}
\label{app:unbounded}
Scrutinizing the above proof, Eqn.~(C.2) is the only part where we explicitly use $\kappa_\star$, which requires the boundedness of $\gX$.
When $\mu(z) = z$, this requirement vanishes, which is something that even \citet{jang2024lowrank} has seemingly missed.
Moreover, in this case, the requirement for $\bm\Theta_0$ to be asymptotically consistent is not needed as well.

To showcase this, we compare the Frobenius norm squared error rate of \texttt{GL-LowPopArt} when $\mu(z) = z$ with \citet[Corollary 3]{negahban2011estimation} in multivariate regression setup, which is equivalent to linear trace regression with $\gX = \{ \ve_i \vz_t^\top \}$ with $\vz_t \sim \gN_{d_2}(\vzero, \bm\Sigma)$ and $\pi = \mathrm{Unif}([d_1]) \otimes \gN_{d_2}(\vzero, \bm\Sigma)$~\citep[Example 1]{negahban2011estimation}.
Then it is easy to see that $\vx_t \sim \pi$ satisfies $\vx_t \sim \frac{1}{d_1} \sum_{i=1}^{d_1} \gN_{d_1 d_2}(\vzero, \bm\Sigma \otimes \ve_i \ve_i^\top)$, and thus, $\E[\vx_t \vx_t^\top] = \frac{1}{d_1} \bm\Sigma \otimes \mI_{d_1}$.
Suppose that the underyling GLM is $\gN(0, \nu^2)$.

By \Cref{thm:estimation-final,prop:improve} with $d_1 T$ observations, \texttt{GL-LowPopArt} attains
\begin{equation}
    \bignorm{\widehat{\bm\Theta} - \bm\Theta_\star}_F^2 \lesssim \frac{\nu^2 r (d_1 \vee d_2)}{T \lambda_{\min}(\bm\Sigma)} \log\frac{d_1 \vee d_2}{\delta}.
\end{equation}
\citet[Corollary 3]{negahban2011estimation} shows the following error rate for nuclear penalized estimator:
\begin{equation}
    \bignorm{\widehat{\bm\Theta} - \bm\Theta_\star}_F^2 \lesssim \frac{\nu^2 r \lambda_{\max}(\bm\Sigma)}{T \lambda_{\min}(\bm\Sigma)^2} \left( (d_1 \vee d_2) + \log\frac{1}{\delta} \right).
\end{equation}
Note that unless $\bm\Sigma$ is well-conditioned, our rate is strictly better than the nuclear penalized estimator by a factor of the condition number $\lambda_{\max}(\bm\Sigma) / \lambda_{\min}(\bm\Sigma)$.

\newpage
\section{\texorpdfstring{PROOF OF PROPOSITION~\ref{prop:improve} -- UPPER AND LOWER BOUNDS OF $\GL$}{PROOF OF PROPOSITION~3.1 -- UPPER AND LOWER BOUNDS OF GL}}
\label{app:improve}

Here, we largely follow the proof strategies of Appendix C.2 and D.2 of \citet{jang2024lowrank}.

\subsection{\texorpdfstring{Upper Bound of \Cref{def:GL-Design}}{Upper Bound of GL-Design}}
We have that
\begin{align}
    H^{(\col)}(\pi; \bm\Theta_\star) &= \lambda_{\max}\left( \sum_{{\color{blue}m}=1}^{d_2} \mD_{\color{blue} m}^{(\col)}(\pi; \bm\Theta_\star) \right) \\
    &\leq \sum_{{\color{blue}m}=1}^{d_2} \lambda_{\max}\left( \mD_{\color{blue} m}^{(col)}(\pi; \bm\Theta_\star) \right) \tag{$\lambda_{\max}$ is convex and 1-homogenous} \\
    &= \sum_{{\color{blue}m}=1}^{d_2} \max_{\vu \in \gS^{d_1 - 1}} \vu^\top \mD_{\color{blue} m}^{(\col)}(\pi; \bm\Theta_\star) \vu \\
    &= \sum_{{\color{blue}m}=1}^{d_2} \max_{\vu \in \gS^{d_1 - 1}} (\ve_{\color{blue}m} \otimes \vu)^\top \mH(\pi; \bm\Theta_\star)^{-1} (\ve_{\color{blue}m} \otimes \vu) \tag{see proof of Lemma~\ref{lem:minsker-variance}} \\
    &\leq \sum_{{\color{blue}m}=1}^{d_2} \max_{\vu \in \gS^{d_1d_2 - 1}} \vu^\top \mH(\pi; \bm\Theta_\star)^{-1} \vu \\
    &= d_2 \lambda_{\max}(\mH(\pi; \bm\Theta_\star)^{-1}) \\
    &= \frac{d_2}{\lambda_{\min}(\mH(\pi; \bm\Theta_\star))}.
\end{align}
One can similarly prove that $H^{(\row)}(\pi; \bm\Theta_\star) \leq \frac{d_1}{\lambda_{\min}(\mH(\pi; \bm\Theta_\star))}$, and the desired conclusion follows.

\subsection{\texorpdfstring{Lower Bound of $\GL(\pi; \bm\Theta_\star)$}{Lower Bound of GL}}

Again, by definition,
\begin{align}
    \GL(\pi; \bm\Theta_\star) &\geq \lambda_{\max}\left( \sum_{{\color{blue}m}=1}^{d_2} [(\mH(\pi; \bm\Theta_\star)^{-1})_{jk}]_{j,k \in \{\ell + d_1({\color{blue}m}-1) : \ell \in [d_1]\}} \right) \\
    &\geq \frac{1}{d_1} \tr\left( \sum_{{\color{blue}m}=1}^{d_2} [(\mH(\pi; \bm\Theta_0)^{-1})_{jk}]_{j,k \in \{\ell + d_1({\color{blue}m}-1) : \ell \in [d_1]\}} \right) \tag{$\lambda_{\max}(\mA) \geq \frac{1}{d} \tr(\mA)$ for any symmetric $\mA \in \sR^{d \times d}$} \\
    &= \frac{1}{d_1} \tr\left( \mH(\pi; \bm\Theta_\star)^{-1} \right) \\
    &\geq \frac{1}{d_1} \frac{(d_1 d_2)^2}{\tr\left( \mH(\pi; \bm\Theta_\star) \right)} \tag{AM-HM inequality on the eigenvalues of $\mH(\pi; \bm\Theta_\star)$},
\end{align}
and similarly,
\begin{equation}
    \GL(\pi; \bm\Theta_\star) \geq \frac{1}{d_2} \frac{(d_1 d_2)^2}{\tr\left( \mH(\pi; \bm\Theta_\star) \right)},
\end{equation}
i.e., $\GL(\pi; \bm\Theta_\star) \geq \frac{(d_1 d_2)^2}{(d_1 \wedge d_2) \tr(\mH(\pi; \bm\Theta_\star))} \geq \frac{(d_1 d_2) (d_1 \vee d_2)}{2 \tr(\mH(\pi; \bm\Theta_\star))}$.
\qed

\newpage
\section{DEFERRED DETAILS REGARDING OPTIMAL GL-DESIGN}
\subsection{\texorpdfstring{Proof of \Cref{prop:unit-ball}}{Proof of Proposition 3.2}}
\label{app:unit-ball}
First, by \Cref{prop:improve}, we have that
\begin{equation}
    \GL(\pi; \bm\Theta_\star) \leq \frac{d_1 \vee d_2}{\lambda_{\min}(\mH(\pi; \bm\Theta_\star))},
\end{equation}
and thus, by definition, it suffices to solve the following convex optimization problem:
\begin{equation}
    \max_{\pi \in \gP([d_1] \times [d_2])} \lambda_{\min}\left( \sum_{i=1}^{d_1} \sum_{j=1}^{d_2} \pi(i, j) \dmu_{i,j} \mathrm{vec}(\ve_i \tilde{\ve}_j^\top) \mathrm{vec}(\ve_i \tilde{\ve}_j^\top)^\top \right) = \min_{i,j} \pi(i, j)\dmu_{i,j},
\end{equation}
where we denote $\dmu_{i,j} := \dmu\!\left((\bm\Theta_\star)_{i,j}\right)$, and the equality holds because the matrix is diagonal.
Thus, we have a convex optimization problem of the following form:
\begin{equation}
    \max_{\pi \in \sR^{d_1 + d_2}} \min_{i,j} \pi(i, j)\dmu_{i,j}, \quad \text{subj. to} \quad \sum_{i=1}^{d_1} \sum_{j=1}^{d_2} \pi(i, j) = 1, \ \pi(i, j) \geq 0.
\end{equation}

To solve this max-min problem, we introduce an auxiliary variable $c \in \sR$ such that $c \leq \pi(i, j)\dmu_{i,j}$ for all $i \in [d_1], j \in [d_2]$. The optimization problem can be equivalently reformulated as a linear program:
\begin{equation}
    \max_{\pi, c} \ c \quad \text{subj. to} \quad \sum_{i=1}^{d_1} \sum_{j=1}^{d_2} \pi(i, j) = 1, \quad \pi(i, j) \geq 0, \ \pi(i, j)\dmu_{i,j} \geq c.
\end{equation}

Assuming $\dmu_{i,j} > 0$ for all $(i, j)$ (as the minimum would otherwise be trivially bounded by zero), the constraint $\pi(i, j)\dmu_{i,j} \geq c$ implies that
\begin{equation}
    \pi(i, j) \geq \frac{c}{\dmu_{i,j}}.
\end{equation}

Summing this over all $i \in [d_1]$ and $j \in [d_2]$, we obtain
\begin{equation}
    1 = \sum_{i=1}^{d_1} \sum_{j=1}^{d_2} \pi(i, j) \geq \sum_{i=1}^{d_1} \sum_{j=1}^{d_2} \frac{c}{\dmu_{i,j}} \implies c \leq \left( \sum_{i=1}^{d_1} \sum_{j=1}^{d_2} \dmu_{i,j}^{-1} \right)^{-1}.
\end{equation}

To demonstrate that this upper bound is achievable, we define a specific distribution $\pi^\star$ where the inequalities hold tightly as equalities:
\begin{equation}
    \pi^\star(i, j) = \frac{\dmu_{i,j}^{-1}}{\sum_{k=1}^{d_1} \sum_{l=1}^{d_2} \dmu_{k,l}^{-1}} \quad \forall i \in [d_1], j \in [d_2].
\end{equation}

Since $\dmu_{i,j} > 0$, it is clear that $\pi^\star(i, j) \in [0, 1]$ and $\sum_{i,j} \pi^\star(i, j) = 1$, confirming that $\pi^\star$ is a valid probability distribution on $[d_1] \times [d_2]$. Evaluating the objective at this optimal allocation $\pi^\star$, we find:
\begin{equation}
    \min_{i,j} \pi^\star(i, j)\dmu_{i,j} = \min_{i,j} \frac{\dmu_{i,j}^{-1}\dmu_{i,j}}{\sum_{k=1}^{d_1} \sum_{l=1}^{d_2} \dmu_{k,l}^{-1}} = \left( \sum_{k=1}^{d_1} \sum_{l=1}^{d_2} \dmu_{k,l}^{-1} \right)^{-1}.
\end{equation}

Thus, the maximum value of the objective function is exactly $\left( \sum_{i=1}^{d_1} \sum_{j=1}^{d_2} \dmu_{i,j}^{-1} \right)^{-1}$. This completes the proof.
\qed

\newpage
\section{\texorpdfstring{PROOF OF THEOREM~\ref{thm:estimation-0} -- ERROR BOUND OF STAGE~I}{PROOF OF THEOREM B.1 -- ERROR BOUND OF STAGE~I}}
\label{app:estimation-0}
We denote $N = N_1$ and $\pi = \pi_1$, and we recall/introduce the following notations:
\begin{align}
    \gL_N(\bm\Theta) &:= \frac{1}{N} \sum_{t=1}^N \left\{ m(\langle \mX_t, \bm\Theta \rangle) - y_t \langle \mX_t, \bm\Theta \rangle \right\} \\
    \mH(\pi; \bm\Theta) &:= \E_{\mX \sim \pi} \left[ \dmu(\langle \mX, \bm\Theta \rangle) \vec(\mX) \vec(\mX)^\top \right] \\
    \bm\Theta_0 &:= \argmin_{\bm\Theta \in \Omega} \left\{ \gL_N(\bm\Theta) + \lambda_N \bignormnuc{\bm\Theta} \right\} \label{eqn:optimization}
\end{align}

The proof here follows the standard M-estimation theory in high-dimensional statistics~\citep[Chapter 9 \& 10]{wainwright2019high}, which we reproduce here for completeness.

Let $\delta \in (0, 1)$ be given.
We collect all requirements on $N$ in {\color{violet}violet} to be aggregated at the end of the proof.

We begin with the following proposition, whose proof is provided in \Cref{app:lambda}, that provides $\lambda_N$'s for different GLMs such that the gradient of $\gL_N$ at the true parameter $\bm\Theta_\star$ is well bounded:
\begin{proposition}[Choosing $\lambda_N$]
\label{prop:lambda}
    Denote the centered random variable $\xi := y - \mu(\langle \mX, \bm\Theta_\star \rangle)$ and $\bar{\kappa}_\star := \E_{\mX \sim \pi} \left[ \dmu(\langle \mX, \bm\Theta_\star \rangle) \right]$.
    By setting $\lambda_N$ as described below, we have $\sP(\bignormop{\nabla \gL_N(\bm\Theta_\star)} \leq \frac{\lambda_N}{2}) \geq 1 - \frac{\delta}{4}$:
    \begin{itemize}
        \item[(i)] \textbf{$|\xi| \leq M$ a.s.}: $\lambda_N = 4 \sqrt{\frac{g(\tau) \bar{\kappa}_\star}{N} \log\frac{4(d_1 + d_2)}{\delta}},$ given that {\color{violet}$N \geq \frac{2M^2}{9 L_\mu g(\tau)} \log\frac{4(d_1 + d_2)}{\delta}$},
        
        \item[(ii)] \textbf{$\bignorm{\xi}_{\psi_1} \leq K$}:\footnote{Here, $\bignorm{\cdot}_{\psi_1}$ is the subexponential norm, defined as $\bignorm{\xi}_{\psi_1} := \inf\{ K > 0 : \E[\exp(|\xi| / K)] \leq 2 \}$. This encompasses Gaussian and Poisson distributions among others.} $\lambda_N = 6 K \sqrt{\frac{1}{N} \log\frac{4(d_1 + d_2)}{\delta}},$ given that {\color{violet}$N \geq \log\frac{4(d_1 + d_2)}{\delta}$}.
    \end{itemize}
\end{proposition}
We then define the distribution-dependent constant $C > 0$ as follows:
\begin{align}
    C &:= \begin{cases}
        16 g(\tau) \bar{\kappa}_\star, & \text{if the GLM is $M$-bounded} \\
        36 K^2, & \text{if the GLM has finite subexponential norm of $K$}
    \end{cases}
\end{align}

Throughout, we will suppose that the event $\bignormop{\nabla \gL_N(\bm\Theta_\star)} \leq \frac{\lambda_N}{2}$ holds.

We now recall the definition of restricted strong convexity (RSC):
\begin{definition}[Definition 9.15 of \citet{wainwright2019high}]
    We say that $\gL_N$ satisfies the \textbf{restricted strong convexity (RSC)} with $\alpha, \tau_N, R \geq 0$ if
    \begin{equation}
        \gL_N(\bm\Theta_\star + \Delta) - \gL_N(\bm\Theta_\star) - \langle \nabla \gL_N(\bm\Theta_\star), \Delta \rangle \geq \frac{\alpha}{2} \bignorm{\Delta}_F^2 - \tau_N^2 \bignorm{\Delta}_\nuc^2, \quad \forall \Delta : \bignorm{\Delta}_F \leq R.
    \end{equation}
\end{definition}
The next lemma shows that RSC holds with high probability:
\begin{lemma}
\label{lem:RSC}
    Suppose that {\color{violet} $N \gtrsim \log d$}.
    Then, there exists a constant $\tilde{R}_s > 0$ (that only depends on $R_s$) such that w.p. at least $1 - \frac{\delta}{4}$, \textbf{RSC} holds with $\alpha = \tilde{R}_s \lambda_{\min}(\mH(\pi; \bm\Theta_\star))$, $\tau_N^2 \asymp L_\mu \sqrt{\frac{1}{N} \log\frac{d_1 + d_2}{\delta}}$, and $R = \frac{1}{\sqrt{d_1 + d_2}}$.
\end{lemma}

As the subspace Lipschitz constant~\citep[Definition 9.18]{wainwright2019high} of $\bignormnuc{\cdot}$ is $\sqrt{r}$, we conclude by invoking the ``master theorem'' for nuclear-norm regularized estimator~\citep[Corollary 9.20]{wainwright2019high}:
\begin{equation}
    \sP\left( \bignormnuc{\widehat{\bm\Theta} - \bm\Theta_\star} \leq \frac{6 r}{\lambda_{\min}(\mH(\pi; \bm\Theta_\star))} \sqrt{\frac{C}{N} \log\frac{2(d_1 + d_2)}{\delta}} \right) \geq 1 - \frac{\delta}{2},
\end{equation}
\emph{given} that $L_\mu \sqrt{\frac{1}{N} \log\frac{d_1 + d_2}{\delta}} r \lesssim \lambda_{\min}(\mH(\pi; \bm\Theta_\star))$ and $\frac{1}{\lambda_{\min}(\mH(\pi; \bm\Theta_\star))} \sqrt{\frac{C r}{N} \log\frac{d_1 + d_2}{\delta}} \lesssim \frac{1}{\sqrt{d_1 + d_2}}$.
These conditions hold when {\color{violet} $N \gtrsim C \lambda_{\min}(\mH(\pi; \bm\Theta_\star))^{-2} (d_1 + d_2) r \log\frac{d_1 + d_2}{\delta}$}, and we conclude by taking the union bound with \Cref{prop:lambda}.
\qed

\begin{remark}
    We believe that our analysis can be extended to the general $\ell_q$-constraint on the singular values of $\bm\Theta_\star$ for $q \in [0, 1)$ as in \citet{fan2019generalized}, and to the case where $\Omega$ is a smooth matrix manifold~\citep{absil2008manifold} using tools from manifold optimization~\citep{boumal2023manifold,yang2014manifold}.
\end{remark}

\subsection{\texorpdfstring{Proof of \Cref{lem:RSC}: Restricted Strong Convexity}{Proof of Lemma ??}}
By Taylor's expansion with integral remainder, we have
\begin{align}
    &\gL_N(\bm\Theta_\star + \Delta) - \gL_N(\bm\Theta_\star) - \langle \nabla \gL_N(\bm\Theta_\star), \Delta \rangle \\ &= \frac{1}{N} \sum_{t=1}^N \left\{ m(\langle \bm\Theta_\star + \Delta, \mX_t \rangle) - m(\langle \bm\Theta_\star, \mX_t \rangle) - \langle \bm\Theta_\star, \mX_t \rangle m'(\langle \bm\Theta_\star, \mX_t \rangle) \right\} \\
    &= \frac{1}{N} \sum_{t=1}^N \langle \Delta, \mX_t \rangle^2 \int_0^1 (1 - z) \underbrace{m''\left( \langle \bm\Theta_\star + z \Delta, \mX_t \rangle \right)}_{=\dmu\left( \langle \bm\Theta_\star + z \Delta, \mX_t \rangle \right)} dz.
\end{align}

As $\mu$ is self-concordant, we can utilize the following self-concordance control lemma:
\begin{lemma}[Lemma 9 of \citet{abeille2021logistic}]
    For any $R_s$-self-concordant $\mu$ that is monotone increasing, the following holds: for any $z_1, z_2 \in \sR$,
    \begin{equation}
        \dmu(z_2) \exp(-R_s|z_1 - z_2|) \leq \dmu(z_1) \leq \dmu(z_2) \exp( R_s |z_1 - z_2|).
    \end{equation}
\end{lemma}

We choose $R = \frac{1}{\sqrt{d}}$, i.e., $\bignorm{\Delta}_F \leq \frac{1}{\sqrt{d}}$. Then, Cauchy-Schwarz inequality and $\bignormop{\mX_t} \leq 1$ imply $| \langle \Delta, \mX_t \rangle | \leq 1$.
Thus, by the above self-concordance lemma,
\begin{equation}
    \int_0^1 (1 - z) \dmu\left( \langle \bm\Theta_\star + z \Delta, \mX_t \rangle \right) dz \geq \int_0^1 (1 - z) e^{- R_s z} \dmu\left( \langle \bm\Theta_\star, \mX_t \rangle \right) dz
    \geq \tilde{R}_s \dmu\left( \langle \bm\Theta_\star, \mX_t \rangle \right),
\end{equation}
where we define the constant
\begin{equation}
    \tilde{R}_s :=
    \begin{cases*}
        \frac{1}{2}, & \quad $R_s = 0$ \\
        \frac{R_s - 1 + e^{-R_s}}{R_s^2}, & \quad $R_s > 0$.
    \end{cases*}
\end{equation}

Let us denote the empirical Hessian of $\gL_N$ at $\bm\Theta_\star$ as
\begin{equation}
    \widehat{\mH}_N(\pi; \bm\Theta_\star) := \frac{1}{N} \sum_{t=1}^N \dmu(\langle \mX_t, \bm\Theta_\star \rangle) \vec(\mX_t) \vec(\mX_t)^\top.
\end{equation}

Thus, we have that
\begin{align}
    &\gL_N(\bm\Theta_\star + \Delta) - \gL_N(\bm\Theta_\star) - \langle \nabla \gL_N(\bm\Theta_\star), \Delta \rangle \\
    &\geq \tilde{R}_s \mathrm{vec}(\Delta)^\top \widehat{\mH}_N(\pi; \bm\Theta_\star) \mathrm{vec}(\Delta) \\
    &= \tilde{R}_s \mathrm{vec}(\Delta)^\top \mH(\pi; \bm\Theta_\star) \mathrm{vec}(\Delta) + \tilde{R}_s \mathrm{vec}(\Delta)^\top \left( \widehat{\mH}_N(\pi; \bm\Theta_\star) - \mH(\pi; \bm\Theta_\star) \right) \mathrm{vec}(\Delta) \\
    &\geq \tilde{R}_s \lambda_{\min}\left( \mH(\pi; \bm\Theta_\star) \right) \bignorm{\Delta}_F^2 - \tilde{R}_s \bignormnuc{\Delta}^2 \underbrace{\left| \left( \frac{\mathrm{vec}(\Delta)}{\bignormnuc{\Delta}} \right)^\top \left( \widehat{\mH}_N(\pi; \bm\Theta_\star) - \mH(\pi; \bm\Theta_\star) \right) \frac{\mathrm{vec}(\Delta)}{\bignormnuc{\Delta}} \right|}_{\triangleq \gE_N(\Delta)}
\end{align}

Using the standard empirical process theory and peeling argument over the Frobenius ball, the following lemma establishes the uniform boundedness of $\gE_N$, whose proof is provided in \Cref{app:uniform-bound}:
\begin{lemma}
\label{lem:uniform-bound}
    $\sP\left( \gE_N(\Delta) \lesssim L_\mu \sqrt{\frac{\log\frac{d}{\delta}}{N}}, \ \forall \Delta : \bignormnuc{\Delta}= 1 \right) \geq 1 - \frac{\delta}{4},$ provided that {\color{violet} $N \gtrsim \log d$}.
\end{lemma}

Combining everything gives the desired result.
\qed

\subsection{\texorpdfstring{Proof of \Cref{lem:uniform-bound}: Uniform Bounding of $\gE_N$}{Proof of Lemma ??}}
\label{app:uniform-bound}
This proof is based on the peeling argument over the Frobenius ball in \citet[Theorem 10.17]{wainwright2019high}, as well as the standard empirical process theory tools such as Rademacher symmetrization and contraction inequalities; see \citet[Chapter 3.4 \& 4]{wainwright2019high}. We will introduce the tools as we go along.

For notational simplicity, we will denote $\frac{\Delta}{\bignormnuc{\Delta}}$ as $\Delta$ satisfying $\bignormnuc{\Delta} = 1$.

First note that
\begin{equation}
    \mathrm{vec}(\Delta)^\top \left( \widehat{\mH}_N(\pi; \bm\Theta_\star) - \mH(\pi; \bm\Theta_\star) \right) \mathrm{vec}(\Delta) = \frac{1}{N} \sum_{t=1}^N \left\{ \dmu(\langle \bm\Theta_\star, \mX_t \rangle) \langle \Delta, \mX_t \rangle^2 - \E\left[ \dmu(\langle \bm\Theta_\star, \mX_t \rangle) \langle \Delta, \mX_t \rangle^2 \right] \right\}.
\end{equation}
Let us define $f_\Delta(\mX) := \dmu(\langle \bm\Theta_\star, \mX \rangle) \langle \Delta, \mX \rangle^2$, and $\gF := \left\{ f_\Delta : \bignormnuc{\Delta} = 1 \right\}$ be the function class indexed by $\Delta$ of unit nuclear norm.
Note that by H\"{o}lder's inequality and $\bignormop{\mX_t} \leq 1$, we have $|\langle \Delta, \mX_t \rangle| \leq \bignormnuc{\Delta} \bignormop{\mX_t} \leq 1$.
Thus, $\gF$ is uniformly bounded by $L_\mu$.

We now recall the following standard result from empirical process theory:
\begin{lemma}[Glivenko-Cantelli Property via Rademacher Complexity, Theorem 4.10 of \citet{wainwright2019high}]
\label{lem:rademacher}
    Let $\mathcal{F}$ be a function class uniformly bounded by $b$. For any $\delta \in (0, 1)$, with probability at least $1 - \delta$,
    \begin{equation}
        \sup_{f \in \mathcal{F}} \left| \frac{1}{N} \sum_{t=1}^N f(\mX_t) - \E[f(\mX)] \right| \leq 2 \gR_N(\mathcal{F}) + b \sqrt{\frac{2 \log(1/\delta)}{N}},
    \end{equation}
    where $\gR_N(\mathcal{F})$ is the Rademacher complexity, defined as
    \begin{equation}
        \gR_N(\mathcal{F}) := \frac{1}{N} \E_{\{\mX_t\}, \{\epsilon_t\}} \left[ \sup_{f \in \mathcal{F}} \left| \sum_{t=1}^N \epsilon_t f(\mX_t) \right| \right],
    \end{equation}
    where $\epsilon_t \overset{i.i.d.}{\sim} \mathrm{Unif}(\{-1, +1\})$ are Rademacher random variables.
\end{lemma}

With this, we have that with probability at least $1 - \frac{\delta}{2}$,
\begin{equation}
    \sup_{\Delta: \bignormnuc{\Delta} = 1} \left| \gE_N(\Delta) \right| \leq 2 \gR_N(\gF) + L_\mu \sqrt{\frac{2 \log\frac{2}{\delta}}{N}}.
\end{equation}

We now recall the seminal Ledoux-Talagrand contraction inequality~\citep{ledoux-talagrand} for bounding $\gR_N(\gF)$:
\begin{lemma}[Ledoux-Talagrand Contraction, Theorem 4.12 of \citet{ledoux-talagrand}]
\label{lem:ledoux-talagrand}
    If the functions $\phi_t : \sR \to \sR$ are $L$-Lipschitz and satisfy $\phi_t(0) = 0$, then 
    \begin{equation}
        \E_{\{\mX_t\}, \{\epsilon_t\}} \left[ \sup_{f \in \gF} \left| \sum_{t=1}^N \epsilon_t \phi_t(f(\mX_t)) \right| \right] \leq 2 L \E_{\{\mX_t\}, \{\epsilon_t\}} \left[ \sup_{f \in \gF} \left| \sum_{t=1}^N \epsilon_t f(\mX_t) \right| \right].
    \end{equation}
\end{lemma}

Define $\phi_t(z) := \dmu(\langle \mX_t, \bm\Theta_\star \rangle) z^2$. For $|z| \leq 1$, the function $\phi_t$ is locally Lipschitz with constant $L_k = 2 L_\mu$.
Applying \Cref{lem:ledoux-talagrand}, we obtain:
\begin{align}
    \gR_N(\gF) &\leq 2 (2 L_\mu) \, \E_{\{\mX_t\}, \{\epsilon_t\}} \left[ \sup_{\Delta: \bignormnuc{\Delta} = 1} \left| \frac{1}{N} \sum_{t=1}^N \epsilon_t \langle \Delta, \mX_t \rangle \right| \right] \\
    &= 4 L_\mu \, \E_{\{\mX_t\}, \{\epsilon_t\}} \left[ \sup_{\bignormnuc{\Delta} = 1} \left| \left\langle \Delta, \frac{1}{N} \sum_{t=1}^N \epsilon_t \mX_t \right\rangle \right| \right] \\
    &= 4 L_\mu \, \E_{\{\mX_t\}, \{\epsilon_t\}} \left[ \bignormop{\frac{1}{N} \sum_{t=1}^N \epsilon_t \mX_t} \right].
\end{align}

We now invoke the following expectation version of the matrix Bernstein inequality:
\begin{lemma}[Matrix Bernstein Inequality, Theorem 6.1.1 of \citet{tropp2015survey}]
\label{lem:matrix-bernstein}
    Let $\mA_1, \ldots, \mA_N$ be $d_1 \times d_2$ matrices with $\bignormop{\mA_i} \leq L$, $\E[\mA_i] = \mA$ and $V_N := \frac{1}{N} \max\left\{ \bignormop{\sum_{i=1}^N \E[\mA_i \mA_i^\top]}, \bignormop{\sum_{i=1}^N \E[\mA_i^\top \mA_i]} \right\}$ be the matrix variance statistics.
    Then,
    \begin{equation}
        \E \left[ \bignormop{\frac{1}{N} \sum_{t=1}^N \mA_t} \right] \leq \sqrt{\frac{2 V_N \log (d_1 + d_2)}{N}} + \frac{L \log (d_1 + d_2)}{3 N}.
    \end{equation}
\end{lemma}

Applying the above with $L = 1$ and $V_N \leq 1$, we have that
\begin{equation}
    \gR_N(\gF) \leq 4 L_\mu \left( \sqrt{\frac{2 \log d}{N}} + \frac{\log d}{3 N} \right).
\end{equation}

Combining everything, we have that with probability at least $1 - \frac{\delta}{4}$,
\begin{align}
    \sup_{\Delta: \bignormnuc{\Delta} = 1} \left| \gE_N(\Delta) \right| &\leq 8 L_\mu \left( \sqrt{\frac{2 \log d}{N}} + \frac{\log d}{3 N} \right) + L_\mu \sqrt{\frac{2 \log\frac{4}{\delta}}{N}} \\
    &\leq 20 L_\mu \sqrt{\frac{\log d}{N}} + L_\mu \sqrt{\frac{2 \log\frac{4}{\delta}}{N}},
\end{align}
where the last inequality holds as long as {\color{violet} $N \geq \frac{\log d}{9}$}.
\qed

\subsection{\texorpdfstring{Proof of \Cref{prop:lambda}: Choosing $\lambda_N$}{Choosing the Nuclear-Norm Regularization Coefficient}}
\label{app:lambda}

We significantly deviate from the original proof of \citet{lee2024glm} in two ways.
First, we only need to bound the Lipschitz constant of $\gL_N$ at $\bm\Theta_\star$ instead of globally.
Second, using this and the fact that the covariates are bounded ($\bignormop{\mX_t} \leq 1$), we can utilize appropirate matrix Bernstein inequalities~\citep{tropp2012matrix,tropp2015survey} instead of covering arguments, which saves us additional dimension dependencies, leads to $\lambda_N$ scaling as $\sqrt{1 / N}$ for a wide range of GLMs, and gives simpler proofs.

For notational simplicity, we denote the centered random variable $\xi_t := \mu(\langle \mX_t, \bm\Theta_\star \rangle) - y_t$.
Then note that
\begin{equation}
    \nabla \gL_N(\bm\Theta_\star) = \frac{1}{N} \sum_{t=1}^N (\mu_t(\bm\Theta_\star) - y_t) \mX_t = \frac{1}{N} \sum_{t=1}^N \xi_t \mX_t.
\end{equation}

Let us prove each part separately:

\subsubsection{\texorpdfstring{Proof of (i) -- GLM bounded by $M$}{Proof of (i) -- GLM bounded by M}}
Here, ``bounded by $M$'' means $|y - \langle \mX, \bm\Theta_\star \rangle| \leq M \ a.s.$.
The original proof of \citet{lee2024glm} is too loose, and thus we instead utilize the matrix Bernstein inequality for bounded random matrices, as stated in \Cref{lem:matrix-bernstein}.
Let $\mA_t = \xi_t \mX_t$, which satisfies $\mA \triangleq \E[\mA_t] = \vzero$ and $\bignormop{\mA_t} \leq M$.
Its matrix variance statistics is bounded as
\begin{align}
    V_N &= \frac{1}{N} \max\left\{ \bignormop{\sum_{t=1}^N \E[\xi_t^2 \mX_t \mX_t^\top]}, \bignormop{\sum_{t=1}^N \E[\xi_t^2 \mX_t^\top \mX_t]} \right\} \\
    &= \frac{1}{N} \max\left\{ \bignormop{\sum_{t=1}^N \E\left[ \mX_t \mX_t^\top \E[\xi_t^2 \mid \mX_t] \right]}, \bignormop{\sum_{t=1}^N \E\left[ \mX_t^\top \mX_t \E[\xi_t^2 \mid \mX_t] \right]} \right\} \tag{Tower property of expectation} \\
    &\leq \frac{g(\tau)}{N} \sum_{t=1}^N \E[\dmu(\langle \mX_t, \bm\Theta_\star \rangle)] \tag{$\E[\xi_t^2] = \Var[y | \mX_t] = g(\tau) \dmu(\langle \mX_t, \bm\Theta_\star \rangle)$, $\bignormop{\mX_t} \leq 1$} \\
    &\leq g(\tau) \underbrace{\E_{\mX \sim \pi} \left[ \dmu(\langle \mX, \bm\Theta_\star \rangle) \right]}_{\triangleq \bar{\kappa}_\star}.
\end{align}

We now recall the high-probability version of the matrix Bernstein inequality:
\begin{lemma}[Matrix Bernstein Inequality, Theorem 6.1.1 of \citet{tropp2015survey}]
\label{lem:matrix-bernstein-whp}
    Let $\mA_1, \ldots, \mA_N$ be $d_1 \times d_2$ matrices with $\bignormop{\mA_i} \leq M$, $\E[\mA_i] = \vzero$ and $V_N := \frac{1}{N} \max\left\{ \bignormop{\sum_{i=1}^N \E[\mA_i \mA_i^\top]}, \bignormop{\sum_{i=1}^N \E[\mA_i^\top \mA_i]} \right\}$ be the matrix variance statistics.
    Then, for any $\delta \in (0, 1)$, as long as $N \geq \frac{4 M^2}{9 V_N} \log\frac{4(d_1 + d_2)}{\delta}$,
    \begin{equation}
        \sP\left( \bignormop{\frac{1}{N}\sum_{i=1}^N \mA_i} \leq 2 \sqrt{\frac{V_N}{N} \log\frac{4(d_1 + d_2)}{\delta}} \right) \geq 1 - \frac{\delta}{4}.
    \end{equation}
\end{lemma}

The above lemma implies that as long as {\color{violet} $N \geq \frac{4 M^2}{9 g(\tau) \bar{\kappa}_\star} \log\frac{4(d_1 + d_2)}{\delta}$}, the following holds w.p. at least $1 - \frac{\delta}{4}$:
\begin{equation}
    \bignormop{\nabla \gL_N(\bm\Theta_\star)} \leq 2 \sqrt{\frac{g(\tau) \bar{\kappa}_\star}{N} \log\frac{4(d_1 + d_2)}{\delta}}.
\end{equation}

\begin{remark}
    \citet{lafond2015lexponential,klopp2014general,klopp2015multinomial} have utilized similar proof techniques involving matrix Bernstein inequalities for bounded random variables.
    In the next section, we show how to extend the proof beyond boundedness.
\end{remark}

\subsubsection{Proof of (ii) -- GLM with Finite SubExponential Norm}
Here, we utilize the following subexponential version of matrix Bernstein inequality, whose proof we provide in \Cref{app:subexp-bernstein} for completeness:
\begin{lemma}
\label{lem:matrix-bernstein-subexp}
    Let $\mA_1, \ldots, \mA_N$ be $d_1 \times d_2$ random matrices with $\bignormop{\mA_t} \leq 1$, and let $\xi_1, \cdots, \xi_N$ be real-valued random variables with $\E[\xi_t | \mA_t] = 0$ and $\bignorm{\xi}_{\psi_1} \leq K$ conditionally on $\mA_t$.
    Then, for any $\delta \in (0, 1)$, as long as $N \geq \log\frac{4(d_1 + d_2)}{\delta}$,
    \begin{equation}
        \sP\left( \bignormop{\frac{1}{N}\sum_{t=1}^N \xi_t \mA_t} \leq 3 K \sqrt{\frac{1}{N} \log\frac{4(d_1 + d_2)}{\delta}} \right) \geq 1 - \frac{\delta}{   }.
    \end{equation}
\end{lemma}
The statement then immediately follows from the above lemma with $\mA_t = \mX_t$.
\qed

\subsubsection{\texorpdfstring{Proof of \Cref{lem:matrix-bernstein-subexp}}{Proof of Sub-Exponential Matrix Bernstein Inequality}}
\label{app:subexp-bernstein}
We largely follow the proof of \citet[Theorem 4.1]{tropp2012matrix}, but adopt to rectangular matrices via Hermitian dilations.
Unlike the remark in \citet[Remark 6.3]{tropp2012matrix}, as we deal with scalar-valued random variables with matrix-variate coefficients, the rectangular extension is rather straightforward.

First, by definition and Taylor series expansion, $\bignorm{\xi_t}_{\psi_1} \leq K$ conditionally on $\mA_t$ implies that
\begin{equation}
\label{eqn:moment-bound}
    \E\left[ \sum_{p=1}^\infty \frac{|\xi_t|^p}{p! \cdot K^p} \Big| \mA_t \right] \leq 1 \Longrightarrow \E\left[ |\xi_t|^p \mid \mA_t \right] \leq p! K^p, \ \forall p \geq 1.
\end{equation}

Now, we consider the Hermitian dilation of $\xi_t \mA_t$, which gives
\begin{equation}
    \gH(\xi_t \mA_t) :=
    \begin{bmatrix}
        \vzero & \xi_t \mA_t \\ \xi_t \mA_t^\top & \vzero
    \end{bmatrix}
    = \xi_t \gH(\mA_t).
\end{equation}
Thus, note that for each $p \geq 2$, we have the following moment bound:
\begin{align}
    \E\left[ \gH(\xi_t \mA_t)^p \right] &\preceq \E[|\xi_t|^p] \mI \tag{$\bignormop{\gH(\mA_t)} = \bignormop{\mA_t} \leq 1$} \\
    &= \E\left[ \E\left[ |\xi_t|^p \mid \mX_t \right] \right] \tag{Tower property of expectation} \\
    &\preceq p! K^p \mI. \tag{Eqn.~\eqref{eqn:moment-bound}}
\end{align}
Then, the moment generating function is bounded as follows: for any $\theta \in (0, K^{-1})$,
\begin{align}
    \E\left[ e^{\theta \gH(\xi_t \mA_t)} \right] &= \mI + \theta \E[\gH(\xi_t \mA_t)] + \sum_{p=2}^\infty \frac{\theta^p \E[\gH(\xi_t \mA_t)^p]}{p!} \\
    &\preceq \mI + \mI \sum_{p=2}^\infty (\theta K)^p \tag{The above moment bound} \\
    &= \mI + \frac{(\theta K)^2}{1 - \theta K} \mI \\
    &\preceq \exp\left( \frac{(\theta K)^2}{1 - \theta K} \mI \right). \tag{Eqn. (2.3) of \citet{tropp2012matrix}}
\end{align}

We then apply the Master Tail Bound for the sum of independent random matrices~\citep[Corollary 3.7]{tropp2012matrix} to obtain
\begin{equation}
    \sP\left( \bignormop{\frac{1}{N} \sum_{t=1}^N \gH(\xi_t \mA_t)} \geq z \right) \leq (d_1 + d_2) \underbrace{\inf_{\theta \in (0, K^{-1})} \exp\left( -\theta N z + \frac{(\theta K)^2}{1 - \theta K} \bignormop{\sum_{t=1}^N \mI} \right)}_{(*)}, \quad \forall z \geq 0.
\end{equation}

With $\bignormop{\sum_{t=1}^N \mI} = N$ and optimizing for $\theta$, we have that
\begin{equation}
    (*) \leq \exp\left( -N \left( \sqrt{1 + \frac{z}{K}} - 1 \right)^2 \right).
\end{equation}
Reparametrizing and {\color{violet} assuming that $N \geq \log\frac{4(d_1 + d_2)}{\delta}$} we have:
\begin{equation}
    \sP\left( \bignormop{\frac{1}{N} \sum_{t=1}^N \gH(\xi_t \mA_t)} \geq 3 K \sqrt{\frac{1}{N} \log\frac{4(d_1 + d_2)}{\delta}} \right) \leq \frac{\delta}{4}
\end{equation}
Finally, we conclude by noting that
\begin{equation}
    \bignormop{\frac{1}{N} \sum_{t=1}^N \gH(\xi_t \mA_t)} = \bignormop{\frac{1}{N} \sum_{t=1}^N \xi_t \mA_t}.
\end{equation}
\qed

\newpage
\section{SUMMARY OF PRIOR WORKS ON GENERALIZED TRACE REGRESSION}
\label{app:comparison}

Here, we review two relevant prior works on generalized trace regression~\citep{fan2019generalized,kang2022generalized}, with focus on their problem setup, assumptions, and results.

\subsection{\texorpdfstring{Nuclear Penalized MLE of \citet{fan2019generalized}}{Nuclear Penalized MLE (Fan et al., 2019)}}
We first list their assumptions:
{
    \renewcommand{\theassumption}{C's}  %
    \begin{assumption}
        Assume the following conditions hold:
        \begin{enumerate}
            \item[(C1)] $\bignorm{\mathrm{vec}(\mX_t)}_{\psi_2} \leq K < \infty$ for some constant $K > 0$
            \item[(C2)] $|\dmu(z)| \leq L_\mu$ for any $z \in \sR$
            \item[(C3)] $\lambda_{\min}(\mH(\pi; \bm\Theta_\star)) > 0$
            \item[(C4)] $\|\bm\Theta_\star\|_F \gtrsim \alpha \sqrt{d_1 \vee d_2}$ for some constant $\alpha > 0$
            \item[(C5)] $|\ddmu(z)| \leq |z|^{-1}$ for any $|z| > 1$\footnote{This slightly resembles the general self-concordance condition with a stretch function~\citep{liu2024free}.}
        \end{enumerate}
    \end{assumption}
}
\citet[Lemma 1]{fan2019generalized} shows that under (C1-C2), $\lambda_{N_1} \asymp K \sqrt{\frac{d_1 \vee d_2}{N_1}}$ suffices to control the gradient of the log-likelihood loss; \citet[Lemma 2]{fan2019generalized} shows that under (C1-C5), local restricted convexity holds with curvature $\lambda_{\min}(\mH(\pi; \bm\Theta_\star))$.

Compared to our setting, our bounded covariate assumption is stronger than their (C1), but we do not require (C4) nor (C5), and by utilizing the boundedness, we obtain tighter guarantees.
Moreover, the condition (C5) excludes many GLMs, such as Poisson distribution.

\subsection{\texorpdfstring{Stein's Lemma-based Estimator of \citet{kang2022generalized}}{Stein's Lemma-based Estimator (Kang et al., 2022)}}
Their Stein's lemma-based estimator achieves the following error bound~\citep[Theorem 4.1]{kang2022generalized}:
\begin{equation}
    \bignorm{\widehat{\bm\Theta}^{\text{Kang,1}} - \bm\Theta_\star}_F^2 \lesssim \frac{M(\pi) (d_1 \vee d_2) r}{\overline{\kappa}(\pi)^2 N},
\end{equation}
where $\overline{\kappa}(\pi) := \E_{\mX \sim \pi}[\dmu(\langle \mX, \bm\Theta_\star \rangle)]$,
\textit{given} that the following assumption holds:
{
    \renewcommand{\theassumption}{3.3}  %
    \begin{assumption}\label{assum:M-pi}
        $\pi$ has a continuously differentiable density $\vp = (p_{ij})$ supported over $\gX$ such that $\E[(S^\vp(\mX))_{ij}^2] \leq M(\pi)$ for all $i, j$, where the matrix-valued score function is defined as $S^\vp(\mX) := \left( - \nabla_{X_{ij}} \log p_{ij}(X_{ij}) \right)_{ij}$.
        Moreover, the columns or rows of random matrix $\mX$ are pairwisely independent.
    \end{assumption}
}
We note that $M(\pi)$ often scales with the dimension.
For $\gX = \gB_F^{d_1 \times d_2}(1)$ and $\pi \sim \gN(\vzero, \frac{c}{d_1 d_2 \log T} \mI)$ for a constant $c > 0$, it can be computed that $M(\pi) \lesssim d_1 d_2$~\citep[Appendix H.2]{jang2024lowrank}, which is what we use in \Cref{tab:comparison}.
Otherwise, we set $M \lesssim (d_1 \vee d_2)^2$ as suggested by \citet{kang2022generalized}.

The continuous differentiability assumption is in place because they rely on the generalized Stein's lemma~\citep[Proposition 1.4]{stein2004simulation}.
This limits their applicability to discrete arm-sets, while our framework is applicable for both continuous and discrete arm-sets.
Also, from the perspective of optimal experimental design, it is not clear how to optimize their bound for $\pi$ while satisfying the conditions above.
Even without those conditions, the function $\pi \mapsto \frac{M(\pi)}{\overline{\kappa}(\pi)}$ is likely to be nonconvex.
On the other hand, we mention that their result is applicable to the semiparametric setting of $y_t = \mu(\langle \mX_t, \bm\Theta_\star \rangle) + \eta_t$ for some subGaussian noise $\eta_t$.

\subsection{\texorpdfstring{Nuclear Penalized MLE of \citet{kang2022generalized}}{Nuclear Penalized MLE (Kang et al., 2022)}}
We also briefly elaborate on the nuclear penalized MLE as analyzed in \citet[Appendix J.4]{kang2022generalized}:
\begin{equation}
    \bignorm{\widehat{\bm\Theta}^{\text{Kang,2}} - \bm\Theta_\star}_F^2 \lesssim \frac{(d_1 \vee d_2) r \sigma(\pi)^2}{c_\mu^2 \lambda_{\min}(\mV(\pi))^2 N},
\end{equation}
\textit{given} that the following assumptions hold:
{
    \renewcommand{\theassumption}{J.1}  %
    \begin{assumption}\label{assum:J1}
        $\pi \in \gP(\gA)$ is such that $\vec(\mX)$ is $\sigma(\pi)$-subGaussian\footnote{This means that for any unit vector $\vu \in \gS^{d_1 d_2 - 1}$, $\vu^\top \vec(\mX)$ is $\sigma(\pi)$-subGaussian.} for $\mX \sim \pi$.
    \end{assumption}
}
{
    \renewcommand{\theassumption}{J.2}  %
    \begin{assumption}\label{assum:J2}
        There is two (dimension-independent) constants $S_2 \leq S$ such that $\gA \subseteq \gB^{d_1 \times d_2} \triangleq \gB^{d_1 \times d_2}_F(S) \cap \gB^{d_1 \times d_2}_{\mathrm{op}}(S_2)$ and likewise for $\bm\Theta_\star$.
    \end{assumption}
}
{
    \renewcommand{\theassumption}{J.3}  %
    \begin{assumption}\label{assum:J3}
        There is a constant $c_2 > 0$ such that
        \begin{equation}
        \label{eqn:cmu}
            c_\mu := \min\left( \inf_{\mX \in \gA, \bm\Theta \in \gB^{d_1 \times d_2}} \dmu(\langle \mX, \bm\Theta \rangle), \ \inf_{|z| \leq (S + 2)\sigma c_2} \dmu(z) \right) > 0.
        \end{equation}
    \end{assumption}
}
\citet[Assumption J.1]{kang2022generalized} assumed that $\lambda_{\min}(\mV(\pi)) \asymp \sigma(\pi)^2 \asymp \frac{1}{d_1 d_2}$, which was also the assumption made by \citet[Assumption 2]{lu2021generalized}.
Indeed, as argued by the two works, one can easily find $\pi$ that satisfies the above conditions, e.g. $\Unif(\gB_F^{d_1 \times d_2}(1))$ or require for ``the convex hull of a subset of arms to contain a ball with radius $R \leq 1$ that does not scale with $d_1$ or $d_2$.''
Similarly, it is unclear how to optimize for $\pi$ in the optimal experimental design setup.
Moreover, the above assumption may fail even for a simple arm-set.
Consider $\pi \sim \Unif(\gM)$, where recall that $\gM$ is the matrix completion basis (see \Cref{prop:unit-ball}).
Then, one can show that $\lambda_{\min}(\mV(\pi)) = \frac{1}{d_1 d_2}$ while $\sigma(\pi)^2 = 1$, leading to a suboptimal guarantee.
Another point is that their bound scales inversely with theglobally worst-case curvature $c_\mu,$, while our bound scales inversely with the instance-specific curvature $\kappa_\star$.
We note that depending on the arm-set geometry, it may be that $\kappa_\star \gg c_\mu$~\citep{abeille2021logistic}.

\newpage
\section{COMPUTATIONAL COMPLEXITY OF \texttt{GL-LowPopArt}}
\label{app:complexity}

Note that there are three computational bottlenecks of \texttt{GL-LowPopArt}: Computing the nuclear penalized estimator $\bm\Theta_0$ in Stage~I (line 4, \Cref{sec:nuclear}), solving the optimal GL-design in Stage~II (line 5, \Cref{sec:gl-lowpopart}), and inverting the vectorized Hessian (line 8, \Cref{sec:adaptive}).
The last component incurs $\gO((d_1 d_2)^3)$ time and $\gO((d_1 d_2)^2)$ space complexities via standard Gauss-Jordan elimination or LU decomposition~\citep{golub2013matrix}; in high dimensions, one could use more efficient algorithms such as the Strassen algorithm~\citep{strassen1969gaussian} that incurs $\gO((d_1 d_2)^{\log_2 7})$ time.
We now provide theoretical discussions of the computational complexity of the other two components.

\paragraph{Computational Complexity of Nuclear Penalized MLE.}
Recall the definition of nuclear penalized MLE:
\begin{equation}
    \widehat{\bm\Theta}_0 := \argmin_{\bm\Theta \in \Omega} \left\{ \phi(\bm\Theta) \triangleq \gL_{N_1}(\bm\Theta) + \lambda_{N_1} \bignormnuc{\bm\Theta} \right\},
\end{equation}
where $\lambda_{N_1} = \sqrt{\frac{C}{N_1} \log\frac{d_1 + d_2}{\delta}}$ as in \Cref{prop:lambda}.

From optimization viewpoint, although $\phi(\cdot)$ is not globally strongly convex, it satisfies Restricted Strong Convexity (RSC); see \Cref{lem:RSC}.
Thus, if the optimization iterates somehow stay within the restricted set of directions where the RSC condition holds, the algorithm achieves a geometric contraction rate, requiring only $\gO(\log \epsilon_1^{-1})$ iterations to yield an approximate minimizer $\widetilde{\bm\Theta}_0$ that satisfies $\phi(\widetilde{\bm\Theta}_0) - \phi(\widehat{\bm\Theta}_0) \leq \epsilon_1$, for $\epsilon_1 \gtrsim \frac{1}{\lambda_{\min}(\mH(\pi; \bm\Theta_\star))} \sqrt{\frac{C r}{N_1} \log\frac{d_1 + d_2}{\delta}}$.\footnote{This lower bound on $\epsilon_1$ is intuitively such that the geometric contraction holds until the optimization error hits the inherent statistical noise floor of the problem.}
\citet[Theorem 2]{agarwal2012fast} proved that the composite gradient descent (CGD)~\citep{nesterov2007composite} satisfies such property, and attains the geometric contraction rate.
Note that per time and space complexity of CGD is $\gO(d_1 d_2 (N_1 + \min(d_1, d_2)))$ and $\gO(d_1 d_2)$, respectively, i.e., the total time and space complexity for CGD to obtain $\epsilon_1$-optimal nuclear penalized MLE is $\gO\left( d_1 d_2 (N_1 + \min(d_1, d_2)) \log \epsilon_1^{-1} \right)$ and $\gO(d_1 d_2)$, respectively.

We now discuss the required level of $\epsilon_1$ to ensure that our error rates remain in-tact.
First, by definition, we have that
\begin{equation}
    \gL_{N_1}(\widetilde{\bm\Theta}_0) + \lambda_{N_1} \bignormnuc{\widetilde{\bm\Theta}_0} \leq \gL_{N_1}(\widehat{\bm\Theta}_0) + \lambda_{N_1} \bignormnuc{\widehat{\bm\Theta}_0} + \epsilon_1
    \leq \gL_{N_1}(\bm\Theta_\star) + \lambda_{N_1} \bignormnuc{\bm\Theta_\star} + \epsilon_1.
\end{equation}
Let $\bm\Delta = \widetilde{\bm\Theta}_0 - \bm\Theta_\star$.
Rearranging the terms, subtracting $\langle \nabla \gL_{N_1}(\bm\Theta_\star), \bm\Delta \rangle$ from both sides, and applying the RSC condition yields:
\begin{equation}
    \underbrace{\gL_{N_1}(\widetilde{\bm\Theta}_0) - \gL_{N_1}(\bm\Theta_\star) - \langle \nabla \gL_{N_1}(\bm\Theta_\star), \bm\Delta \rangle}_{\gtrsim \lambda_{\min}(\mH(\pi; \bm\Theta_\star)) \bignorm{\bm\Delta}_F^2} \leq -\langle \nabla \gL_{N_1}(\bm\Theta_\star), \bm\Delta \rangle + \lambda_{N_1} (\bignormnuc{\bm\Theta_\star} - \bignormnuc{\bm\Theta_\star + \bm\Delta}) + \epsilon_1,
\end{equation}
where we absorb the tolerance term $L_\mu \sqrt{\frac{1}{N_1} \log\frac{d_1 + d_2}{\delta}} \bignormnuc{\bm\Delta}^2$ into the curvature term by assuming that $N_1$ is sufficiently large.

For the right-hand side, we condition on the event that $\lambda_{N_1} \geq 2 \bignormop{\nabla \gL_{N_1}(\bm\Theta_\star)}$, which holds with high probability by \Cref{prop:lambda}.
Applying the decomposability of the nuclear norm with respect to the rank-$r$ subspace $M$ of $\bm\Theta_\star$~\citep[Chapter 10.2]{wainwright2019high}, we obtain:
\begin{equation}
    \lambda_{\min}(\mH(\pi; \bm\Theta_\star)) \bignorm{\bm\Delta}_F^2 \lesssim \lambda_{N_1} \bignormnuc{\bm\Delta} + \lambda_{N_1} (\bignormnuc{\bm\Delta_M} - \bignormnuc{\bm\Delta_{M^\perp}}) + \epsilon_1
    \lesssim \lambda_{N_1} \bignormnuc{\bm\Delta_M} + \epsilon_1.
\end{equation}
Since $\text{rank}(\bm\Delta_M) \leq 2r$, we have $\bignormnuc{\bm\Delta_M} \leq \sqrt{2r} \bignorm{\bm\Delta_M}_F \leq \sqrt{2r} \bignorm{\bm\Delta}_F$. This gives a simple quadratic inequality:
\begin{equation}
    \lambda_{\min}(\mH(\pi; \bm\Theta_\star)) \bignorm{\bm\Delta}_F^2 - \lambda_{N_1} \sqrt{r} \bignorm{\bm\Delta}_F - \epsilon_1 \lesssim 0.
\end{equation}
Solving for $\bm\Delta$ and plugging in $\lambda_{N_1}$ gives
\begin{equation}
    \bignorm{\bm\Delta}_F = \bignorm{\widetilde{\bm\Theta}_0 - \bm\Theta_\star}_F \lesssim \frac{1}{\lambda_{\min}(\mH(\pi; \bm\Theta_\star))} \sqrt{\frac{C r}{N_1} \log\frac{d_1 + d_2}{\delta}} + \underbrace{\sqrt{\frac{\epsilon_1}{\lambda_{\min}(\mH(\pi; \bm\Theta_\star))}}}_{(*)}.
\end{equation}
To ensure the optimization error $(*)$ does not dominate, it suffices to set
\begin{equation}
    \epsilon_1 \lesssim \frac{C r}{\lambda_{\min}(\mH(\pi; \bm\Theta_\star)) N_1} \log\frac{d_1 + d_2}{\delta},
\end{equation}
which results in the following time and space complexities:
\begin{equation}
    \gO\left( d_1 d_2 (N_1 + \min(d_1, d_2)) \log\left( \frac{\lambda_{\min}(\mH(\pi; \bm\Theta_\star)) N_1}{r} \right) \right) \text{ time and } \gO\left( d_1 d_2 \right) \text{ space}.
\end{equation}

\paragraph{Computational Complexity of Optimal GL-Design.}
First, recall the optimal GL-design problem:
\begin{equation}
    \min_{\pi \in \gP(\gX)} \left\{ \GL(\pi; \bm\Theta_\star) = \max\{ H^{(\row)}(\pi; \bm\Theta_\star), H^{(\col)}(\pi; \bm\Theta_\star) \} \right\},
\end{equation}
where
\begin{align}
    H^{(\row)}(\pi; \bm\Theta_\star) &:= \lambda_{\max}\left( \sum_{{\color{blue}m}=1}^{d_2} \mD_{\color{blue}m}^{(\row)}(\pi; \bm\Theta_\star) \right), \quad \mD_{\color{blue}m}^{(\row)}(\pi_2) :=
     [(\mH(\pi; \bm\Theta_\star)^{-1})_{jk}]_{j,k \in {\color{purple}\gI^{(\row)}}},\\
    H^{(\col)}(\pi; \bm\Theta_\star) &:= \lambda_{\max}\left( \sum_{{\color{blue}m}=1}^{d_1} \mD_{\color{blue}m}^{(\col)}(\pi; \bm\Theta_\star) \right), \quad \mD_{\color{blue}m}^{(\col)}(\pi; \bm\Theta_\star) := (\mH(\pi; \bm\Theta_\star)^{-1})_{{\color{teal}\gI^{(\col)}}, {\color{teal}\gI^{(\col)}}},
\end{align}
and the index sets defined as ${\color{purple}\gI^{(\row)}} := \{d_1 (l-1) + m: l\in [d_2]\}$ and ${\color{teal}\gI^{(\col)}} := [d_1 (m-1)+1:d_1 m]$.

We elaborate on how to equivalently reformulate the GL-design as a Semidefinite Program (SDP).
To do so, we first introduce an auxiliary positive semidefinite matrix variable $\mZ \in \mathbb{S}_+^{d_1 d_2}$ to upper bound the inverse of principal submatrices of the inverse Hessian, and a scalar $w \in \mathbb{R}$ to bound the maximum eigenvalue.

By applying the Schur complement~\citep{lmi-book}, the condition $\mZ \succeq \mH(\pi; \bm\Theta_\star)^{-1}$ is strictly equivalent to the Linear Matrix Inequality (LMI):
\begin{equation}
    \underbrace{\begin{pmatrix} \mH(\pi; \bm\Theta_\star) & \mI_{d_1 d_2} \\ \mI_{d_1 d_2} & \mZ \end{pmatrix}}_{(**)} \succeq 0.
\end{equation}
Defining the principal submatrix extraction operators $\gP^{(\row)}_m(\mZ) := \mZ_{\gI^{(\row)}_m, \gI^{(\row)}_m}$ and $\gP^{(\col)}_m(\mZ) := \mZ_{\gI^{(\col)}_m, \gI^{(\col)}_m}$ (which are linear operators), the exact SDP reformulation of the GL-design is given by:
\begin{align}
    \min_{\pi \in \mathbb{R}^{|\gX|}, \mZ \in \mathbb{S}^{d_1 d_2}, w \in \mathbb{R}} \quad & w \nonumber \\
    \text{subject to} \quad & \sum_{x \in \gX} \pi(x) = 1, \quad \pi(x) \geq 0 \quad \forall x \in \gX, \nonumber \\
    & \begin{pmatrix} \sum_{x \in \gX} \pi(x) \mH(x; \bm\Theta_\star) & \mI_{d_1 d_2} \\ \mI_{d_1 d_2} & \mZ \end{pmatrix} \succeq 0, \label{eq:sdp_formulation} \\
    & \sum_{m=1}^{d_2} \gP^{(\row)}_m(\mZ) \preceq w \mI_{d_1}, \nonumber \\
    & \sum_{m=1}^{d_1} \gP^{(\col)}_m(\mZ) \preceq w \mI_{d_2}. \nonumber
\end{align}

By employing standard interior-point methods~\citep{nesterov1994interior}, the time complexity to solve this SDP to an $\epsilon_2$-optimal solution scales polynomially with the number of variables and the size of the LMI, yielding $\gO\left(\text{poly}(|\gX|, d_1, d_2) \log \epsilon_2^{-1}\right)$. The space complexity is dominated by the augmented matrix $(**)$, requiring $\gO((d_1 d_2)^2)$ memory.

\newpage
\section{E-OPTIMAL DESIGN FOR STAGE~I}
\label{app:stage-i}

Here, we briefly elaborate on the use of E-optimal design for Stage~I.
As discussed in \Cref{rmk:e-optimal}, from the error rate of Stage~I (\cref{thm:estimation-0}), the natural optimal design is $\argmax_{\pi \in \gP(\gX)} \lambda_{\min}(\mH(\pi; \bm\Theta_\star)).$
However, as we do not have any prior knowledge about $\bm\Theta_\star$, the best one could do is to consider a na\"{i}ve lower bound of $\lambda_{\min}(\mH(\pi; \bm\Theta_\star)) \geq \kappa_\star \lambda_{\min}(\mV(\pi))$.
This motivates the following optimal design:
\begin{equation}
\label{eqn:E-optimal}
    \pi_E \gets \argmax_{\pi \in \gP(\gX)} \lambda_{\min}(\mV(\pi)),
\end{equation}
known as the \textit{E-optimal design}~\citep{pukelsheim2005design}, previously considered in sparse linear bandits~\citep{hao2020sparse} and bandit phase retrieval~\citep{lattimore2021phase}.
This is a convex optimization problem, and can be solved efficiently.
Below, we collect some results regarding the computation aspects of $\pi_E$.

Unlike G-optimal (or equivalently, D-optimal) design where it is guaranteed that one can obtain an optimal design with support size at most $\frac{d_1 d_2 (d_1 d_2 + 1)}{2} + 1$~\citep{kiefer-wolfowitz,todd2016ellipsoid}, there is no analogous guarantee on the support size of $\pi_E$.
One could however obtain an \emph{approximate} E-optimal design with guaranteed bounded support size.

First, if $\gX$ is discrete, then one can use the polynomial-time algorithm of \citet{allen-zhu2021optimal} to obtain $\widehat{\pi}_E$ satisfying $|\supp(\widehat{\pi}_E)| \lesssim d_1 d_2$ and $\lambda_{\min}(\mV(\widehat{\pi}_E)) \geq \frac{1}{2} \lambda_{\min}(\mV(\pi_E))$.

Now suppose that $\gX$ is continuous.
For this, we employ two-stage approach: first solve the E-optimal design, then sparsify its support via $\epsilon$-\textit{approximate Carath\'{e}odory solver}~\citep{barman2015caratheodory,mirrokni2017caratheodory,combettes2023caratheodory},\footnote{Recently, \citet{combettes2023caratheodory} showed that the Frank-Wolfe algorithm~\citep{frank1956optim} is effective in solving the approximate Carath\'{e}odory problem, making it as efficient as solving the G-optimal design with bounded support~\citep{todd2016ellipsoid}.}\,\footnote{The approximate Carath\'{e}odory theorem~\citep[Theorem 2]{barman2015caratheodory} states that $|\supp(\widehat{\pi}_E)| \lesssim \epsilon^{-2} \mathrm{diam}(\vec(\gX))^2$ where $\vec(\gX) := \{ \vec(\mX) \vec(\mX)^\top : \mX \in \gX \}$, and we have that $\mathrm{diam}(\vec(\gX))^2 \leq 4 (d_1 \wedge d_2)^2$ when $\gX \subseteq \gB^{d_1 \times d_2}_{\mathrm{op}}(1)$.}
which outputs a $\widehat{\pi}_E$ such that $\bignorm{\mV(\pi_E) - \mV(\widehat{\pi}_E)}_F \leq \epsilon$ and $|\supp(\widehat{\pi}_E)| \lesssim \frac{(d_1 \wedge d_2)^2}{\epsilon^2}$.
The approximation error in the objective $\lambda_{\min}(\mV(\pi))$ is controlled via the Hoffman-Wielandt inequality for eigenvalue perturbations~\citep{hoffman-wielandt}:
\begin{equation}
    |\lambda_{\min}(\mV(\pi_E)) - \lambda_{\min}(\mV(\widehat{\pi}_E))| \leq \bignorm{\mV(\pi_E) - \mV(\widehat{\pi}_E)}_F
    \leq \epsilon.
\end{equation}
The sensible choice for $\epsilon$ is $\epsilon = \frac{\lambda_{\min}(\mV(\pi_E))}{2}$, but in practice, it is often difficult to know $\lambda_{\min}(\mV(\pi_E))$ beforehand.
But for special arm-sets, this is possible. For instance, when $\gX = \gB^{d_1 \times d_2}_{\mathrm{op}}(1)$, we have that $\lambda_{\min}(\mV(\pi_E)) = \frac{1}{d_1 \vee d_2}$~\citep[Appendix D.2]{jang2024lowrank}.

\newpage
\section{\texorpdfstring{PROOF OF THEOREM~\ref{thm:lower-bound} -- LOCAL MINIMAX LOWER BOUND}{PROOF OF THEOREM 4.1 -- LOCAL MINIMAX LOWER BOUND}}
\label{app:lower-bound}
WLOG assume that $d_1 = \max(d_1, d_2)$.
For given $\bm\Theta_\star$, let $\mU \mD \mV^\top$ be its SVD.

Inspired by \citet[Theorem 5]{rohde2011estimation} and \citet[Theorem 2]{abeille2021logistic}, we consider the following set of $d_1 \times d_2$ matrices:
\begin{equation}
    \Theta_{r, \varepsilon, \beta} := \left\{ \left( 1 - \varepsilon \right) \bm\Theta_\star + \varepsilon \mU' \mV^\top \in \sR^{d_1 \times d_2} : \mU' \in \left\{ 0, \beta \right\}^{d_1 \times r} \right\},
\end{equation}
where $\varepsilon \in (0, 1)$ and $\beta > 0$ will be specified later.
By construction, we have that for any $\bm\Theta \in \Theta_{r,\varepsilon,\beta}$, $\rank(\bm\Theta) \leq r$ and
\begin{align}
    \bignormnuc{\bm\Theta} &\leq (1 - \varepsilon) \bignormnuc{\bm\Theta_\star} + \varepsilon \bignormnuc{\mU' \mV^\top} \\
    &= (1 - \varepsilon) S_* + \varepsilon \bignormnuc{\mU'} \tag{unitary invariance of $\bignormnuc{\cdot}$} \\
    &\leq (1 - \varepsilon) S_* + \varepsilon \sqrt{r} \bignorm{\mU'}_F \tag{Cauchy-Schwartz inequality on the singular values of $\mU'$} \\
    &\leq (1 - \varepsilon) S_* + \varepsilon \beta r \sqrt{d_1} \tag{by construction}.
\end{align}
Thus, it can be verified that {\color{violet}$\beta \leq \frac{S_*}{r \sqrt{d_1}}$} implies $\bignormnuc{\bm\Theta} \leq S_*$, i.e., $\Theta_{r,\varepsilon,\beta} \subset \gN(\bm\Theta_\star; \varepsilon, r, S_*)$.

By construction, $\bignorm{\bm\Theta_1 - \bm\Theta_2}_F^2$ is closely related to the Hamming distance of the $\vec(\mU')$'s, which are basically binary sequences.
With this, we recall the Gilbert-Varshamov bound:
\begin{lemma}[Gilbert–Varshamov bound; Lemma 2.9 of \citet{tsybakov09}; Theorem 1 of \citet{gilbert1952comparison}; \citet{varshamov1964estimate}]
\label{lem:gilbert}
    Let $m \geq 8$ and $\Omega := \{0, 1\}^m$.
    Then there exists $\{ \omega^{(0)}, \omega^{(1)}, \cdots, \omega^{(M)} \} \subset \Omega$ with $M \geq 2^{m/8}$ such that $\omega^{(0)} = (0, \cdots, 0)$ and
    \begin{equation}
        d_H(\omega^{(j)}, \omega^{(k)}) := \sum_{\ell=1}^m \indicator[(\omega^{(j)})_\ell \neq (\omega^{(k)})_\ell] \geq \frac{m}{8}, \quad \forall 0 \leq j < k \leq M.
    \end{equation}
\end{lemma}
Thus, we can find a $\Theta_{r,\varepsilon,\beta}^0 \subset \Theta_{r,\varepsilon,\beta}$ such that $|\Theta_{r,\varepsilon,\beta}^0| \geq 2^{\frac{r d_1}{8}}$, and for any $\bm\Theta_i = (1 - \varepsilon) \bm\Theta_\star + \varepsilon \mU_i' \mD \mV^\top \in \Theta_{r,\varepsilon,\beta}^0$ with $i \in \{1, 2\}$ and $\mU_1 \neq \mU_2$,
\begin{equation}
    \bignorm{\bm\Theta_1 - \bm\Theta_2}_F^2 = \varepsilon^2 \bignorm{(\mU_1' - \mU_2') \mV^\top}_F^2
    = \varepsilon^2 \bignorm{(\mU_1' - \mU_2')}_F^2
    \geq \varepsilon^2 \frac{\beta^2 r d_1}{8},
\end{equation}
where we denote $\sigma_{\min} = \sigma_{\min}(\bm\Theta_\star)$ to be the minimum non-zero singular value of $\bm\Theta_\star$.

Furthermore, we have that for any $\bm\Theta = (1 - \varepsilon) \bm\Theta_\star + \varepsilon \mU' \mV^\top \in \Theta_{r,\varepsilon,\beta}^0$,
\begin{align}
    \bignorm{\bm\Theta_\star - \left( (1 - \varepsilon) \bm\Theta_\star + \varepsilon \mU' \mV^\top \right)}_F^2 &= \varepsilon^2 \bignorm{\bm\Theta_\star - \mU' \mV^\top}_F^2 \\
    &\geq \varepsilon^2 \left( \bignorm{\bm\Theta_\star}_F^2 - \bignorm{\mU'}_F^2 \right) \tag{triangle inequality and unitary invariance of $\bignorm{\cdot}_F$} \\
    &\geq \varepsilon^2 \left( \bignorm{\bm\Theta_\star}_F^2 - \beta^2 r d_1 \right) \tag{by construction} \\
    &\geq \varepsilon^2 \frac{\beta^2 r d_1}{8},
\end{align}
which in turn holds when {\color{violet}$\bignorm{\bm\Theta_\star}_F^2 \geq \frac{9 \beta^2 r d_1}{8}$}.
We will see that this indeed holds with our $\beta$ specified later.

For $\bm\Theta \in \sR^{d_1 \times d_2}$, let $\sP_{\bm\Theta}$ be the probability distribution of the observations $\{(\mX_t, y_t)\}_{t \in [N]}$, with $y_t \sim p(\cdot | \mX_t; \bm\Theta)$.

We now compute the KL between $\sP_{(1 - \varepsilon) \bm\Theta_\star + \varepsilon \bm\Theta'}$ and $\sP_{\bm\Theta_\star}$ for any $\bm\Theta' = \mU' \mV^\top \in \Theta_{r,\varepsilon,\beta}$ by connecting it with the Bregman divergence:
\begin{definition}
    For a $m : \sR \rightarrow \sR$, the {\bf Bregman divergence} $D_m(\cdot, \cdot)$ is defined as follows:
    \begin{equation}
        D_m(z_1, z_2) := m(z_1) - m(z_2) - m'(z_2) (z_1 - z_2).
    \end{equation}
\end{definition}
We recall the following well-known lemma from information geometry, which simplifies the computation of KL between two GLMs by implicitly making use of their dually flat structure~\citep{infogeom,nielsen2020infogeom,brekelmans2020bregman}:
\begin{lemma}
\label{lem:kl-bregman}
    Consider two GLMs $p_1 \triangleq p(\cdot | \mX; \bm\Theta_1)$ and $p_2 \triangleq p(\cdot | \mX; \bm\Theta_2)$ with the same log-partition function $m$.
    Then, we have that $\KL(p_2, p_1 | \mX) \triangleq \KL(p(\cdot | X, \bm\Theta_2), ) = \frac{1}{g(\tau)} D_m(\langle \mX, \bm\Theta_1 \rangle, \langle \mX, \bm\Theta_2 \rangle)$.
\end{lemma}
We then have that
\begin{align}
    \KL(\sP_{(1 - \varepsilon) \bm\Theta_\star + \varepsilon \bm\Theta'}, \sP_{\bm\Theta_\star} | \mX) &= \frac{1}{g(\tau)} D_m(\langle \mX, \bm\Theta_\star \rangle, (1 - \varepsilon) \langle \mX, \bm\Theta_\star \rangle + \varepsilon \langle \mX, \bm\Theta' \rangle) \\
    &= \frac{1}{g(\tau)} \varepsilon^2 \langle \mX, \bm\Theta_\star - \bm\Theta' \rangle^2 \int_0^1 v \dmu(\langle \mX, \bm\Theta_\star \rangle + \varepsilon \langle \mX, \bm\Theta' - \bm\Theta_\star \rangle v) dv. \tag{Taylor expansion with integral remainder}
\end{align}
We recall a useful self-concordance control lemma from \citet{abeille2021logistic,faury2020logistic}:
\begin{lemma}[A Modification of Lemma 9 of \citet{abeille2021logistic}]
\label{lem:abeille-9}
    Let $\mu : \sR \rightarrow \sR$ be a strictly increasing function satisfying $|\ddmu| \leq R_s \dmu$ for some $R_s \geq 0$.
    Then, for any $z_1, z_2 \in \sR$ and $\varepsilon > 0$, $\dmu(z_1 + \varepsilon z_2) \leq \dmu(z_1) \exp(R_s \varepsilon |z_2|).$
\end{lemma}
With this, we have that
\begin{align}
    \KL(\sP_{(1 - \varepsilon) \bm\Theta_\star + \varepsilon \bm\Theta'}, \sP_{\bm\Theta_\star} | \mX) &\leq \frac{1}{g(\tau)} \varepsilon^2 \dmu(\langle \mX, \bm\Theta_\star \rangle) \langle \mX, \bm\Theta_\star - \bm\Theta' \rangle^2 \int_0^1 v \exp(R_s \varepsilon |\langle \mX, \bm\Theta' - \bm\Theta_\star \rangle| v) dv \\
    &\leq \frac{1}{2 g(\tau)} \varepsilon^2 \dmu(\langle \mX, \bm\Theta_\star \rangle) \langle \mX, \bm\Theta_\star - \bm\Theta' \rangle^2 \exp(R_s \varepsilon |\langle \mX, \bm\Theta' - \bm\Theta_\star \rangle|) \\
    &\overset{(*)}{\leq} \frac{1}{2 g(\tau)} \varepsilon^2 \dmu(\langle \mX, \bm\Theta_\star \rangle) \langle \mX, \bm\Theta_\star - \bm\Theta' \rangle^2 \exp \left( R_s \varepsilon (1 + \beta \sqrt{d_1 r}) S_* \right) \\
    &\leq \frac{e}{2 g(\tau)} \varepsilon^2 \dmu(\langle \mX, \bm\Theta_\star \rangle) \langle \mX, \bm\Theta_\star - \bm\Theta' \rangle^2, 
\end{align}
\textit{given} that {\color{violet}$R_s \varepsilon (1 + \beta \sqrt{d_1 r}) S_* \leq 1$}.
Note that $(*)$ holds regardless of whether we assume $\gX \subseteq \gB^{d_1 \times d_2}_F(1)$ (which is what we assume in the statement) or $\gX \subseteq \gB^{d_1 \times d_2}_{\mathrm{op}}(1)$ (which is implied from the first case).
To see this, if the first case holds, then
\begin{equation}
    \langle \mX, \bm\Theta_\star - \bm\Theta' \rangle \leq \bignorm{\mX}_F \bignorm{\bm\Theta - \bm\Theta_\star}_F \leq \bignormnuc{\bm\Theta - \bm\Theta_\star} \leq (1 + \beta \sqrt{d_1 r}) S_*,
\end{equation}
and if the second case holds,
\begin{equation}
    \langle \mX, \bm\Theta_\star - \bm\Theta' \rangle \leq \bignormop{\mX} \bignormnuc{\bm\Theta - \bm\Theta_\star} \leq (1 + \beta \sqrt{d_1 r}) S_*.
\end{equation}

\begin{remark}
    \citet[Lemma 4]{lee2024logistic} has utilized a similar argument (Taylor integral remainder with self-concordance) to provide a lower bound on the KL divergence during the online learning regret analysis.
    However, they restricted their attention to the Bernoulli distribution.
\end{remark}

Thus, recalling that $\bm\pi_N = (\pi_t)_{t \in [N]}$ and using the chain rule for KL~\citep[Exercise 14.12]{banditalgorithms},
\begin{align}
    \KL(\sP_{(1 - \varepsilon) \bm\Theta_\star + \varepsilon \bm\Theta'}, \sP_{\bm\Theta_\star}) &= \sum_{t=1}^N \E_{\mX_t \sim \pi_t}[\KL(\sP_{(1 - \varepsilon) \bm\Theta_\star + \varepsilon \bm\Theta'}, \sP_{\bm\Theta_\star} | \mX_t)] \\
    &\leq \frac{e}{2 g(\tau)} \varepsilon^2 \vec(\bm\Theta_\star - \bm\Theta')^\top \left( \sum_{t=1}^N \mH(\pi_t; \bm\Theta_\star) \right) \vec(\bm\Theta_\star - \bm\Theta') \\
    &\leq \frac{e N}{2 g(\tau)} \varepsilon^2 \lambda_{\max}(\mH(\bm\pi_N; \bm\Theta_\star)) \bignorm{\bm\Theta_\star - \bm\Theta'}_F^2 \tag{$\mH(\bm\pi_N; \bm\Theta_\star) := \frac{1}{N} \sum_{t=1}^N \mH(\pi_t; \bm\Theta_\star)$} \\
    &\leq \frac{e N}{2 g(\tau)} \varepsilon^2 \lambda_{\max}(\mH(\bm\pi_N; \bm\Theta_\star)) (1 + \beta \sqrt{d_1 r})^2 S_*^2.
\end{align}

Then we have that
\begin{align}
    \frac{1}{|\Theta_{r,\varepsilon}^0|} \sum_{\bm\Theta' \in \Theta_{r,\varepsilon}^0} \KL(\sP_{\bm\Theta'}, \sP_{\bm\Theta_\star}) &\leq \frac{e \varepsilon^2 N \lambda_{\max}(\mH(\bm\pi_N; \bm\Theta_\star)) (1 + \beta \sqrt{d_1 r})^2 S_*^2}{2 g(\tau)} \\
    &= \frac{4 e N \varepsilon^2 \lambda_{\max}(\mH(\bm\pi_N; \bm\Theta_\star)) (1 + \beta \sqrt{d_1 r})^2 S_*^2}{g(\tau) r d_1} \frac{r d_1}{8}.
\end{align}
As $\log|\Theta_{r,\varepsilon,\beta}^0| \geq \log(2^{\frac{r d_1}{8}}) = \frac{r d_1}{8} \log 2,$ 
\begin{equation}
    \frac{1}{|\Theta_{r,\varepsilon,\beta}^0|} \sum_{\bm\Theta' \in \Theta_{r,\varepsilon,\beta}^0} \KL(\sP_{\bm\Theta'}, \sP_{\bm\Theta_\star}) \leq \frac{1}{16} \log|\Theta_{r,\varepsilon}^0|
\end{equation}
holds with $\varepsilon^2 \leq \frac{r d_1 g(\tau) \alpha \log 2}{2^6 e N \lambda_{\max}(\mH(\bm\pi_N; \bm\Theta_\star)) (1 + \beta\sqrt{d_1 r})^2 S_*^2}$ where $\alpha = \frac{1}{16}$.

We choose
\begin{equation}
    \beta^2 = \frac{\gamma}{r d_1} \Rightarrow
    \varepsilon^2 = \frac{\alpha \log 2}{2^6 e (1 + \sqrt{\gamma})^2} \frac{r d_1 g(\tau)}{N \lambda_{\max}(\mH(\bm\pi_N; \bm\Theta_\star)) S_*^2}.
\end{equation}
We now check the {\color{violet} requirements}:
\begin{align}
    \beta \leq \frac{S_*}{r\sqrt{d_1}}
    &\Longleftrightarrow \gamma \leq \frac{S_*^2}{r} \\
    \bignorm{\bm\Theta_\star}_F^2 \geq \frac{9 \beta^2 r d_1}{8} &\Longleftrightarrow \gamma \leq \frac{8}{9} \bignorm{\bm\Theta_\star}_F^2 \\
    R_s \varepsilon (1 + \beta \sqrt{d_1 r}) S_* \leq 1 &\Longleftrightarrow N \geq \frac{R_s^2}{2^{10}} \frac{\log 2}{e} \frac{r d_1 g(\tau)}{\lambda_{\max}(\mH(\bm\pi_N; \bm\Theta_\star))}.
\end{align}

The proof concludes by invoking \citet[Theorem 2.5]{tsybakov09} with $\alpha = \frac{1}{16}$,\footnote{No efforts were made to optimize the constants.} which we recall here for completeness:
\begin{lemma}[Theorem 2.5 of \citet{tsybakov09}]
\label{lem:tsybakov-2.5}
    Let $\Theta$ be a subset of a metric space with metric $d(\cdot, \cdot)$, and let $\bm\theta \mapsto \sP_{\bm\theta}$ be the probability measure parametrized by $\bm\theta$.
    Suppose that there exists $\{ \bm\theta_0, \bm\theta_1, \cdots, \bm\theta_M \} \subset \Theta$ for some $M \geq 2$ such that
    \begin{itemize}
        \item[(i)] $d(\bm\theta_j, \bm\theta_k) \geq 2b > 0, \quad \forall 0 \leq j < k \leq M$,
        \item[(ii)] $\sP_{\bm\theta_j} \ll \sP_{\bm\theta_0}, \quad \forall j = 1, 2, \cdots, M$, and
        \item[(iii)] there exists a $\alpha \in (0, 1/8)$ such that $\frac{1}{M} \sum_{j=1}^M \KL(\sP_{\bm\theta_j}, \sP_{\bm\theta_0}) \leq \alpha \log M.$
    \end{itemize}
    Then, we have the following high-probability minimax lower bound:
    \begin{equation}
        \inf_{\widehat{\bm\theta}} \sup_{\bm\theta_\star \in \Theta} \sP_{\bm\theta_\star}(d(\widehat{\bm\theta}, \bm\theta_\star) \geq b) \geq \frac{\sqrt{M}}{1 + \sqrt{M}} \left( 1 - 2\alpha - \sqrt{\frac{2\alpha}{\log M}} \right) > 0.
    \end{equation}
\end{lemma}

\qed

We now provide the proofs of the missing lemmas:
\begin{proof}[Proof of \cref{lem:kl-bregman}]
    This follows from brute-force computation:
    \begin{align}
        \KL(p_2, p_1) &= \E_{y \sim p_2}\left[ \log\frac{p_2(y)}{p_1(y)} \right] \\
        &= \frac{1}{g(\tau)} \E_{y \sim p_2}\left[ y \langle \mX, \bm\Theta_2 - \bm\Theta_1 \rangle + m(\langle \mX, \bm\Theta_1 \rangle) - m(\langle \mX, \bm\Theta_2 \rangle)  \right] \tag{recall the probability density of GLMs} \\
        &= \frac{m(\langle \mX, \bm\Theta_1 \rangle) - m(\langle \mX, \bm\Theta_2 \rangle) - m'(\langle \mX, \bm\Theta_2 \rangle) \langle \mX, \bm\Theta_1 - \bm\Theta_2 \rangle}{g(\tau)} \tag{$\E[y] = m'(\langle \mX, \bm\Theta_2 \rangle)$} \\
        &= \frac{1}{g(\tau)} D_m(\langle \mX, \bm\Theta_1 \rangle, \langle \mX, \bm\Theta_2 \rangle).
    \end{align}
\end{proof}

\begin{proof}[Proof of \cref{lem:abeille-9}]
    We provide the slightly modified proof of \citet[Lemma 9]{abeille2021logistic} for completeness.

    Starting from the self-concordance, we have that for any $z_1, z_2 \in \sR$
    \begin{equation}
        -R_s \leq \frac{\ddmu(z)}{\dmu(z)} \leq R_s, \quad \forall z \in \sR
        \Longrightarrow - R_s \varepsilon |z_2| \leq \underbrace{\int_{(z_1 + \varepsilon z_2) \wedge z_1}^{\dmu(z_1 + \varepsilon z_2) \vee z_1} \frac{\ddmu(z)}{\dmu(z)} dz}_{= \log \frac{\dmu((z_1 + \varepsilon z_2) \vee z_1)}{\dmu((z_1 + \varepsilon z_2) \wedge z_1)}} \leq R_s \varepsilon |z_2|.
    \end{equation}
    If $z_2 \geq 0$, then we have that from the upper bound,
    \begin{equation}
        \dmu(z_1 + \varepsilon z_2) \leq \dmu(z_1) \exp(R_s \varepsilon z_2)
        = \dmu(z_1) \exp(R_s \varepsilon |z_2|).
    \end{equation}
    If $z_2 < 0$, then we have that from the lower bound,
    \begin{equation}
        \dmu(z_1 + \varepsilon z_2) \exp(R_s \varepsilon z_2) \leq \dmu(z_1)
        \Longrightarrow
        \dmu(z_1 + \varepsilon z_2) \leq \dmu(z_1) \exp(- R_s \varepsilon z_2)
        = \dmu(z_1) \exp(R_s \varepsilon |z_2|).
    \end{equation}
\end{proof}

\newpage
\section{ADDITIONAL DETAILS FOR BILINEAR DUELING BANDITS: SETTING}
\label{app:bilinear}

\subsection{Motivation}
\label{app:motivation-bilinear}
Transitivity — the property that if \( i \succ j \) and \( j \succ k \), then \( i \succ k \) — is one of the key assumptions that distinguish the dueling bandit setting~\citep{yue2009dueling,yue2012dueling,sui2018dueling,bengs2021survey}.
Within this stochastic transitivity framework, the most commonly considered model is the Bradley-Terry-Luce (BTL) model~\citep{bradley-terry}: each arm $k$ has an unknown utility(reward) $r_k \in \sR$ such that for each $(i, j) \in [K] \times [K]$, $p_{i,j} := \sP(i \succ j) = \mu(r_i - r_j)$ with $\mu(z) := (1 + e^{-z})^{-1}.$
When $K$ is large, without any additional structural assumption, the statistical guarantees (e.g., regret in dueling bandits) often increase polynomially in $K$.
One very natural way of bypassing this issue is to impose a linear structure on the utility, resulting in the so-called linear BTL model: each arm $k$ is endowed with a known feature vector $\bm\phi_k \in \sR^d$ and $r_k = \langle \bm\phi_k, \bm\theta_\star \rangle$ for some unknown $\bm\theta_\star \in \sR^d$.
This model has been successfully applied in various domains, with reinforcement learning with human feedback ~\citep{DPO} being one of the most prominent applications.
Coming back to dueling bandits, with such linear structure, the regret of dueling bandits has been improved from $\poly(K)$ to $d$ or $\sqrt{d \log K}$ by exploiting the linear BTL model~\citep{saha2021contextual,bengs2022transitivity}.

However, the literature has two main gaps, both of which we intend to fill with our newly proposed setting and new analyses.

\paragraph{Linear-like Structure in Dueling Bandits with General Preferences.}
The (linear) BTL model cannot model nontransitive preferences, which hinders its applicability in various scenarios, from simple nontransitive games such as rock-paper-scissors, Blotto-style games~\citep{balduzzi2018evaluation,balduzzi2019open,bertrand2023elo}, and even human preferences~\citep{may1954intransitivity,tversky1969intransitivity,NashLearning,IPO,swamy2024minimaximalist,zhang2024bilinear}.

In most of the prior literature on dueling bandits and general preference learning (i.e., not assuming linear BTL model), the learner must either learn or adapt to the entire unstructured preference matrix $\mP \in [0, 1]^{K \times K}$.
This means that, again, the statistical guarantees are expected to depend polynomially in $K$.
Given that the linear structure has enabled the development of efficient algorithms for linear and dueling bandits with large action spaces and contextual information, the question of how to impose linear-like structure to arbitrary preference matrix $\mP$ has been a significant and longstanding open question.

There have been two notable advancements in this direction, one theoretical and one practical.
The first advancement is by \citet{wu2024dueling}, whose setting we briefly describe here.
The learner has access to a feature map $(i, j) \in [K] \times [K] \mapsto \bm\phi_{i,j} \in \sR^d$ satisfying $\bm\phi_{i, j} = - \bm\phi_{j, i}$.
The preference probability is defined as $p_{i, j} = \mu(\langle \bm\phi_{i, j}, \bm\theta_\star \rangle),$ where $\bm\theta_\star \in \sR^d$ is unknown.
With this model, the authors have improved the Borda regret's dependency on $K$ from polynomial to logarithmic.
However, it is unrealistic to know all item \textit{pair-wise} features that linearly encode the underlying preferences.
Arguably, a more realistic scenario is knowing only item-wise features, namely, $\bm\phi_k \in \sR^d$ for $k \in [K]$.

One may wonder if there is a contextual preference model that incorporates \textit{item-wise} features while being potentially nontransitive.
The second advancement, due to \citet{zhang2024bilinear}, tackles this by proposing the contextual bilinear preference model: for each item pair $(i, j) \in [K] \times [K]$, the preference model is defined as
\begin{equation}
\label{eqn:bilinear-preference2}
    p_{i, j} = \mu\left( \bm\phi_i^\top \bm\Theta_\star \bm\phi_j \right),
\end{equation}
where $\bm\Theta_\star$ is a $d \times d$ skew-symmetric matrix of low rank.
However, their paper does not provide any statistical guarantees when this is used in dueling bandits, or even regarding the estimation error of the preference model; rather, their main focus is experimentally validating this model in modeling human preferences and its implications for the downstream RLHF task.
Note that we adopt the same preference model, exept we allow for the underlying arm-set $\gA$ to be continuous.

Although not discussed further in \citet{zhang2024bilinear}, we believe this is a very natural way of incorporating some sort of linearity into general preferences, and that it deserves more attention from the dueling bandits community as well.
Indeed, such bilinear model has been used in modeling interaction of two items, with applications to drug discovery~\citep{luo2017drug}, server scheduling~\citep{kim2021bilinear}, personalized recommendation~\citep{chu2009recommendation}, link prediction~\citep{menon2011link}, relational learning~\citep{nickel2011relational}, and more.
The bandit community was introduced to this model by bilinear bandits~\citep{jang2021bilinear,jun2019bilinear}, later extended to low-rank matrix-armed bandits~\citep{lu2021generalized,kang2022generalized,jang2024lowrank}; refer to Appendix~\ref{app:related-work} for further related works on low-rank bandits.
Roughly speaking, the learner now only needs to learn $\Theta(d^2)$ parameters of $\bm\Theta_\star$ instead of $\Theta(K^2)$ parameters of $\mP$.
Furthermore, using the low-rank structure of $\bm\Theta_\star$, the learner can further improve the regret's dependency in $d$.
Although not discussed in \citet{zhang2024bilinear}, we also note that this is the rank-$d$ version of the low-rank preference model of \citet{rajkumar2016comparison}, as one can write $\mu^{-1}(\mP) = \bm\Phi^\top \bm\Theta_\star \bm\Phi$ where $\bm\Phi = [\bm\phi_1 \cdots \bm\phi_K] \in \sR^{d \times K}$ and $\mu^{-1}$ is applied entry-wise.

\paragraph{Variance-Aware Borda Regret Bound.}
The Borda regret resembles the strong regret~\citep{yue2012dueling}, and it ``respects'' the inherent problem of the difficulty of dueling bandits where two arms are chosen rather than a single arm~\citep{saha2021borda,wu2024dueling}.
Its original motivation is from search engine, in which the regret corresponds to ``the fraction of users who would prefer the best retrieval function over the selected ones.''~\citep{yue2009dueling}.

All the existing guarantees for the Borda regret either assume a fixed gap~\citep{saha2021borda} or incur a $1 / c_\mu$ dependency~\citep{wu2024dueling}, where $c_\mu$ can be thought of as the worst-case badness of linear approximation of the true preference signal.
In other words, the current Borda regret bound seems to suggest that the lower the variance (which roughly corresponds to the derivative of the inverse link function in the context of GLMs), the higher the regret. 
However, the vast literature on logistic and generalized linear bandits~\citep{abeille2021logistic,lee2024glm,lee2024logistic} suggest otherwise.
\citet{abeille2021logistic} first proved a $\widetilde{\gO}(d \sqrt{T \kappa_\star})$ regret bound for logistic bandits as well as a matching (local minimax) lower bound, the correct dependency on the variance-dependent quantity.
Thus, it should be expected that a similar variance-dependent quantity should pop up in the optimal Borda regret bounds.

\subsection{A Sufficient Condition for the Bilinear Preference to be Stochastic Transitive}
\label{app:bilinear-transitive}
A preference model is \textbf{stochastic transitive w.r.t. $\mu$}~\citep{bengs2022transitivity} if there exists a $f : [K] \rightarrow \sR$ such that $(\mP)_{ij} = \mu(f(i) - f(j))$.
Here, we prove that certain collinearity between the features $\bm\phi_i$'s in the bilinear preference model (Eqn.~\eqref{eqn:bilinear-preference2}) implies stochastic transitivity:
\begin{proposition}
\label{prop:bilinear-transitive}
    If there exists an orthonormal $\mQ \in \sR^{d \times d}$ such that $\{((\mQ^\top \bm\phi_k)_{2m - 1}, (\mQ^\top \bm\phi_k)_{2m})\}_{k \in [K]}$ is collinear in $\sR^2$ for each $m \in [r]$, then the bilinear preference model is stochastic transitive w.r.t. $\mu$.
    When $r = 1$ (i.e., $\rank(\bm\Theta_\star) = 2$), this is also a necessary condition.
\end{proposition}
\begin{proof}
    The proof is heavily inspired by \citet{jiang2011hodge}, where the authors provide a decomposition of the space of preferences via combinatorial Hodge theory; this has been also utilized in later machine learning literature on ranking with potentially nontransitive components~\citep{bertrand2023elo,balduzzi2018evaluation,balduzzi2019open}.
    
    From the combinatorial Hodge decomposition~\citep[Theorem 2]{jiang2011hodge}, a $f$ that satisfies the stochastic transitivity exists if and only if for any $(i, j, k) \in [K]^3$,
    \begin{equation}
        \bm\phi_i^\top \bm\Theta_\star \bm\phi_j + \bm\phi_j^\top \bm\Theta_\star \bm\phi_k+ \bm\phi_k^\top \bm\Theta_\star \bm\phi_i = 0.
    \end{equation}
    The quantity on the LHS is known as the \textit{combinatorial curl}~\citep{jiang2011hodge}.
    
    Let $\bm\Theta_\star = \mQ \bm\Lambda \mQ^\top$ be its canonical form (Lemma~\ref{lem:skew}), and let $\bm\varphi_i := \mQ^\top \bm\phi_i$.
    Let $\{\lambda_m\}_{m \in [r]} \subset \sR_{>0}$ be the nonzero components of $\bm\Lambda$.
    Then, the above curl-free requirement boils down to
    \begin{equation}
        \sum_{m = 1}^r \lambda_m
        \underbrace{\begin{vmatrix}
            1 & 1 & 1 \\
            (\bm\varphi_i)_{2m - 1} & (\bm\varphi_j)_{2m - 1} & (\bm\varphi_k)_{2m - 1} \\
            (\bm\varphi_i)_{2m} & (\bm\varphi_j)_{2m} & (\bm\varphi_k)_{2m}          
        \end{vmatrix}}_{\triangleq V_m}
        = 0.
    \end{equation}
    One sufficient condition for above to hold (necessary as well if $r=1$) is if $V_m = 0$ for all $m \in [r]$.
    Geometrically, $V_m$ is the signed volume of the parallelopipe in $\sR^3$, spanned by the three column vectors.
    For the volume to be zero, it must be that $\{((\bm\varphi_i)_{2m - 1}, (\bm\varphi_i)_{2m}), ((\bm\varphi_j)_{2m - 1}, (\bm\varphi_j)_{2m}), ((\bm\varphi_k)_{2m - 1}, (\bm\varphi_k)_{2m})\}$ is collinear in $\sR^2$.
    As this must hold for any $i, j, k \in [K]^3$, it must be that $\{((\bm\varphi_k)_{2m - 1}, (\bm\varphi_k)_{2m})\}_{k \in [K]}$ is collinear as well, for each $m \in [r]$.
\end{proof}

\begin{remark}
    We believe that the above result is extendable to the general case via decomposing the general preference into its transitive and cyclic components~\cite{jiang2011hodge}.
    But then, geometrically, it is unclear how to choose the right features such that the non-transitive and transitive components are compatible with each other, which corresponds to the ``harmonic'' component from the combinatorial Hodge decomposition~\citep{jiang2011hodge}.
\end{remark}

\subsection{Miscellaneous Mathematical Preliminaries}
\label{app:preliminaries}
Here, for completeness and to foster future directions, we provide a bit orthogonal, yet interesting (and hopefully useful) mathematical preliminaries regarding skew-symmetric matrices and anti-symmetric tensor product space.

\subsubsection{Skew-Symmetric Matrix}
A matrix $\mA \in \sR^{d \times d}$ is \textbf{skew-symmetric} (or anti-symmetric) if $\mA^\top = -\mA$.
It is known that the rank of a skew-symmetric matrix must be even~\citep[Section 10.3]{linalgebra}, and it admits the following decomposition, which is its canonical form:
\begin{lemma}[Corollary 2.5.11 of \citet{hornjohnson}\footnote{A fun(?) historical note: this decomposition has been repeatedly rediscovered and renamed: Murnaghan-Wintner decomposition~\citep{murnaghan1931canonical}, Youla decomposition~\citep{youla1961decomposition}, or the Schur decomposition~\citep{balduzzi2018evaluation}, although the latter name is a bit inaccurate as the ``usual'' Schur decomposition should result in an upper triangular matrix in the middle~\citep[]{hornjohnson}.}]
\label{lem:skew}
    $\mA$ is a skew-symmetric of rank $2r \leq d$ if and only if there exists a (unique) orthogonal $\mQ$ (i.e., $\mQ^\top \mQ = \mQ \mQ^\top = \mI_d$) and $\{\lambda_\ell\}_{\ell \in [r]} \subset \sR_{>0}$ such that $\mA = \mQ \bm\Lambda \mQ^\top$, where
    \begin{equation}
        \bm\Lambda = \left(\bigoplus_{\ell \in [r]} \lambda_\ell \mS \right)\oplus \vzero_{d - 2r},
    \end{equation}
    where $\oplus$ is the matrix direct sum and $\mS := \begin{bmatrix} 0 & 1 \\ -1 & 0 \end{bmatrix}$.
    Moreover, $\{ \pm \lambda_\ell i \}_{\ell \in [r]}$ are the eigenvalues of $\mA$.
\end{lemma}
We also remark that the above form can be quite efficiently computed~\citep{ward1978skewsymmetric,penke2020skewsymmetric}.

Let $\Skew(d) := \{ \bm\Theta \in \sR^{d \times d} : \bm\Theta^\top = - \bm\Theta \}$.
It is a well-known that $\Skew(d)$ is a linear subspace of $\sR^{d \times d}$, and that the mapping $\mA \mapsto \frac{1}{2} (\mA - \mA^\top)$ is an orthogonal projection onto $\Skew(d)$~\citep[Chapter 6.6]{linalgebra}.
We will also consider rank-constrained $\Skew(d)$, defined as $\Skew(d; 2r) := \{ \bm\Theta \in \sR^{d \times d} : \bm\Theta^\top = - \bm\Theta, \ \rank(\bm\Theta) = 2r \}$.
This is a matrix manifold whose dimension is given as follows (see Appendix~\ref{app:skew-dim} for the proof):
\begin{proposition}
\label{prop:skew-dim}
    $\dim(\Skew(d; 2r)) = 2dr - (2r^2 + r).$
\end{proposition}

\subsubsection{2nd-Order Tensor Product Space}
\label{sec:prelim-tensor}
Here, we largely follow the exposition of Section 2 of \citet{garcia2023tensor} and Section I.5 of \citet{matrixanalysis}, to which we refer interested readers for an overview of general tensor algebra over Hilbert space.

We define the \textbf{2nd-order tensor power} of $\sR^d$ as $(\sR^d)^{\otimes 2} := \{ \vx \otimes \vy : \vx, \vy \in \sR^d \}$, where the inner product\footnote{Such inner product is unique~\citep[Proposition 3.8.2]{matrixanalysis}.} is such that $\langle \vx_1 \otimes \vx_2, \vy_1 \otimes \vy_2 \rangle = \langle \vx_1, \vy_1 \rangle \langle \vx_2, \vy_2 \rangle$.
Then, its orthonormal basis is given as $\{\ve_i \otimes \ve_j\}_{(i, j) \in [d]^2}$.

Consider the symmetrization and antisymmetrization operators, defined as $\gP_S(\vx \otimes \vy) := \vx \odot \vy := \frac{1}{2} (\vx \otimes \vy + \vy \otimes \vx)$ and $\gP_A(\vx \otimes \vy) := \vx \wedge \vy := \frac{1}{2} (\vx \otimes \vy - \vy \otimes \vx).$
Then, one can orthogonally decompose $(\sR^d)^{\otimes 2} = (\sR^d)^{\odot 2} \oplus (\sR^d)^{\wedge 2} $, where the two spaces are spanned by their respective \textit{orthonormal} basis: $(\sR^d)^{\odot 2} = \mathrm{span}\left( \left\{ \ve_i \odot \ve_i \right\}_{i \in [d]} \cup \left\{ \sqrt{2} (\ve_i \odot \ve_j) \right\}_{1 \leq i < j \leq d} \right),$ and $(\sR^d)^{\wedge 2} = \mathrm{span}\left( \left\{ \sqrt{2} (\ve_i \wedge \ve_j) \right\}_{1 \leq i < j \leq d} \right).$

Let us focus on the antisymmetric part.
It is known that $\gP_A$ is an orthogonal projection onto $\sR^{\wedge 2}$ with the following idempotent, full row-rank matrix representation of $\gP_A$:
\begin{equation}
    \mP_A := \sqrt{2}
    \begin{bmatrix}
        \ve_1 \wedge \ve_2 & \ve_1 \wedge \ve_3 & \cdots & \ve_{d-1} \wedge \ve_d
    \end{bmatrix}
    \in \sR^{d^2 \times \binom{d}{2}}.
\end{equation}
It satisfies $\mP_A^\top \mP_A = \mI_{\binom{d}{2}}$ and $\mP_A \mP_A^\top (\vx \otimes \vy) = \vx \wedge \vy$.

\subsection{Proof of Proposition~\ref{prop:skew-dim}}
\label{app:skew-dim}
The proof utilizes some tools from topology, Lie group theory and matrix theory.
Our main references are \citet{topology}, Chapter 21 of \citet{manifolds} and \citet{hornjohnson}.

Consider the generalized linear group $\GL_d(\sR) := \{ \mX \in \sR^{d \times d} : \det(\mX) \neq 0 \}$, which is a Lie group of dimension $d^2$.
We then define the group action of $\GL_d(\sR)$ on $\Skew(d; 2r)$ as the following: 
\begin{equation}
\label{eqn:action}
    (\mX, \mA) \mapsto \mX \mA \mX^\top, \quad \mX \in \GL_d(\sR), \mA \in \Skew(d; 2r).
\end{equation}

We now utilize the following lemma:
\begin{lemma}[Theorem 21.20 of \citet{manifolds}]
\label{lem:lee}
    Let $X$ be a set and $G$ be a Lie group that acts on $X$ \textit{transitively}, i.e., for any $x, y \in X$ there exists a $g \in G$ such that $(g, x) = y$.
    Suppose that there exists a point $p \in X$ such that the stabilizer group $G_p$ is closed in $G$.
    Then, $X$ has a unique smooth manifold structure w.r.t. which the given action is smooth.
    With this structure, $\dim X = \dim G - \dim G_p$.
\end{lemma}
We first show that our group action indeed satisfies the assumptions of the above lemma.
For simplicity, let us denote
\begin{equation}
    \mS_{d,2r} := \underbrace{\bigoplus_{\ell \in [r]} \begin{bmatrix} 0 & 1 \\ -1 & 0 \end{bmatrix}}_{=: \mS_{2r}} \oplus \vzero_{d - 2r}.
\end{equation}
\begin{claim}
    The action as defined in Eqn.~\eqref{eqn:action} is transitive.
\end{claim}
\begin{proof}
    To see this, consider two $\mA, \mB \in \Skew(d; 2r)$.
    Then by Lemma~\ref{lem:skew}, there exists $\mU_\mA, \mU_\mB \in O(d)$ and $\{ \lambda_{\ell,\mA}^2, \lambda_{\ell,\mB}^2 \}_{\ell \in [r]}$ such that $\mA = \mU_\mA \bm\Lambda_\mA \mS_{d,2r} \bm\Lambda_\mA^\top \mU_\mA^\top$ and $\mB = \mU_\mB \bm\Lambda_\mB \mS_{d,2r} \bm\Lambda_\mB^\top \mU_\mB^\top$, where
    \begin{equation}
        \bm{\Lambda}_\mA = \operatorname{diag}(\underbrace{\lambda_{1,\mA}, \lambda_{1,\mA}}_{\text{twice}}, \cdots, \underbrace{\lambda_{r,\mA}, \lambda_{r,\mA}}_{\text{twice}}, \underbrace{0, 0, \dots, 0}_{\text{remaining entries}})
    \end{equation}
    and similarly for $\bm\Lambda_\mB$.
    Then, defining $\mX = (\mU_B \bm\Lambda_B) (\mU_A \bm\Lambda_\mA)^{-1} \in \GL_d(\sR)$, it can be seen that $(\mX, \mA) = \mB$.
\end{proof}
For the point $p$ in the above lemma, we choose $\mS_{d,2r} \in \Skew(d; 2r)$.
Let us denote its stabilizer group as $S_{d,2r} := \{ \mX \in \GL_{d - 2r}(\sR) : \mX \mS_{d,2r} \mX^\top = \mS_{d,2r} \}$.
\begin{claim}
    $S_{d,2r}$ is closed in $\GL_d(\sR)$.
\end{claim}
\begin{proof}
    Consider a mapping $\rho : \mX \mapsto \mX \mS_{d,2r} \mX^\top$, which is continuous.
    Noting that $S_{d,2r} = \rho^{-1}(\{\mS_{d,2r}\})$ and that $\{\mS_{d,2r}\}$ is closed (in Hausdorff space, which $\GL_d(\sR)$ is), $S_{d,2r}$ is also closed by continuity.
\end{proof}

We now characterize $S_{d,2r}$.

Using block matrix notation, we need to characterize $\mX = \begin{bmatrix} \mX_{11} & \mX_{12} \\ \mX_{21} & \mX_{22} \end{bmatrix}$ such that $\mX$ is invertible and $\mX \mS_{2r} \mX^\top = \mS_{2r}$.
After some tedious computations, we have that
\begin{equation}
    \begin{bmatrix}
        \mX_{11} \mS_{2r} \mX_{11}^\top & \mX_{11} \mS_{2r} \mX_{21}^\top \\ \mX_{21} \mS_{2r} \mX_{11}^\top & \mX_{21} \mS_{2r} \mX_{21}^\top
    \end{bmatrix}
    =
    \begin{bmatrix}
        \mS_{2r} & \vzero_{2r \times (d - 2r)} \\ \vzero_{(d - 2r) \times 2r} & \vzero_{2r \times 2r}
    \end{bmatrix}.
\end{equation}
Consider the first block.
Taking the determinant, we can deduce that $\det(\mX_{11})^2 = 1 \neq 0$, i.e., $\mX_{11}$ should be invertible.
As $\mS_{2r}$ is also invertible, the antidiagonal blocks implies that $\mX_{21} = \vzero_{(d - 2r) \times 2r}$.

So far, we have that $\mX$ should be of the form
\begin{equation}
    \mX =
    \begin{bmatrix}
    \mX_{11} & \mX_{12} \\
    \vzero_{(d - 2r) \times 2r} & \mX_{22},
    \end{bmatrix}
\end{equation}
where $\mX_{11} \in \Sym(2p) := \{ \mX \in \GL_n(\sR) : \mX \mS_{2r} \mX^\top = \mX \}$.
By Schur's determinant formula, as $\mX$ must be invertible, we must have that
\begin{equation}
    \det(\mX) = \det(\mX_{11}) \det(\mX_{22}) \neq 0,
\end{equation}
i.e., $\mX_{22}$ should also be invertible.

We now derive the dimension of $\GL_{d - 2r}(\sR)$ $\Sym(2r)$.
\begin{claim}
    $\dim(\GL_{d - 2r}(\sR)) = (d - 2r)^2$.
\end{claim}
\begin{proof}
    Let $n = d - 2r$.
    Then, note that $\GL_n(\sR) = \det^{-1}(\sR \setminus \{0\})$.
    As $\det$ is continuous and $\sR \setminus \{0\}$ is open, $\GL_n(\sR) \subset \sR^{n \times n}$ is open, and we are done.
\end{proof}

\begin{claim}
    $\dim(\Sym(2r)) = 2r^2 + r$.
\end{claim}
\begin{proof}
    We do this by counting the number of independent constraints, then subtracting it from $\dim(\GL_{2r}(\sR)) = 4r^2$.
    Let us denote $\mS := \begin{bmatrix}
            0 & 1 \\ -1 & 0
        \end{bmatrix}$ for simplicity.
    First, for a $\mA \in \sR^{2 \times 2}$, note that
    \begin{equation}
        \mA \mS \mA^\top = \det(\mA) \mS.
    \end{equation}
    Now consider a $\mX \in \mathrm{GL}_{2r}(\sR)$, consisting of $r$ number of $2 \times 2$ blocks:
    \begin{equation}
    \mX =
    \begin{bmatrix}
        \mX_{11} & \mX_{12} & \cdots & \mX_{1r} \\
        \mX_{21} & \mX_{22} & \cdots & \mX_{2r} \\
        \vdots & \vdots & \ddots & \vdots \\
        \mX_{r1} & \mX_{r2} & \cdots & \mX_{rr}
    \end{bmatrix}.
    \end{equation}
    Then, by the block matrix multiplication and the above result, we have that
    \begin{align}
        \left(\mX \mS_{2r} \mX^\top\right)_{i,j} = 
        \begin{cases}
            \left( \sum_{k=1}^r \det(\mX_{ik}) \right) \mS, \quad & i = j, \\
            \sum_{k=1}^r \mX_{ik} \mJ \mX_{kj}^\top, \quad &i \neq j
        \end{cases} = 
        \begin{cases}
            \mS, \quad & i = j, \\
            \vzero_{2 \times 2}, \quad &i \neq j
        \end{cases}.
    \end{align}
    where here, $(\cdot)_{i,j}$ refers to the $2 \times 2$ block at the $(i, j)$ location.

    There are $r$ constraints for $i = j$ and $4 \binom{r}{2} = 2r(r - 1)$ constraints for $i \neq j$, which amounts to $2r^2 - r$ constraints in total.
    Thus, the dimension of $\Sym(2r)$ becomes $4r^2 - (2r^2 - r) = 2r^2 + r$.
\end{proof}

All in all, we have that
\begin{align}
    \dim(S_{d,2r}) &= \underbrace{\dim(\Sym(2r))}_{\text{degrees of freedom for }\mX_{11}} + \underbrace{\dim(\sR^{2r \times (d - 2r)})}_{\text{degrees of freedom for }\mX_{12}} + \underbrace{\dim(\GL_{d - 2r}(\sR))}_{\text{degrees of freedom for }\mX_{22}} \\
    &= (2r^2 + r) + 2r(d - 2r) + (d - 2r)^2 \\
    &= d^2 + 2r^2 + r - 2dr.
\end{align}

Applying Lemma~\ref{lem:lee}, we have that
\begin{equation}
    \dim(\Skew(d; 2r)) = \dim(\GL_d(\sR)) - \dim(S_{d,2r})
    = 2dr - (2r^2 + r).
\end{equation}
\qed

\newpage
\section{ADDITIONAL DETAILS FOR BILINEAR DUELING BANDITS: REGRET ANALYSIS}
\label{app:bilinear2}

\subsection{\texorpdfstring{Proof of \cref{thm:borda-bound} -- Borda Regret Upper Bound for Bilinear Dueling Bandits}{Proof of Theorem 6.1 -- Borda Regret Upper Bound for Bilinear Dueling Bandits}}
\label{app:bilinear-alg}

We state the full version of the Borda regret bound and give its proof:
\begin{theorem}[Full Statement of \cref{thm:borda-bound}]\label{thm:borda-bound-full}
    Let us denote $\GL_{\min} := \GL_{\min}(\gX)$, $\bar{\kappa}(\pi) := \E_{\mX \sim \pi} \left[ \dmu(\langle \mX, \bm\Theta_\star \rangle) \right]$, and $\pi_\star \in \argmin_{\pi \in \gP} \GL(\pi; \bm\Theta_\star)$.
    Choose $N_1$ and $N_2$ as
    \begin{align}
        N_1 &\asymp \frac{\bar{\kappa}(\pi_\star) d r}{\lambda_{\min}(\mH(\pi_\star; \bm\Theta_\star))^2} \log\frac{d}{\delta} \vee R_s \frac{\bar{\kappa}(\pi_\star) r^2}{\lambda_{\min}(\mH(\pi_\star; \bm\Theta_\star))^3} \sqrt{\frac{N_2 \log\frac{d}{\delta}}{\GL_{\min}(\gX)}}, \\
        N_2 &= \left( \GL_{\min} \log\frac{d}{\delta} \right)^{1/3} (\kappa_\star^B T)^{2/3},
    \end{align}
    and let us assume that $T \geq N_1 + N_2$.
    Then, the following Borda regret bound of \texttt{BETC-GLM-LR}\footnote{This is an acronym for \emph{Borda Explore-Then-Commit for Generalized Linear Models with Low-Rank structure}.} holds with probability at least $1 - \delta$:
    \begin{equation}
        \Reg^B(T) \lesssim \left( \GL_{\min} \log\frac{d}{\delta} \right)^{1/3} (\kappa_\star^B T)^{2/3} + R_s L_\mu \left( \frac{\GL_{\min}}{\kappa_\star^B} \log\frac{d}{\delta} \right)^{2/3} T^{1/3} + N_1.
    \end{equation}
    Here, it is clear that the first term dominates when $T$ is sufficiently large.
\end{theorem}
\begin{proof}
We na\"{i}vely bound the instantaneous regret from the exploration phase with $1$, and thus, the cumulative regret up to the forced exploration is $N_1 + N_2$.

After the exploration phase, the instantaneous regret is the same as $B(\bm\phi_\star) - B(\widehat{\bm\phi})$.
This is bounded as follows:
\begin{align}
    B(\bm\phi_\star) - B(\widehat{\bm\phi}) &= \E_{\bm\phi' \sim \Unif(\gX)} \left[ \mu\left( \bm\phi_\star^\top \bm\Theta_\star \bm\phi' \right) - \mu( \widehat{\bm\phi}^\top \bm\Theta_\star \bm\phi') \right] \\
    &\leq \E_{\bm\phi' \sim \Unif(\gX)} \left[ \mu\left( \bm\phi_\star^\top \bm\Theta_\star \bm\phi' \right) - \mu( \bm\phi_\star^\top \widehat{\bm\Theta} \bm\phi') \right] \tag{Definition of $\widehat{\bm\phi}$} \\
    &\overset{(*)}{=} \underbrace{\E_{\bm\phi' \sim \Unif(\gX)} \left[ \dmu\left( \bm\phi_\star^\top \bm\Theta_\star \bm\phi' \right) \bm\phi_\star^\top (\bm\Theta_\star - \widehat{\bm\Theta}) \bm\phi' \right]}_{\triangleq Q_1} + \underbrace{\E_{\bm\phi' \sim \Unif(\gX)} \left[ - \left( \bm\phi_\star^\top (\bm\Theta_\star - \widehat{\bm\Theta}) \bm\phi' \right)^2 \tilde{\theta}(\bm\phi') \right]}_{\triangleq Q_2} \tag{First-order Taylor expansion with integral remainder}
\end{align}
where at $(*)$, we define
\begin{equation}
    \tilde{\theta}(\bm\phi') := \int_0^1 (1 - z) \ddmu\left( \bm\phi_\star^\top \left( (1 - z) \bm\Theta_\star + z \widehat{\bm\Theta} \right) \bm\phi' \right) dz.
\end{equation}
$Q_1$ can be bounded as
\begin{align}
    Q_1 &= \E_{\bm\phi' \sim \Unif(\gX)} \left[ \dmu\left( \bm\phi_\star^\top \bm\Theta_\star \bm\phi' \right) \bm\phi_\star^\top (\bm\Theta_\star - \widehat{\bm\Theta}) \bm\phi' \right] \\
    &\leq \left( \max_{\bm\phi' \in \gX} \left| \bm\phi_\star^\top (\bm\Theta_\star - \widehat{\bm\Theta}) \bm\phi' \right| \right) \E_{\bm\phi' \sim \Unif(\gX)} \left[ \dmu\left( \bm\phi_\star^\top \bm\Theta_\star \bm\phi' \right) \right] \\
    &\leq \kappa_\star^B \bignormop{\widehat{\bm\Theta} - \bm\Theta_\star} \tag{rectangular quotient relation for $\bignormop{\cdot}$ \& $\bm\phi_\star, \bm\phi' \in \gB^d(1)$ \& definition of $\kappa_\star^B$} \\
    &\lesssim \kappa_\star^B \sqrt{\frac{\GL_{\min}}{N_2} \log\frac{d}{\delta}}. \tag{\cref{thm:estimation-final}}
\end{align}

By self-concordance, we have that $|\tilde{\theta}(\bm\phi')| \leq \frac{1}{2} R_s L_\mu$ for any $\bm\phi' \in \gX$, and thus, $Q_2$ can be bounded as
\begin{align}
    Q_2 
    \leq \frac{1}{2} R_s L_\mu \E_{\bm\phi' \sim \Unif(\gX)} \left[ \left( \bm\phi_\star^\top (\bm\Theta_\star - \widehat{\bm\Theta}) \bm\phi' \right)^2 \right]
    \lesssim \frac{R_s L_\mu \GL_{\min}}{N_2} \log\frac{d}{\delta}.
\end{align}

Combining everything, we have that
\begin{equation}
     B(\bm\phi_\star) - B(\widehat{\bm\phi}) \lesssim \kappa_\star^B \sqrt{\frac{\GL_{\min}}{N_2} \log\frac{d}{\delta}} + \frac{R_s L_\mu \GL_{\min}}{N_2} \log\frac{d}{\delta}.
\end{equation}

All in all, we have
\begin{align}
    \Reg^B(T) &\lesssim N_1 + N_2 + (T - N_1 - N_2) \left( \kappa_\star^B \sqrt{\frac{\GL_{\min}}{N_2} \log\frac{d}{\delta}} + \frac{R_s L_\mu \GL_{\min}}{N_2} \log\frac{d}{\delta}. \right) \nonumber \\
    &\leq N_1 + N_2 + T \sqrt{\frac{\GL_{\min}}{N_2} \log\frac{d}{\delta}} \left( \kappa_\star^B + R_s L_\mu \sqrt{\frac{\GL_{\min}}{N_2} \log\frac{d}{\delta}} \right).
\end{align}
Let us optimize for $N_2$ using the last expression.

If we choose $N_2 = \left( \GL_{\min} \log\frac{d}{\delta} \right)^{1/3} (\kappa_\star^B T)^{2/3}$, we have
\begin{equation}
    \Reg^B(T) \lesssim N_1 + \left( \GL_{\min} \log\frac{d}{\delta} \right)^{1/3} (\kappa_\star^B T)^{2/3} + R_s L_\mu \left( \frac{\GL_{\min}}{\kappa_\star^B} \log\frac{d}{\delta} \right)^{2/3} T^{1/3}.
\end{equation}
With this, we have the following requirement on $N_1$, as stated in Eqn.~\eqref{eqn:N1-requirement}: recalling that $\bar{\kappa}(\pi) = \E_{\mX \sim \pi} \left[ \dmu(\langle \mX, \bm\Theta_\star \rangle) \right]$ and $\pi_\star \in \argmin_{\pi \in \gP} \GL(\pi; \bm\Theta_\star)$,
\begin{equation}
    N_1 \gtrsim \frac{\bar{\kappa}(\pi_\star) d r}{\lambda_{\min}(\mH(\pi_\star; \bm\Theta_\star))^2} \log\frac{d}{\delta} \vee R_s \frac{\bar{\kappa}(\pi_\star) r^2}{\lambda_{\min}(\mH(\pi_\star; \bm\Theta_\star))^3} \sqrt{\frac{N_2 \log\frac{d}{\delta}}{\GL_{\min}(\gX)}}
\end{equation}
\end{proof}

\subsection{\texorpdfstring{Relations to \citet{wu2024dueling}}{Relations to Wu et al. (2024)}}
\label{app:bilinear-regret}

\paragraph{Reduction to \citet{wu2024dueling}.}
To our knowledge, \citet{wu2024dueling} is the only comparable competitor in our setting of Borda regret minimization.
To do that, we first describe how to reduce our bilinear dueling bandits to their setting.
Recall that \citet{wu2024dueling} require vector-valued features for each pair of items, $\bm\phi_{i,j} = -\bm\phi_{j,i}$.
As $\bm\Theta_\star = \widetilde{\bm\Theta}_\star - \widetilde{\bm\Theta}_\star^\top$ for some $\widetilde{\bm\Theta}_\star \in \sR^{d \times d}$, one can rewrite the bilinear preference as
\begin{equation}
    \mu\left( \bm\phi_i^\top (\widetilde{\bm\Theta}_\star - \widetilde{\bm\Theta}_\star^\top) \bm\phi_j \right)
    = \mu\left( \left\langle \widetilde{\bm\Theta}_\star, \bm\phi_i \bm\phi_j - \bm\phi_j \bm\phi_i^\top \right\rangle \right).
\end{equation}
One may be tempted to set $\bm\phi_{i,j} = \mathrm{vec}(\bm\phi_i \bm\phi_j^\top - \bm\phi_j \bm\phi_i^\top)$.
However, recalling the discussions from \cref{sec:prelim-tensor}, one must set $\bm\phi_{i,j} = \mP_A^\top \mathrm{vec}(\bm\phi_i \bm\phi_j^\top - \bm\phi_j \bm\phi_i^\top)$ for $\bm\phi_{i,j}$'s to be able to fully span $\sR^{\wedge 2}$.
Setting $\bm\theta_\star = \mP_A^\top \mathrm{vec}(\widetilde{\bm\Theta}_\star) \in \sR^{d^2}$ and the reduction is complete.

\paragraph{Comparing Regret Upper Bounds.}
A na\"{i}ve application of the algorithm of \citet{wu2024dueling} using the above reduction attains a Borda regret bound of $\widetilde{\gO}(c_\mu^{-1} d^{4/3} T^{2/3})$ up to some epsilon-net error (see their Remark 5.3), where
\begin{equation}
    c_\mu := \min_{\bignorm{\vx}_2 \leq 1, \bignorm{\bm\theta - \bm\theta_\star} \leq 1} \dmu(\langle \vx, \bm\theta \rangle) > 0.
\end{equation}
They have also assumed that $\lambda_{\min}(\mV(\pi^U)) \geq \lambda_0$ for some constant $\lambda_0 > 0$, where $\pi^U \sim \Unif(\gA \times \gA)$~\citep[Assumption 3.1]{wu2024dueling}.
We remark that in many cases, $\lambda_0$ is \textit{not} constant and can be arbitrarily small dimension-wise. In particular, both \citet{wu2024dueling} and our work assumes $\|\phi_{i,j}\|_2 \leq 1$, one can prove that \textbf{$\lambda_0 \leq \frac{1}{d^2}$ for any $\gA$} under this assumption and it is \textit{impossible} to make $\lambda_0$ as a constant, since 
\begin{align}
    \mathrm{tr}\left(\mV(\pi) \right)
    &=\mathrm{tr}\left(\sum_{i,j} \pi(\bm\phi_{i,j}) \bm\phi_{i,j} \bm\phi_{i,j}^\top\right) \\
    &=\sum_{i,j}\pi(\bm\phi_{i,j})\mathrm{tr}\left(\bm\phi_{i,j} \bm\phi_{i,j}^\top\right) \tag{Linearity of $\mathrm{tr}$}\\
&\leq\sum_{i,j}\pi(\bm\phi_{i,j}) \tag{For a vector $\vv$, $\mathrm{tr}(\vv \vv^\top)=\|\vv\|_2^2$ and $\|\bm\phi_{i,j}\|_2 \leq 1$}=1
\end{align}
and $\mathrm{tr}(\mV(\pi))=\sum_{i=1}^{d^2} \lambda_i (\mV(\pi))$. 

Still, for a fair comparison, let us first compare with our bound under the same assumption and various arm-sets: utilizing \Cref{prop:unit-ball},
\begin{itemize}
    \item When $\gA = \gB^d(1)$ ($\gX = \gB^{d \times d}_{\mathrm{op}}(1)$), our regret bound becomes $\widetilde{\gO}\left( \kappa_\star^{-1/3} (\kappa_\star^B d T)^{2/3} \right)$. This is a strict improvement by a factor of $d^{2/3}$ and curvature-dependent quantities.
    \item When $\gA = \{ \ve_i : i \in [d] \}$ ($\gX = \gM$), our regret bound becomes $\widetilde{\gO}\left( \Harm(\bm\Theta_\star)^{-1/3} d (\kappa_\star^B T)^{2/3} \right)$. This is a strict improvement by a factor of $d^{1/3}$ and curvature-dependent quantities.
    \item For general $\gA \subseteq \gB^d(1)$ ($\gX \subseteq \gB^{d \times d}_{\mathrm{op}}(1)$), by \Cref{prop:improve}, our regret bound can be upper bounded with $\tilde{\gO}\left( \kappa_\star^{-1/3} (\kappa_\star^B T)^{2/3} (d \lambda_0)^{1/3} \right)$. This is a strict improvement by a factor of $(d \lambda_0)^{1/3}$, when $\lambda_0 \geq \frac{1}{d^3}$.\footnote{Of course, if the geometry of $\gA$ is ill-distributed, then $\lambda_0$ can be arbitrarily small, and we do not claim that ours is always good. Rather, when $\gA$ is sufficiently well-distributed (which is arguably the usual case). our \texttt{GL-LowPopArt} shows the benefit of exploiting the geometry and curvature.}
\end{itemize}
\begin{remark}[Dependency on $r$]
    A keen reader may notice that our regret bound is independent of the rank $r$ of the matrix $\bm\Theta_\star$, which is also the case for bilinear bandits~\citep[Theorem 4.6]{jang2021bilinear}, albeit for a different reason.
    This is because our \texttt{GL-LowPopArt} exploits the low-rankness of $\gA$ (which induces a matrix-valued arm-set of operator norm at most $1$) and the parameter space $\Skew(d; 2r)$, analogous to bilinear bandits~\citep{jun2019bilinear,jang2021bilinear} and low-rank bandits~\citep{jang2024lowrank,lu2021generalized,kang2022generalized}.
\end{remark}

We now elaborate on the curvature-dependent quantity $\frac{(\kappa_\star^B)^2}{\kappa_\star}$ in our Borda regret bound, an instance-specific scaling not previously reported in the dueling bandits literature.
Let us first recall their definitions:
\begin{equation}
    \kappa_\star := \min_{\bm\phi, \bm\phi' \in \gA} \dmu\left( \bm\phi^\top \bm\Theta_\star \bm\phi' \right), \quad \kappa_\star^B := \E_{\bm\phi' \sim \Unif(\gA)}[\dmu(\bm\phi_\star^\top \bm\Theta_\star \bm\phi')].
\end{equation}
Here, $\kappa_\star$ represents the worst-case flatness (minimum derivative) across all possible pairs of arms, while $\kappa_\star^B$ represents the average flatness when playing the true Borda winner $\bm\phi_\star$ against a uniformly drawn arm.
By definition, we have $\kappa_\star \leq \kappa_\star^B$.
The dueling nature of the regret reveals an interesting dichotomy depending on the relationship between these two quantities:

\begin{itemize}
    \item \textbf{Uniform Flatness ($\kappa_\star \asymp \kappa_\star^B$):} If the landscape's flatness is relatively uniform --meaning the hardness of estimating the preference for the worst-case pair is of the same order as estimating preferences involving the Borda winner -- then the curvature-dependent factor simplifies to $\frac{(\kappa_\star^B)^2}{\kappa_\star} \asymp \kappa_\star$. Consequently, our permanent regret scales as $\widetilde{\gO}(\kappa_\star^{1/3} T^{2/3})$. This implies that a \textit{flatter} problem (smaller $\kappa_\star$, which corresponds to smaller variance in the feedback) strictly \textit{reduces} the permanent regret, analogous to the findings in generalized linear bandits \citep{abeille2021logistic,lee2024glm}.
    
    \item \textbf{Adversarial Flatness ($\kappa_\star \ll \kappa_\star^B$):} Conversely, if there exists an adversarial pair of arms that is significantly flatter (and thus much harder to learn) than the pairs involving the Borda winner, the ratio $\frac{(\kappa_\star^B)^2}{\kappa_\star}$ blows up. In this regime, the difficulty of identifying the parameters in the worst-case flat region overwhelmingly dominates the variance-reduction we would otherwise enjoy around the Borda winner. As a result, the permanent regret no longer benefits from the general flatness of the problem.
\end{itemize}

\paragraph{Regret Lower Bound.}
\citet[Theorem 4.1]{wu2024dueling} obtain a regret lower bound of $\Omega(d^{2/3} T^{2/3})$ for $\bm\phi_{i,j}, \bm\theta_\star \in \sR^d$, and a similar lower bound for unstructured dueling bandits has been obtained by \citet[Theorem 16]{saha2021borda}; $T^{2/3}$ stems from the fact that the exploration and exploitation cannot be mixed.
This suggests that at least in terms of $T$, our \texttt{BETC-GLM-LR} is also optimal.

However, their lower bound cannot be directly applied to our setting, as our bilinear dueling bandits, in essence, constrain the matrix arm to be of rank-1.
It is clear that their hard instance, based on the lower bound for stochastic linear bandits~\citep{dani2008stochastic}, cannot be instantiated as our setting.
We leave obtaining a tight lower bound to future work, considering how even in stochastic bilinear bandits (non-dueling), the lower bound remains open~\citep{kotlowski2019banditpca,jang2021bilinear,jun2019bilinear}.
It would also be interesting from the curvature perspective on whether the above dichotomy is fundamental, i.e., whether one can derive a regret lower bound that depends on $\frac{(\kappa_\star^B)^2}{\kappa_\star}$.
A potential starting point may be from the regret lower bound of \citet[Theorem 6.1]{jang2024lowrank}, although they do not consider the Borda regret nor nonlinear link function.

\newpage
\section{PRELIMINARY EXPERIMENTS: 1-BIT MATRIX COMPLETION/RECOVERY}
\label{app:experiments}

In this appendix, we present preliminary numerical results on 1-bit matrix completion and recovery~\citep{davenport2014completion} to demonstrate the empirical effectiveness of \texttt{GL-LowPopArt}. For results in the Gaussian (i.e., linear) setting, we refer readers to the experiments in \citet{jang2024lowrank}. The source code to reproduce these experiments is publicly available on our GitHub repository.\footnote{\url{https://github.com/nick-jhlee/GL-LowPopArt}}

\subsection{Experimental Setting}
\paragraph{Dataset.}
We set the dimensions $d_1 = d_2 = 3$, the rank $r = 1$, and the tolerance level $\delta = 0.001$. We evaluate three different choices for the action set $\gX$:
\begin{enumerate}
    \item \textbf{Matrix Completion:} $\gX = \{ \ve_i \ve_j^\top : 1 \leq i, j \leq 3 \}$, where $\ve_i$ denotes the standard basis vector in $\sR^{d_1}$.
    \item \textbf{Matrix Recovery (Random):} $K = 50$ arms sampled uniformly at random from the unit sphere $\gS^{d_1 d_2 - 1}(1)$.
    \item \textbf{Matrix Recovery (Hard):} A hard instance proposed by \citet[Appendix C.3]{jang2024lowrank}, defined as:
    \begin{equation}
        \gX_{\text{hard}} = \left\{  \sqrt{\frac{1}{d}} \vec^{-1}(\ve_1) \right\} \cup \left\{ \vec^{-1}\left( \sqrt{\frac{1}{d+1}} \ve_1 + \sqrt{\frac{d}{d+1}} \ve_i \right) \ : \ i = 2, \ldots, d^2 \right\}.
    \end{equation}
\end{enumerate}
For each experimental setting, we repeat the process 30 times across a wide range of sample sizes $N$. In each repetition, the underlying true parameter is generated as $\bm\Theta_\star = 2 \mU \mU^\top$, where $\mU$ is obtained via the QR decomposition of $\mU' \sim \gN(0, 1)^{d \times r}$. For the random matrix recovery setting, the arm set $\gX$ is also resampled independently for each repetition.

\paragraph{Algorithms.}
We compare the following algorithms: (i) nuclear-penalized MLE with the uniform design (`U'), (ii) \texttt{GL-LowPopArt} with a uniform Stage~II design (`U+U'), and (iii) \texttt{GL-LowPopArt} with an optimal GL-design in Stage~II (`U+GL'). We employ the theoretically prescribed hyperparameters without further tuning: $\lambda_N = \sqrt{\frac{2}{N} \log\frac{12}{\delta}}$ for the nuclear-penalized MLE (\Cref{prop:lambda}), and $\nu = \sqrt{\frac{2}{5 \GL(\pi; \bm\Theta_0) N_2} \log\frac{24}{\delta}}$ for \texttt{GL-LowPopArt}.

To solve the Stage~I nuclear-penalized MLE, we use a restarted FISTA-style proximal gradient method with a backtracking line search~\citep{fista_restart,fista,cai2010thresholding}. For Stage~II and the optimal design computations, we use CVXPY~\citep{diamond2016cvxpy,agrawal2018rewriting} equipped with the MOSEK solver.

To ensure a fair comparison, we fix the total sample budget $N$ across all methods, enforcing $N_1 + N_2 = N$, where $N_i$ denotes the number of samples allocated to Stage $i$. Specifically, we set $N_1 = \floor{N / 10}$ and $N_2 = N - N_1$.\footnote{While the main text establishes that $N_1 \asymp \sqrt{N}$ suffices asymptotically, we allocate a larger proportional budget to $N_1$ here to account for finite-sample effects and ensure sufficient initial exploration.}

\begin{remark}
    We also experimented with the Burer-Monteiro factorization (BMF) approach using a small random initialization~\citep{stoger2021lowrank,kim2023lowrank}, factoring the parameter as $\bm\Theta = \mU \mU^\top$. However, this method yielded the worst performance across all scenarios. This suggests that the non-convex loss landscape is highly non-benign in the noisy setting, corroborating the observations of \citet{ma2023bmf}. We have omitted these results for brevity.
\end{remark}

\subsection{Results \& Discussion}
We report the mean operator norm error for each algorithm alongside $95\%$ two-sided confidence intervals, computed using the empirical standard error over the 30 independent repetitions (using the Student's $t$-distribution or normal approximation)~\citep{student1908probable}.

\Cref{fig:1-main} summarizes the updated results.\footnote{After the camera-ready submission, we revised the experiments in two ways: we now report the operator norm error, and we fixed an indexing bug in the previous implementation of optimal GL-design. We then re-ran all experiments, which led to some changes in the empirical trends, reflected in the new discussion.}
Across all three tasks, \texttt{GL-LowPopArt} substantially outperforms the nuclear-penalized MLE baseline.
In particular, the baseline error decreases much more slowly and exhibits a long finite-sample plateau (e.g., up to $N \approx 10^5$ for \textbf{Matrix Recovery (Hard)} and $N \approx 10^4$ for other tasks), whereas both Stage-II variants of \texttt{GL-LowPopArt} improve much more rapidly as $N$ grows. We now elaborate on the difference in performance between uniform and optimal GL-design for Stage~II.

For \textbf{matrix completion}, the gain from using the optimal GL-design in Stage~II is negligible: the uniform and GL-design perform very similarly, and the uniform Stage-II design is slightly better in the small-sample regime.
For \textbf{matrix recovery (random)}, the gain is also somewhat modest, but in the same small-sample regime, the GL-design is slightly better than uniform design.
All in all, the two Stage-II variants can be extremely close, with only a small edge for the GL-design in parts of the range.
The clearest benefit of the optimal GL-design appears in \textbf{matrix recovery (hard)}, where it consistently outperforms the uniform design by the largest margin in the small sample regime, but the gap narrows down as $N$ increases.
Thus, the empirical benefit of GL-design is most pronounced on geometrically anisotropic instances, but still modest compared to our initial expectations.
We leave a more systematic empirical study of this phenomenon to future work.

\begin{figure}[!h]
    \centering
    \includegraphics[width=0.6\linewidth]{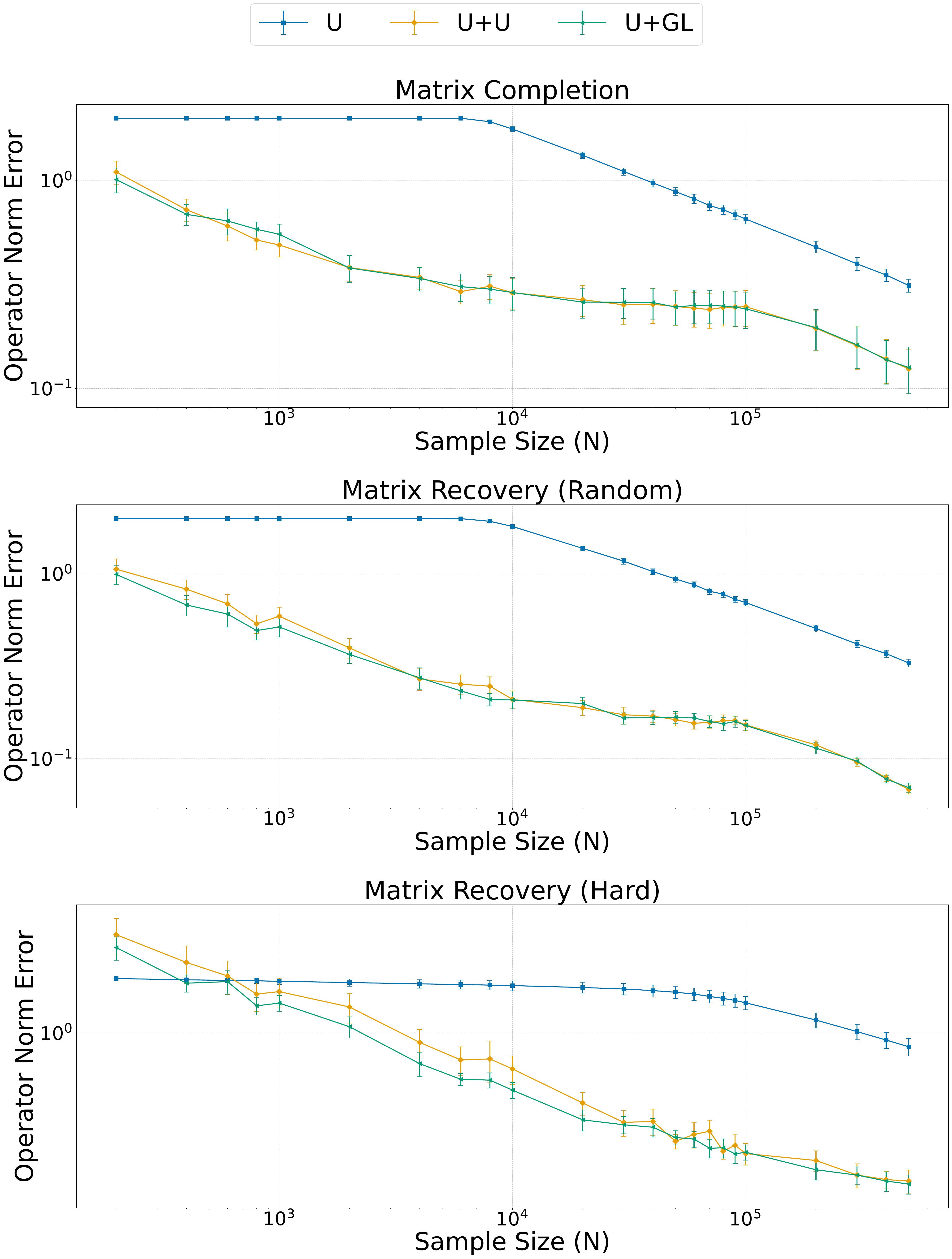}
    \caption{Log-log plots of operator norm errors across a wide range of sample sizes $N$, averaged over 30 independent repetitions. The first, second, and third rows correspond to the matrix completion, random matrix recovery, and hard matrix recovery settings, respectively.}
    \label{fig:1-main}
\end{figure}

\clearpage
\section{ADDITIONAL FUTURE DIRECTIONS}
\label{app:future}

\paragraph{(1) Computational Efficiency.}
Although \texttt{GL-LowPopArt} is computationally tractable, it is not efficient.
The main bottleneck stems from two $\mathcal{O}((d_1 d_2)^3)$ operations in Stage~II: inverting the Hessian $\mH(\pi_2; \bm\Theta_0)$ for matrix one-sample estimators and performing SVD on $\bm{\Theta}_1$.
Improving the computational efficiency of these steps \textit{while retaining statistical efficiency} remains an important future direction.

\paragraph{(2) Beyond Single-Switch Designs.}
We derived our estimation rate using a single-switch design, where the data collection policy is fixed within each of the two stages.
An interesting question is whether granting the learner full control over a history-dependent adaptive policy would enable a statistical rate that fundamentally outperforms the $\GL$-Design bound achieved by our \texttt{GL-LowPopArt}.
The study of valid statistical inference under such general adaptive data collection mechanisms is an active research front~\citep{zhang2021adaptively,zrnic2024active,bibaut2025adaptive}.
Investigating the precise trade-off between the frequency and flexibility of adaptive control and the theoretical limits of the error rate is an important direction for future work.

\paragraph{(3) Hard Instances under Fixed Passive Designs.}
One may wonder whether the bounds in our \Cref{prop:improve,prop:unit-ball} are tight for those specific arm-sets.
Moreover, readers familiar with \citet{jang2024lowrank} may wonder if, similar to their Lemma 3.6, we could construct a hard arm-set instance $\gA_{\mathrm{hard}}$ such that $\GL_{\min}\left( \gA_{\mathrm{hard}} \right) \asymp \frac{1}{H_{\min}\left( \gA_{\mathrm{hard}} \right)}$ where $H_{\min}\left( \gA_{\mathrm{hard}} \right) := \max_{\pi \in \gP(\gA)} \lambda_{\min}(\mH(\pi; \bm\Theta_\star))$ is the E-optimal design related to the true Hessian.
These questions, although we leave to future work, would help us understand better the tightness and importance of our proposed optimal GL-design

\paragraph{(4) Beyond Low Rank.}
Our estimator naturally adapts to low-rank structure, but many problems of interest exhibit additional or alternative structures, such as row/column sparsity~\citep{zhao2014structured}, joint low-rank and sparse decompositions~\citep{yang2013dirty,oymak2015simultaneous,richard2012simultaneous,zhao2017simultaneous}, or other structured priors.
Extending \texttt{PopArt}-style estimators~\citep{jang2022popart} to these settings could yield analogous instance-wise optimality guarantees, and may have immediate impact in high-dimensional applications such as genomics, recommender systems, and structured bandits.

\paragraph{(5) Robustness to Model Misspecification.}
Our analysis assumes a well-specified GLM, as is standard in statistical learning and bandits~\citep[Chapter~24.4]{banditalgorithms}.
Under misspecification, however, the Stage~I estimator converges not to the true $\bm\Theta_\star$, but to the KL projection of the true distribution onto the assumed GLM class~\citep{white1982misspecified}.
This may introduce a persistent bias that Stage~II cannot eliminate, leading to degraded performance.
While mild forms of misspecification (e.g., variance misestimation in the Gaussian case) may only result in conservative but still consistent estimates---see, e.g., adaptive procedures such as the square-root LASSO~\citep{klopp2014general}---more severe mismatches remain challenging.
An important future direction is to develop \texttt{GL-LowPopArt} variants that explicitly account for GLM uncertainty, either through Bayesian approaches~\citep{walker2013misspecified} or misspecification-robust estimators~\citep{robins1994robust,fortunati2017misspecified}.

\end{document}